\documentclass[a4paper,11pt,english]{article}
\usepackage{arxiv}
\usepackage{amsmath,amsfonts,amssymb,amsthm,enumerate,times}
\usepackage{mathtools}
\usepackage[usenames,dvipsnames]{color}
\usepackage{hyperref}
\hypersetup{
    breaklinks,
    colorlinks,
    linkcolor=gray,
    citecolor=gray,
    urlcolor=gray,
    pdftitle={},
    pdfauthor={Alex Goessmann}
}

\usepackage{tikz}
\usepackage{graphicx}
\usepackage{float}
\usepackage{comment}
\usepackage{csquotes}
\usepackage{braket}
\usepackage{listings}
\usepackage{verbatim}
\usepackage{etoolbox}
\usepackage[utf8]{inputenc}
\usepackage[english]{babel}
\usepackage[T1]{fontenc}
\usepackage{titlesec}
\usepackage{fancyhdr}
\usepackage{bbm}
\usepackage{bm}
\usepackage{algpseudocode}
\usepackage{algorithm}
\usepackage{lipsum}

\newtheorem{theorem}{Theorem}
\newtheorem{lemma}{Lemma}
\newtheorem{corollary}{Corollary}
\newtheorem{definition}{Definition}
\newtheorem{example}{Example}

\DeclareUnicodeCharacter{FB01}{fi}
\usepackage[round]{natbib}
\usepackage{wasysym}

\usepackage{xcolor} 

\definecolor{codegreen}{rgb}{0,0.6,0}
\definecolor{codegray}{rgb}{0.5,0.5,0.5}
\definecolor{codepurple}{rgb}{0.58,0,0.82}
\definecolor{backcolour}{rgb}{0.95,0.95,0.92}

\newcommand{\pythoninline}[1]{\lstinline[language=Python, basicstyle=\ttfamily]|#1|}

\usepackage{authblk}

\newcommand{\python}{$\mathrm{python}$}
\newcommand{\tnreason}{$\mathrm{tnreason}$}

\newcommand{\spengine}{$\mathrm{engine}$}
\newcommand{\sprepresentation}{$\mathrm{representation}$}
\newcommand{\spreasoning}{$\mathrm{reasoning}$}
\newcommand{\spapplication}{$\mathrm{application}$}

\newcommand{\layeronespec}{\textbf{Layer 1}: Storage and manipulations}
\newcommand{\layertwospec}{\textbf{Layer 2}: Specification of workload}
\newcommand{\layerthreespec}{\textbf{Layer 3}: Applications in reasoning}

\newcommand{\curvertnreason}{2.0.0}

\newcommand{\ComputationActivationNetwork}{Computation-Activation Network}
\newcommand{\ComputationActivationNetworks}{Computation-Activation Networks}

\newcommand{\CompActNet}{CompActNet}
\newcommand{\CompActNets}{CompActNets}

\newcommand{\HybridLogicNetwork}{Hybrid Logic Network}

\newcommand{\HybridLogicNetworks}{Hybrid Logic Networks}

\newcommand{\firstOrderLogic}{first-order logic}

\newcommand{\defref}[1]{Def.~\ref{#1}}
\newcommand{\theref}[1]{Thm.~\ref{#1}}
\newcommand{\lemref}[1]{Lem.~\ref{#1}}

\newcommand{\algoref}[1]{Algorithm~\ref{#1}}

\newcommand{\exaref}[1]{Example~\ref{#1}}

\newcommand{\secref}[1]{Sect.~\ref{#1}}
\newcommand{\figref}[1]{Figure~\ref{#1}}

\newcommand{\ozset}{\{0,1\}}

\newcommand{\setwithout}[2]{#1\backslash#2}
\newcommand{\setexcept}[1]{\backslash#1}
\newcommand{\whileSymbol}{$\mathrm{while}$}

\newcommand{\modspace}{\,\,\,\mathrm{mod}\,\,}

{\begin{center}
     \begin{algorithmic}
         \hspace{1cm}}
{\end{algorithmic}\end{center}}

\newcommand{\defcols}{\,:\,} 
\newcommand{\wcols}{\,:\,} 
\newcommand{\ncond}{,\,} 

\newcommand{\andspace}{\quad\text{and}\quad}
\newcommand{\ifspace}{\text{if}\quad}

\newcommand{\iosepline}{\vspace{0.5em} \hrule \vspace{0.5em}}

\newcommand{\mutinfof}[2]{I\left(#1;#2\right)}
\newcommand{\condmutinfof}[3]{\mutinfof{#1}{#2|#3}}

\newcommand{\rr}{\mathbb{R}}
\newcommand{\nn}{\mathbb{N}}

\newcommand{\difofwrt}[2]{\frac{\partial #1}{\partial #2}}

\newcommand{\cardof}[1]{\left|#1\right|}

\newcommand{\imageof}[1]{\mathrm{im}\left(#1\right)}

\newcommand{\chainingfunction}{h}

\newcommand{\nonzerofunction}{\ones_{\neq0}}
\newcommand{\nonzeroof}[1]{\nonzerofunction\left(#1\right)}

\newcommand{\valof}[1]{\mathrm{val}\left(#1\right)}

\newcommand{\expof}[1]{\mathrm{exp}\left[#1\right]}
\newcommand{\probtensor}{\mathbb{P}}
\newcommand{\probtensorof}[1]{\probtensor^{#1}}
\newcommand{\probtensorofat}[2]{\probtensor^{#1}\left[#2\right]}

\newcommand{\probfamilyofat}[3]{\probofat{#1}{#2|#3}}

\newcommand{\independent}[2]{\left(#1\perp#2\right)}
\newcommand{\condindependent}[3]{\left(#1\perp#2\middle)\,\right|\,#3}

\newcommand{\probat}[1]{\probtensor\left[#1\right]}
\newcommand{\probof}[1]{\probtensor^{#1}}
\newcommand{\probofat}[2]{\probof{#1}\left[#2\right]}
\newcommand{\probwith}{\probat{\shortcatvariables}}
\newcommand{\probwithnodes}{\probat{\nodevariables}}

\newcommand{\condprobat}[2]{\mathbb{P}\big[#1\big|#2\big]}
\newcommand{\condprobof}[2]{\condprobat{#1}{#2}}

\newcommand{\margprobat}[1]{\probat{#1}}

\newcommand{\lnof}[1]{\ln \left[ #1 \right] }

\newcommand{\ones}{\mathbb{I}}

\newcommand{\onesat}[1]{\ones\left[#1\right]}

\newcommand{\oneswith}{\onesat{\shortcatvariables}}
\newcommand{\zeros}{0}
\newcommand{\zerosat}[1]{\zeros\left[#1\right]}
\newcommand{\identity}{\dirdelta}
\newcommand{\identityat}[1]{\identity\left[#1\right]}

\newcommand{\restrictionofto}[2]{{#1}|_{#2}}

\newcommand{\hardlegset}{A}
\newcommand{\hardparam}{(\hardlegset,\headindexof{\hardlegset})}
\newcommand{\hybridparam}{(\hardlegset,\headindexof{\hardlegset},\canparam)}
\newcommand{\hybridparamsetofdim}[1]{\mathcal{P}_{#1}}
\newcommand{\hybridparamset}{\hybridparamsetofdim{\seldim}}

\newcommand{\hlnparameters}{\hlnstat,\hybridparam}
\newcommand{\hlnat}[1]{\probofat{\hlnparameters}{#1}}
\newcommand{\hlnwith}{\hlnat{\shortcatvariables}}

\newcommand{\formula}{f}
\newcommand{\formulaof}[1]{\formula_{#1}}
\newcommand{\formulaat}[1]{\formula\left[#1\right]}
\newcommand{\formulaofat}[2]{\formulaof{#1}\left[#2\right]}

\newcommand{\formulawith}{\formulaat{\shortcatvariables}}

\newcommand{\hlnformulaof}[1]{\formula^{\hlnstat,(#1)}}
\newcommand{\hlnformulaparams}{\hardlegset,\headindexof{\hardlegset}}

\newcommand{\hlnformula}{\hlnformulaof{\hlnformulaparams}} 
\newcommand{\hlnformulaat}[1]{\hlnformula\left[#1\right]}
\newcommand{\hlnformulawith}{\hlnformulaat{\shortcatvariables}}

\newcommand{\formulavar}{\headvariableof{\formula}}

\newcommand{\enumformula}{\formulaof{\selindex}}
\newcommand{\enumformulaat}[1]{\enumformula\left[#1\right]}

\newcommand{\exformula}{\formula}

\newcommand{\exformulaat}[1]{\exformula\left[#1\right]}

\newcommand{\secexformula}{h} 

\newcommand{\secexformulaat}[1]{\secexformula\left[#1\right]}

\newcommand{\variableset}{A} 

\newcommand{\parspace}{\rr^{\seldim}}

\newcommand{\atomorder}{d}

\newcommand{\atomenumerator}{k}

\newcommand{\atomenumeratorin}{\atomenumerator\in[\atomorder]}

\newcommand{\atomlegindex}{\catindex}

\newcommand{\atomlegindexof}[1]{\atomlegindex_{#1}}

\newcommand{\loss}{\mathcal{L}_{\datamap}}
\newcommand{\lossof}[1]{\loss\left(#1\right)}

\newcommand{\partitionfunction}{\mathcal{Z}}

\newcommand{\partitionfunctionof}[1]{\partitionfunction\left(#1\right)}

\newcommand{\kb}{\mathcal{KB}}

\newcommand{\kbat}[1]{\kb\left[#1\right]}
\newcommand{\kbwith}{\kbat{\shortcatvariables}}

\newcommand{\seckb}{\tilde{\kb}}

\newcommand{\elformat}{\mathrm{EL}}
\newcommand{\cpformat}{\mathrm{CP}}

\newcommand{\ttformat}{\mathrm{TT}}
\newcommand{\maxformat}{\mathrm{MAX}}

\newcommand{\elgraph}{\elformat}
\newcommand{\maxgraph}{\maxformat}

\newcommand{\tnset}{\mathcal{T}}
\newcommand{\tnsetof}[1]{\tnset^{#1}} 

\newcommand{\nodevariables}{\catvariableof{\nodes}}

\newcommand{\secnodevariables}{\catvariableof{\secnodes}}

\newcommand{\extnetasset}{\left\{\hypercoreofat{\edge}{\catvariableof{\edge}}\wcols\edgein\right\}}

\newcommand{\exrandom}{\catvariableof{0}}
\newcommand{\secexrandom}{\catvariableof{1}}
\newcommand{\thirdexrandom}{\catvariableof{2}}

\newcommand{\exranddim}{\catdimof{0}}

\newcommand{\secexranddim}{\catdimof{1}}

\newcommand{\expdistof}[1]{\probtensorof{#1}}
\newcommand{\expdistofat}[2]{\expdistof{#1}[#2]}
\newcommand{\expdist}{\probtensorof{(\sstat,\canparam,\basemeasure)}}
\newcommand{\expdistwith}{\probtensorofat{(\sstat,\canparam,\basemeasure)}{\shortcatvariables}}
\newcommand{\expdistat}[1]{\expdist\left[#1\right]}

\newcommand{\genstatshortcatencoding}[1]{\gamma^{\sstat}\left[\indexedshortcatvariables,\selvariable\right]}

\newcommand{\trivbm}{\ones}

\newcommand{\sstat}{t}
\newcommand{\sstatat}[1]{\sstat\left(#1\right)}
\newcommand{\sstatof}[1]{\sstat^{#1}}

\newcommand{\hlnstat}{\sstat}

\newcommand{\sstatcoordinateof}[1]{\sstat_{#1}}

\newcommand{\sstatcoordinateofat}[2]{\sstat_{#1}\left[#2\right]}

\newcommand{\hlnstatccwith}{\bencodingofat{\hlnstat}{\headvariables,\shortcatvariables}}

\newcommand{\bencsstatat}[1]{\bencodingof{\sstat}\left[#1\right]}
\newcommand{\bencsstatwith}{\bencsstatat{\headvariables,\shortcatvariables}}

\newcommand{\datamean}{\meanparamof{\datamap}}
\newcommand{\datameanat}[1]{\datamean\left[#1\right]}
\newcommand{\datameanwith}{\datameanat{\selvariable}}

\newcommand{\graph}{\mathcal{G}}
\newcommand{\graphof}[1]{\graph^{#1}}

\newcommand{\nodes}{\mathcal{V}}
\newcommand{\nodesof}[1]{\nodes^{#1}}
\newcommand{\innodes}{\nodesof{\mathrm{in}}}
\newcommand{\outnodes}{\nodesof{\mathrm{out}}}

\newcommand{\incomingnodes}{\edge^{\mathrm{in}}}
\newcommand{\outgoingnodes}{\edge^{\mathrm{out}}}

\newcommand{\nodesa}{A}
\newcommand{\nodesb}{B}
\newcommand{\nodesc}{C}

\newcommand{\nodea}{a}
\newcommand{\nodeb}{b}

\newcommand{\secnodes}{\mathcal{U}}

\newcommand{\thirdnodes}{\mathcal{W}}

\newcommand{\node}{v}
\newcommand{\nodein}{\node\in\nodes}
\newcommand{\secnode}{\tilde{\node}}

\newcommand{\edges}{\mathcal{E}}

\newcommand{\edge}{e}
\newcommand{\edgeof}[1]{\edge_{#1}}
\newcommand{\secedge}{\tilde{\edge}}

\newcommand{\edgein}{\edge\in\edges}

\newcommand{\hypercore}{\tau}
\newcommand{\hypercoreat}[1]{\hypercore\left[#1\right]}
\newcommand{\hypercorewith}{\hypercoreat{\shortcatvariables}}
\newcommand{\hypercorewithnodes}{\hypercoreat{\nodevariables}}

\newcommand{\edgehypercorewith}{\hypercoreofat{\edge}{\catvariableof{\edge}}}

\newcommand{\hypercoreof}[1]{\hypercore^{#1}}
\newcommand{\hypercoreofat}[2]{\hypercoreof{#1}\left[#2\right]}
\newcommand{\sechypercore}{\tilde{\hypercore}}
\newcommand{\sechypercoreof}[1]{\sechypercore^{#1}}
\newcommand{\sechypercoreofat}[2]{\sechypercore^{#1}\left[#2\right]}
\newcommand{\sechypercoreat}[1]{\sechypercore\left[#1\right]}

\newcommand{\exfunction}{q}
\newcommand{\exfunctionof}[1]{\exfunction_{#1}}

\newcommand{\secexfunction}{g}

\newcommand{\secexfunctionof}[1]{\secexfunction^{#1}}

\newcommand{\statesetfunction}{\exfunction}
\newcommand{\statesetfunctionev}[1]{\statesetfunction\left(#1\right)}

\newcommand{\scheduler}{S}

\newcommand{\dirovedges}{\edges^{\rightarrow}}

\newcommand{\sedge}{\edgeof{0}} 
\newcommand{\redge}{\edgeof{1}} 
\newcommand{\secsedge}{\edgeof{2}}
\newcommand{\thirdsedge}{\edgeof{3}}

\newcommand{\preedgesetwrt}[2]{\edges^{\rightarrow(#1,#2)}}
\newcommand{\preedgeset}{\preedgesetwrt{\sedge}{\redge}}

\newcommand{\notucon}{\lnot}

\newcommand{\woneoplus}{\bigoplus^{(1)}}
\newcommand{\modsumsymbol}{\tilde{+}}
\newcommand{\marysumsymbol}{+^{\catdim}}

\newcommand{\indexinterpretation}{I}
\newcommand{\indexinterpretationof}[1]{\indexinterpretation_{#1}}
\newcommand{\indexinterpretationat}[1]{\indexinterpretation(#1)}
\newcommand{\indexinterpretationofat}[2]{\indexinterpretationof{#1}(#2)}

\newcommand{\invindexinterpretation}{\indexinterpretation^{-1}}
\newcommand{\invindexinterpretationof}[1]{\indexinterpretation_{#1}^{-1}}
\newcommand{\invindexinterpretationat}[1]{\invindexinterpretation(#1)}
\newcommand{\invindexinterpretationofat}[2]{\invindexinterpretationof{#1}(#2)}

\newcommand{\arbset}{\mathcal{U}}
\newcommand{\arbsetof}[1]{\arbset^{#1}}
\newcommand{\arbelement}{u}

\newcommand{\insymbol}{\mathrm{in}}
\newcommand{\outsymbol}{\mathrm{out}}
\newcommand{\inset}{\arbsetof{\insymbol}}
\newcommand{\outset}{\arbsetof{\outsymbol}}

\newcommand{\sudokunum}{n}
\newcommand{\sudokustartevidence}{E^{\mathrm{start}}}
\newcommand{\sudokukbof}[1]{\kb^{#1}}

\newcommand{\contractionof}[2]{\left\langle #1\right\rangle_{\left[ #2 \right]}}

\newcommand{\breakablecontractionof}[2]{\big\langle #1 \big\rangle_{\big[ #2 \big]}}
\newcommand{\contraction}[1]{\contractionof{#1}{\varnothing}}
\newcommand{\normalizationofwrt}[3]{\left\langle #1\right\rangle_{\left[ #2 | #3 \right]}}

\newcommand{\normalizationof}[2]{\normalizationofwrt{#1}{#2}{\varnothing}}

\newcommand{\bencodingof}[1]{\beta^{#1}}
\newcommand{\bencodingofat}[2]{\bencodingof{#1}\left[#2\right]}

\newcommand{\softactsymbol}{\alpha}
\newcommand{\softactsymbolof}[1]{\softactsymbol^{#1}}
\newcommand{\softactsymbolofat}[2]{\softactsymbol^{#1}\left[#2\right]}
\newcommand{\softacttensor}{\softactsymbol^\canparam}
\newcommand{\softacttensorat}[1]{\softacttensor\left[#1\right]}
\newcommand{\softacttensorwith}{\softacttensorat{\headvariables}}
\newcommand{\softactleg}{\softactsymbolof{\selindex,\canparam}}
\newcommand{\softactlegat}[1]{\softactleg\left[#1\right]}
\newcommand{\softactlegwith}{\softactsymbolofat{\selindex,\canparam}{\headvariableof{\selindex}}}

\newcommand{\dirdelta}{\delta}
\newcommand{\dirdeltaof}[1]{\dirdelta^{#1}}
\newcommand{\dirdeltaofat}[2]{\dirdeltaof{#1}\left[#2\right]}
\newcommand{\dirdeltawith}{\dirdeltaofat{[\catorder],\catdim}{\shortcatvariables}}

\newcommand{\onehotmap}{\epsilon}
\newcommand{\onehotmapof}[1]{\onehotmap_{#1}}
\newcommand{\onehotmapofat}[2]{\onehotmap_{#1}\left[#2\right]}

\newcommand{\tbasis}{\onehotmapof{1}}
\newcommand{\tbasisat}[1]{\tbasis\left[#1\right]}
\newcommand{\fbasis}{\onehotmapof{0}}
\newcommand{\fbasisat}[1]{\fbasis\left[#1\right]}

\newcommand{\canparam}{\theta}

\newcommand{\canparamat}[1]{\canparam\left[#1\right]}

\newcommand{\canparamwith}{\canparamat{\selvariable}}

\newcommand{\canparamwithin}{\canparamwith\in\parspace}

\newcommand{\hardactsymbol}{\kappa}
\newcommand{\hardactsymbolof}[1]{\hardactsymbol^{#1}}

\newcommand{\hardacttensor}{\hardactsymbolof{\hardparam}}
\newcommand{\hardacttensorwith}{\hardacttensor\left[\headvariables\right]}

\newcommand{\meanparam}{\mu}
\newcommand{\meanparamof}[1]{\meanparam_{#1}}
\newcommand{\meanparamat}[1]{\meanparam\left[#1\right]}

\newcommand{\basemeasure}{\nu}

\newcommand{\basemeasureat}[1]{\basemeasure\left[#1\right]}
\newcommand{\basemeasurewith}{\basemeasureat{\shortcatvariables}}

\newcommand{\acttensor}{\xi} 
\newcommand{\acttensorof}[1]{\acttensor^{#1}}
\newcommand{\acttensorofat}[2]{\acttensor^{#1}\left[#2\right]}
\newcommand{\acttensorat}[1]{\acttensor\left[#1\right]}
\newcommand{\acttensorwith}{\acttensorat{\headvariables}}
\newcommand{\acttensorleg}{\acttensorof{\selindex}}
\newcommand{\acttensorlegat}[1]{\acttensorleg\left[#1\right]}
\newcommand{\acttensorlegwith}{\acttensorleg\left[\headvariableof{\selindex}\right]}
\newcommand{\paracttensor}{\acttensorof{\hybridparam}}
\newcommand{\paracttensorwith}{\acttensorofat{\hybridparam}{\headvariables}}

\newcommand{\tnet}{\tau}
\newcommand{\tnetof}[1]{\tnet^{#1}}
\newcommand{\tnetofat}[2]{\tnetof{#1}\left[#2\right]}

\newcommand{\extnet}{\tnetof{\graph}}

\newcommand{\messagesymbol}{\chi}
\newcommand{\mesfromto}[2]{\messagesymbol_{#1 \rightarrow #2}}
\newcommand{\mesfromtowith}[2]{\messagesymbol_{#1 \rightarrow #2}\left[\catvariableof{#1 \cap #2}\right]}
\newcommand{\mesfromtoat}[3]{\mesfromto{#1}{#2}\left[#3\right]}
\newcommand{\messagewith}{\mesfromtowith{\sedge}{\redge}}

\newcommand{\realizabledistsof}[1]{\Lambda^{#1}}

\newcommand{\elrealizabledistsof}[1]{\realizabledistsof{#1,\elformat}}

\newcommand{\cansof}[1]{\Lambda^{#1}}

\newcommand{\selvariable}{L} 
\newcommand{\seldim}{p}
\newcommand{\selindex}{\ell}

\newcommand{\catvariable}{X} 
\newcommand{\catdim}{m}
\newcommand{\catindex}{x} 
\newcommand{\catenumerator}{\atomenumerator}
\newcommand{\catorder}{\atomorder}

\newcommand{\headvariable}{Y} 
\newcommand{\headdim}{n}
\newcommand{\headindex}{y}

\newcommand{\datdim}{m}
\newcommand{\datindex}{j}

\newcommand{\selvariableof}[1]{\selvariable_{#1}}
\newcommand{\catvariableof}[1]{\catvariable_{#1}}
\newcommand{\headvariableof}[1]{\headvariable_{#1}}

\newcommand{\catvariablelist}{\catvariableof{0},\ldots,\catvariableof{\atomorder-1}}

\newcommand{\shortcatvariablelist}{\catvariableof{[\atomorder]}}

\newcommand{\shortcatindices}{\catindexof{[\catorder]}}

\newcommand{\shortheadindices}{\headindexof{[\seldim]}}

\newcommand{\catdimof}[1]{\catdim_{#1}}
\newcommand{\headdimof}[1]{\headdim_{#1}}

\newcommand{\selindexof}[1]{\selindex_{#1}}
\newcommand{\catindexof}[1]{\catindex_{#1}} 
\newcommand{\headindexof}[1]{\headindex_{#1}}

\newcommand{\selindexin}{\selindex\in[\seldim]}

\newcommand{\datindexin}{\datindex\in[\datdim]}

\newcommand{\catenumeratorin}{\catenumerator\in[\catorder]}

\newcommand{\indexedcatvariableof}[1]{\catvariableof{#1}=\catindexof{#1}}
\newcommand{\indexedselvariableof}[1]{\selvariableof{#1}=\selindexof{#1}}
\newcommand{\indexedheadvariableof}[1]{\headvariableof{#1}=\headindexof{#1}}

\newcommand{\indexedcatvariable}{\indexedcatvariableof{}}
\newcommand{\indexedselvariable}{\indexedselvariableof{}}

\newcommand{\catstatesof}[1]{[\catdimof{#1}]}

\newcommand{\catspace}{\bigotimes_{\atomenumeratorin} \rr^{\catdimof{\atomenumerator}}}

\newcommand{\datanum}{\datdim}

\newcommand{\catvariables}{\catvariablelist}
\newcommand{\shortcatvariables}{\shortcatvariablelist}
\newcommand{\indexedshortcatvariables}{\shortcatvariables=\shortcatindices}
\newcommand{\shortcatindicesin}{\shortcatindices\in\facstates}

\newcommand{\headvariables}{\headvariableof{[\seldim]}}

\newcommand{\secselindex}{\tilde{\selindex}}

\newcommand{\nodestatesof}[1]{\bigtimes_{\node\in#1}\catstatesof{\node}}
\newcommand{\atomstates}{\bigtimes_{\atomenumeratorin}[2]}

\newcommand{\facstates}{\bigtimes_{\atomenumeratorin}\catstatesof{\atomenumerator}}

\newcommand{\facspace}{\catspace}

\newcommand{\seccatenumerator}{\tilde{\catenumerator}}

\newcommand{\seccatvariable}{Y} 
\newcommand{\seccatindex}{y}
\newcommand{\seccatorder}{p} 

\newcommand{\thirdcatvariable}{Z}
\newcommand{\thirdcatvariableof}[1]{\thirdcatvariable_{#1}}
\newcommand{\thirdcatindex}{z}
\newcommand{\thirdcatindexof}[1]{\thirdcatindex_{#1}}
\newcommand{\indexedthirdcatvariable}{\thirdcatvariable=\thirdcatindex}
\newcommand{\thirdcatdim}{n}

\newcommand{\seccatvariableof}[1]{\seccatvariable_{#1}}

\newcommand{\secshortcatvariables}{\seccatvariableof{[\seccatorder]}}

\newcommand{\seccatindexof}[1]{\seccatindex_{#1}}

\newcommand{\tildecatvariable}{\tilde{\catvariable}}
\newcommand{\tildecatvariableof}[1]{\tildecatvariable_{#1}}
\newcommand{\tildecatindex}{\tilde{\catindex}}
\newcommand{\tildecatindexof}[1]{\tildecatindex_{#1}}

\newcommand{\catindicesof}[1]{{\catindexof{0}^{#1},\ldots,\catindexof{\atomorder-1}^{#1}}}

\newcommand{\secdecvariable}{\tilde{\decvariable}}
\newcommand{\secdecvariableof}[1]{\secdecvariable_{#1}}

\newcommand{\datamap}{D}

\newcommand{\dataset}{\left((\catindicesof{\datindex})\,:\,\datindexin\right)}

\newcommand{\decvariable}{I}
\newcommand{\decvariableof}[1]{\decvariable_{#1}}
\newcommand{\decindex}{i} 

\newcommand{\decdim}{n}
\newcommand{\decdimof}[1]{\decdim_{#1}}

\newcommand{\indexeddecvariable}{\decvariable=\decindex}

\newcommand{\dotsize}{0.15cm}
\newcommand{\nodeminsize}{0.8cm}
\newcommand{\nodegrayscale}{gray!50}
\newcommand{\colorlabelsize}{\tiny}
\newcommand{\corelabelsize}{\small}

\newcommand{\probcolor}{purple!60}
\newcommand{\concolor}{cyan}

\newcommand{\newmessagecolor}{blue}

\usetikzlibrary{arrows.meta}
\usetikzlibrary{shapes,positioning}
\usetikzlibrary{decorations.markings}
\usetikzlibrary{calc}
\usetikzlibrary{matrix}

\tikzset{
    midarrow/.style={
        postaction={decorate},
        decoration={markings, mark=at position 0.5 with {\arrow{>}}}
    },
    midbackarrow/.style={
        postaction={decorate},
        decoration={markings, mark=at position 0.5 with {\arrow{<}}}
    },
    ->-/.style={midarrow},
    -<-/.style={midbackarrow}
}

\newcommand{\shortminus}{\scalebox{0.4}[1.0]{$-$}}

\newcommand{\drawvariabledot}[2]{
    \draw[fill] (#1,#2) circle (\dotsize);
}

\newcommand{\drawatomindices}[2]{
    \begin{scope}
        [shift={(#1,#2)}]
        \draw[-<-] (0,1)--(0,-1) node[midway,left] {\colorlabelsize $\catvariableof{0}$};
        \draw[-<-] (1.5,1)--(1.5,-1) node[midway,left] {\colorlabelsize $\catvariableof{1}$};
        \node[anchor=center] (text) at (3,0) {$\cdots$};
        \draw[-<-] (4,1)--(4,-1) node[midway,right] {\colorlabelsize $\catvariableof{\atomorder\shortminus1}$};
    \end{scope}
}

\newcommand{\drawatomcore}[3]{
    \begin{scope}
        [shift={(#1,#2)}]
        \draw (-1,-1) rectangle (5,-3);
        \node[anchor=center] (text) at (2,-2) {#3};
    \end{scope}
}

\newenvironment{sudoku4x4}{%
    \begin{tikzpicture}[
        sudokucell/.style={
            minimum size=0.6cm,
            draw=gray!50,
            anchor=center
        },
        sudokumatrix/.style={
            matrix of nodes,
            nodes=sudokucell,
            inner sep=0pt,
            row sep=-\pgflinewidth,
            column sep=-\pgflinewidth,
            draw=black,
            very thick,
        }
    ]
}{%
    \end{tikzpicture}
}

\begin{document}

    \title{A tensor network formalism for neuro-symbolic AI}
    \author{Alex Goessmann\footnote{alex.goessmann@wias-berlin.de}}
    \author{Janina Schütte}
    \author{Maximilian Fröhlich}
    \author{Martin Eigel}
    \affil{Weierstrass Institute of Applied Analysis and Stochastics \\
    Anton-Wilhelm-Amo-Straße 39\\
    Berlin, 10117, Germany\\ 
    \ \\
    *alex.goessmann@wias-berlin.de}

    \maketitle

    \begin{abstract}
        The unification of neural and symbolic approaches to artificial intelligence remains a central open challenge.
In this work, we introduce a tensor network formalism, which captures sparsity principles originating in the different approaches in tensor decompositions.
In particular, we describe a basis encoding scheme for functions and model neural decompositions as tensor decompositions.
The proposed formalism can be applied to represent logical formulas and probability distributions as structured tensor decompositions.
This unified treatment identifies tensor network contractions as a fundamental inference class and formulates efficiently scaling reasoning algorithms, originating from probability theory and propositional logic, as contraction message passing schemes.
The framework enables the definition and training of hybrid logical and probabilistic models, which we call \HybridLogicNetworks{}.
The theoretical concepts are accompanied by the \python{} library \tnreason{}, which enables the implementation and practical use of the proposed architectures.
    \end{abstract}

    \section{Introduction}


Modern artificial intelligence is dominated by large-scale neural models that excel at various tasks but mostly remain black-boxes. 
While these models offer adaptability, the two main concerns when integrating these architectures into safety-critical processes are reliability and explainability.
To match these demands, artificial intelligence has followed symbolic paradigms, including probabilistic and logical approaches.
However, these paradigms have been mostly neglected due to the success of black-box neural models.
The logical tradition of artificial intelligence, historically motivated by the resemblance of human thought to formal logics \cite{mccarthy_programs_1959}, offers explicit structures and human-readable inference.
However, the main problem hindering the success of this classical approach is the inability of classical \firstOrderLogic{} to handle uncertainty or scale to complex real-world data.
Probabilistic graphical models~\cite{pearl_probabilistic_1988,koller_probabilistic_2009} provide insights based on encoded variable independences and causality~\cite{pearl_causality_2009}. 
While probabilistic models and Statistical Relational AI~\cite{nickel_review_2016,getoor_introduction_2019} have improved uncertainty handling, bridging these paradigms remains the central goal of \emph{Neuro-Symbolic AI}~\cite{hochreiter_toward_2022, sarker_neuro-symbolic_2022, colelough_neuro-symbolic_2024}.
The field seeks a single, mathematically coherent framework combining structural clarity with neural adaptability.
Although progress has been made, for example with Markov Logic Networks~\cite{richardson_markov_2006}, a fully unified substrate that treats logical and probabilistic inference as instances of the same operation is still missing. 

In this work, we propose to fill the gap between the probabilistic, neural, and logical paradigms with \emph{tensor networks} in a framework called \tnreason{}. 
Tensor spaces capture both the semantics of logical formulas (by boolean tensors) and probability distributions (by normalized non-negative tensors).
This abstraction eliminates the traditional divide between symbolic and neural representations: logical inference, probabilistic computations, and neural inference become different instances of the same underlying operation.
As naive tensors are prone to the curse of dimensionality, we turn to distributed representation schemes by tensor networks.
We show that fundamental sparsity principles of neural and symbolic AI, such as conditional independence, the existence of sufficient statistics, and neural model decomposition, are equivalent to tensor network decompositions.
Moreover, we identify tensor network contraction as the fundamental operation underlying inference tasks, such as computing marginal distributions and deciding entailment.
While these contractions are in general computationally hard, efficient schemes to perform inference are known as message passing schemes.
These algorithms have appeared in different communities under names such as belief propagation \cite{pearl_probabilistic_1988,mezard_information_2009} and constraint propagation \cite{mackworth_consistency_1977}.
We review these schemes based on our tensor network formalism.

To capture both logical and probabilistic models, and exploit their neural decompositions, we introduce \emph{\ComputationActivationNetworks{} (\CompActNets{})}, an expressive tensor network architecture.
The architecture consists of two complementary subarchitectures that serve the purpose of computation and activation, respectively.
The \emph{computation network} prepares auxiliary hidden variables with deterministic dependence on the main variables.
We interpret the network as a distributed computation scheme of functions describing this dependence, whose decomposition is related to sparsity concepts of the corresponding functions. 
These auxiliary variables represent logical formulas or more generic statistics in different contexts.
The \emph{activation network} then assigns numerical values to the states of those auxiliary variables, and in this way activates the variables to represent factors of a model.
Logical models emerge when the activation network is a boolean tensor, probabilistic exponential families when they are elementary positive-valued tensors, and hybrid models in the most general cases.

\subsection{Related works}

Historically rooted in quantum many-body physics ~\cite{white_density-matrix_1993}, tensor networks found their first major success with Matrix Product States (MPS), originally developed to efficiently capture the quantum dynamics and ground states of one-dimensional spin chains ~\cite{affleck_rigorous_1987}.
This format remains a standard tool in the field, with recent contributions refining it for tasks such as large-scale stochastic simulations and variational circuit operations ~\cite{sander_large-scale_2025, sander_quantum_2025}.
To address the topological constraints of MPS, the landscape of architectures was subsequently expanded to include Projected Entangled Pair States (PEPS) for two-dimensional lattices and the Multi-scale Entanglement Renormalization Ansatz (MERA), which utilizes a hierarchical geometry to represent scale-invariant critical systems and has recently been adapted for simulating quantum systems ~\cite{orus_tensor_2019, berezutskii_simulating_2025}.

Beyond the quantum realm, these formats have been successfully adapted to applied mathematics, particularly for solving high-dimensional parametric PDEs~\cite{eigel_variational_2019,eigel_adaptive_2020,dolgov_hybrid_2019,dolgov_polynomial_2015,trunschke_weighted_2025}, sampling problems and approximation of the Hamilton-Jacobi-Belman equation~\cite{gruhlke_reverse_2026,eigel_dynamical_2023,dolgov_data-driven_2023,cui_deep_2022}, modeling complex continuous fields and learning dynamical laws~\cite{hagemann_sampling_2025, eigel_adaptive_2017, goessmann_tensor_2020, lubich_dynamical_2013}.
Furthermore, they exhibit properties helpful for handling these high-dimensional spaces, such as restricted isometry properties~\cite{goessmann_uniform_2021}.
Recent advancements have demonstrated the efficacy of these methods in capturing multiscale phenomena in fluid dynamics and turbulence, proving that the tensor network formalism offers a robust alternative to classical numerical schemes ~\cite{gourianov_tensor_2025}.

The unification of neural, symbolic, and probabilistic approaches to interpretable model architectures has been a long-standing aim of Neuro-Symbolic AI.
A central goal is to achieve \emph{intrinsic explainability}, which, unlike post-hoc interpretations analysing input influence after training~\cite{lipton_mythos_2018,barredo_arrieta_explainable_2020}, aims at explainability of the architecture itself.
Early connectionist approaches~\cite{towell_knowledge-based_1994,avila_garcez_connectionist_1999} towards Neuro-Symbolic AI focus on embedding logical rules into neural connectivity.
Further, fruitful relations with statistical relational learning have been identified~\cite{marra_statistical_2024}.

\emph{Tensor networks} have recently gained interest as a unifying language for AI, framed by Logical Tensor Networks \cite{badreddine_logic_2022} and Tensor Logic \cite{domingos_tensor_2025}.
Furthermore, the MeLoCoToN approach \cite{ali_explicit_2025} applies tensor network architectures similar to \CompActNets{} in combinatorial optimization problems.
Specifically, tensor networks have emerged as a highly efficient mathematical framework for handling data in high-dimensional spaces, effectively circumventing the "curse of dimensionality" that typically plagues grid-based methods ~\cite{hackbusch_tensor_2012}.
By decomposing high-order tensors into networks of low-rank components, these structures reduce the storage and computational complexity from exponential to polynomial with respect to the dimension ~\cite{oseledets_tensor-train_2011, hackbusch_new_2009, hitchcock_expression_1927}.

\subsection{Structure of the paper}

\begin{figure}
    \centering
    \newcommand{\sketchtextsize}{\fontsize{8pt}{8pt}\selectfont}

\begin{tikzpicture}[yscale=1, xscale=0.925]

    \node[anchor=center] at (0,0.1) {\bf Neural paradigm};
    \draw[dashed] (-5,0.5) rectangle (5,-2.75);
    \node[anchor=center] at (0,-0.5) {
        \sketchtextsize Decomposition of functions into neurons (\defref{def:decompositionHypergraph}, \theref{the:functionDecompositionRep})};

    \node[anchor=west] at (-7.5,-5.1) {\bf Probabilistic paradigm};
    \draw[dashed] (-8,-1) rectangle (2,-5.5);
    \node[anchor=west, align=left] at (-7.75,-3.25) {
        \sketchtextsize Independence and related concepts \\[-3pt]
        \sketchtextsize (\defref{def:independence}, \defref{def:condIndependence}, \defref{def:sufStatistic})};
    \node[anchor=west, align=left] at (-7.75,-4.25) {
        \sketchtextsize Marginal and conditional distributions \\[-3pt]
        \sketchtextsize (\defref{def:marginalConditionalDistribution})};
    
    \node[anchor=east] at (7.5,-4.85) {\bf Logical paradigm};
    \draw[dashed] (8,-1.25) rectangle (-2,-5.25);
    \node[anchor=east, align=right] at (7.75,-3.5) {
        \sketchtextsize Semantics by boolean tensors \\[-3pt]
        \sketchtextsize (\defref{def:formulas})};
    \node[anchor=east, align=right] at (7.75,-4.25) {
        \sketchtextsize Entailment (\defref{def:logicalEntailment})};

    \node[anchor=center, align=center] at (-3.5,-2) {
        \sketchtextsize Graphical models \\[-3pt]
        \sketchtextsize (\defref{def:markovNetwork})};

    \node[anchor=center, align=center] at (3.5,-2) {
        \sketchtextsize Syntactical \\[-3pt]
        \sketchtextsize decompositions \\[-3pt]
        \sketchtextsize (\defref{def:syntacticalDecomposition})};

    \node[anchor=center, align=center] at (0,-2) {
        \sketchtextsize \CompActNets{} \\[-3pt]
        \sketchtextsize (\defref{def:compActNets})};

    \node[anchor=center, align=center] at (0,-4) {
        \sketchtextsize \HybridLogicNetworks{} \\[-3pt]
        \sketchtextsize (\defref{def:hybridLogicNetwork})};

\end{tikzpicture}
    \caption{Sketch of the concepts in the neural, probabilistic and logical paradigms, which we define based on tensor network decompositions and contractions.}\label{fig:paradigmsSketch}
\end{figure}
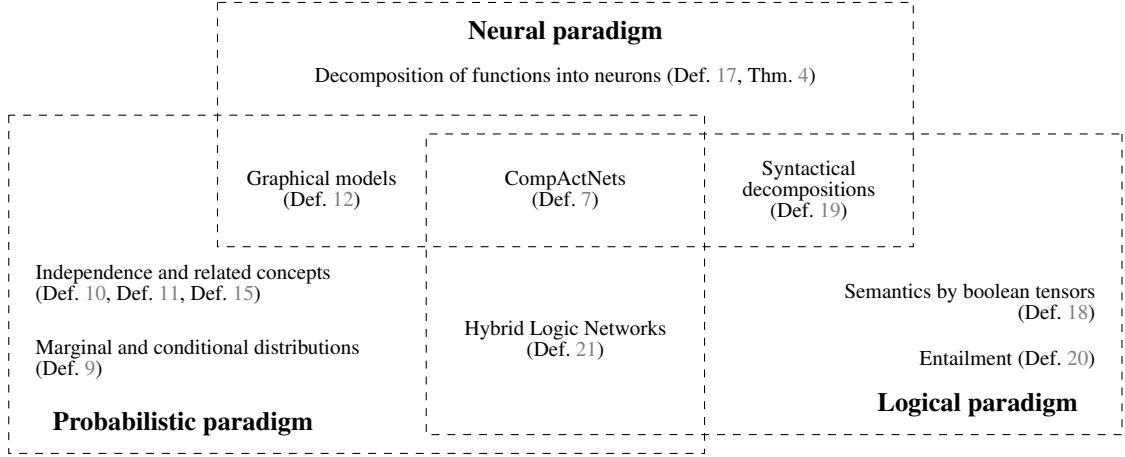

The paper is organized as follows.
In Section~\ref{sec:notation} we introduce the basic concepts and notation for categorical variables, tensors, and tensor networks, which are applied in the following sections.
Sections~\ref{sec:probPar}, \ref{sec:neurPar}, and \ref{sec:logPar} anchor the tensor network formalism in the basic paradigms of artificial intelligence (see \figref{fig:paradigmsSketch}). 
The \emph{probabilistic paradigm} is discussed in Section~\ref{sec:probPar}, where we in particular show that concepts of independence, graphical models, and sufficient statistics correspond to specific tensor network decompositions.
Section~\ref{sec:neurPar} is dedicated to the \emph{neural paradigm}, where we show that generic function decompositions have analogous tensor network representations.
Section~\ref{sec:logPar} turns to the \emph{logical paradigm}, where we study tensor equivalents of propositional formulas, knowledge bases, and entailment.
In Section~\ref{sec:hln} we present \emph{\HybridLogicNetworks{}} as an application of the unified tensor network formalism, combining logical and probabilistic models.
We briefly discuss the implementation of these concepts in our open source \python{} package \tnreason{} in Section~\ref{sec:implementation}, before concluding the paper in Section~\ref{sec:conclusion}.

    \section{Foundations}\label{sec:notation}

In this section, we introduce the hypergraph-based tensor network formalism and define the most general tensor network architecture of \CompActNets{} based on this formalism.

\subsection{Tensors}
\label{sec:tensors}

Tensors are multiway arrays and a generalization of vectors and matrices to higher orders.
We first provide a formal definition as real maps from index sets enumerating the coordinates of vectors, matrices, and higher-order tensors.
To ease the notation, we abbreviate sets as $[\catorder]=\{0,\ldots,\catorder-1\}$, tuples of state indices by $\catindexof{[\catorder]}=(\catindexof{0},\ldots,\catindexof{\catorder-1})$ and tuples of variables by $\catvariableof{[\catorder]}=(\catvariableof{0},\ldots,\catvariableof{\catorder-1})$.

\begin{definition}[Tensor]
    \label{def:tensor}
    For $\atomenumeratorin$, let $\catdimof{\atomenumerator}\in\nn$ and let $\catvariableof{\atomenumerator}$ be categorical variables taking values in $[\catdimof{\catenumerator}]$.
    A tensor $\hypercoreat{\catvariables}$ of order $\catorder$ and leg dimensions $\catdimof{0},\dots,\catdimof{\atomorder-1}$ is defined through its coordinates
    \begin{align*}
        \hypercoreat{\indexedshortcatvariables} =
        \hypercoreat{\indexedcatvariableof{0},\ldots,\indexedcatvariableof{\catorder-1}} \in \rr
    \end{align*}
    for index tuples
    \begin{align*}
        \shortcatindices=(\catindexof{0},\ldots,\catindexof{\catorder-1}) \in \facstates \, .
    \end{align*}
    Tensors $\hypercoreat{\shortcatvariables}$ are elements of the tensor space
    \begin{align*}
        \bigotimes_{\atomenumeratorin} \rr^{\catdimof{\atomenumerator}} \,,
    \end{align*}
    which is a linear space, enriched with the operations of coordinate-wise summation and scalar multiplication.
    We call a tensor $\hypercoreat{\shortcatvariables}$ boolean, when all coordinates are in $\{0,1\}$, and positive, when all coordinates are greater than $0$.
\end{definition}

We introduced tensors here in a non-canonical way based on categorical variables assigned to their axes.
While this may look like syntactic sugar at this point, it allows us to define contractions without further specification of axes, based on comparisons of shared variables.
We occasionally also allow for variables $\catvariable$ taking values in infinite sets such as $\rr$, in which case we denote the set of values to a variable by $\valof{\catvariable}$.


\begin{example}[Delta tensor]\label{exa:diracDelta}
    Given a tuple of variables $\shortcatvariables=(\catvariables)$ with identical dimension $\catdim$, where $\catorder\geq 1$, the delta tensor is the element
    \begin{align*}
        \dirdeltawith \in \bigotimes_{\catenumeratorin} \rr^{\catdim}
    \end{align*}
    with coordinates
    \begin{align*}
        \dirdeltaofat{[\catorder],\catdim}{\indexedshortcatvariables} =
        \begin{cases}
            1 \quad & \ifspace \catindexof{0} = \ldots = \catindexof{\catorder-1} \\
            0 & \text{else}
        \end{cases} \, .
    \end{align*}
    We depict this tensor by black dots, which sometimes appears as auxiliary elements in tensor network diagrams (see e.g. \figref{fig:contraction}).
    For $\catorder=1$, the delta tensor is the trivial vector, whose coordinates are constantly $1$, which we denote by $\onesat{\catvariable}$.
\end{example}

\subsection{Tensor networks and contractions}

We use a standard visualization of tensors (dating back to \cite{penrose_spinors_1987}) by blocks with lines depicting the axes of the tensor.
Additionally, we assign to each axis of the tensor the corresponding variable $\catvariableof{\atomenumerator}$:
\begin{center}
    \begin{tikzpicture}[scale=0.35,thick] 

\begin{scope}[shift={(16,2)}]

\draw (-1,-1) rectangle (5,-3);
\node[anchor=center] (text) at (2,-2) {\corelabelsize $\hypercore$};
\draw (0,-3)--(0,-5) node[midway,left] {\colorlabelsize $\catvariableof{0}$};
\draw (1.5,-3)--(1.5,-5) node[midway,left] {\colorlabelsize $\catvariableof{1}$};
\node[anchor=center] (text) at (3,-4) {$\cdots$};
\draw (4,-3)--(4,-5) node[midway,right] {\colorlabelsize $\catvariableof{\atomorder\shortminus1}$};

\end{scope}


\end{tikzpicture}
\end{center}
We now associate categorical variables with nodes of a hypergraph and tensors with hyperedges, which are arbitrary subsets of nodes.
Based on this association we continue with the definition of tensor networks.

\begin{definition}[Tensor network]
    \label{def:tensorNetwork}
    Let $\graph=(\nodes,\edges)$ be a hypergraph, let $\catvariableof{\node}$ for $\node\in\nodes$ be categorical variables with dimensions $\catdimof{\node} \in \nn$, and let
    \begin{align*}
        \hypercoreofat{\edge}{\catvariableof{\edge}} \in \bigotimes_{\node\in\edge}\rr^{\catdimof{\node}}
    \end{align*}
    be tensors for $\edge\in\edges$, where we denote by $\catvariableof{\edge}$ the set of categorical variables $\catvariableof{\node}$ with $\node\in\edge$.
    Then, we call the set
    \begin{align*}
        \tnetofat{\graph}{\catvariableof{\nodes}} = \{\hypercoreofat{\edge}{\catvariableof{\edge}}  \wcols \edge\in\edges\}
    \end{align*}
    the tensor network of the decorated hypergraph $\graph$.
    The set of tensor networks on $\graph$ such that all tensors have non-negative coordinates is denoted by $\tnsetof{\graph}$.
\end{definition}

As examples we now present the $\cpformat$ and the $\ttformat$ formats in our hypergraph notation.

\begin{example}[The $\cpformat$ format]\label{exa:cpFormat}
    The Candecomp-Parafac ($\cpformat$) tensor format (see \cite{hitchcock_expression_1927}) corresponds in our notation to a hypergraph (see \figref{fig:CPHypergraph}) defined by
    \begin{itemize}
        \item nodes $\catvariableof{[\catorder]}$ and a single hidden variable $\decvariable$, decorated by dimensions $\catdimof{[\catorder]}$ and the CP-rank $\decdim$, respectively
        \item edges
            $\big\{\edgeof{\catenumerator}=\{\catvariableof{\catenumerator},\decvariable\} \wcols \catenumeratorin \big\}$
        each decorated by a matrix $\hypercoreofat{\edgeof{\catenumerator}}{\catvariableof{\catenumerator},\decvariable}\in\mathbb{R}^{m_k\times n}$.
    \end{itemize}

    \begin{figure}
        \begin{center}
            \begin{tikzpicture}[scale=0.35,thick]

                \begin{scope}[shift={(-19,-2)}]
                    \coordinate[label=left:$a)$] (A) at (-2,4);

                    \node[circle, draw, thick, fill=\nodegrayscale, minimum size = \nodeminsize] (A) at (0,0) {};
                    \node[anchor=center] (A) at (0,0) {\corelabelsize $\catvariableof{0}$};

                    \node[circle, draw, thick, fill=\nodegrayscale, minimum size = \nodeminsize] (A) at (4,0) {};
                    \node[anchor=center] (A) at (4,0) {\corelabelsize $\catvariableof{1}$};

                    \coordinate[label=below:$\hdots $] (A) at (9,0.5);

                    \node[circle, draw, thick, fill=\nodegrayscale, minimum size = \nodeminsize] (A) at (14,0) {};
                    \node[] (text) at (14,0) {\corelabelsize $\catvariableof{\catorder\shortminus1}$};

                    \node[circle, draw, thick, fill=\nodegrayscale, minimum size = \nodeminsize] (A) at (7,4) {\corelabelsize $\decvariable$};

                    \draw (6.25,3.2) -- (0.5,1) node[midway,above] {\colorlabelsize $\edgeof{0}$};
                    \draw (6.75,2.9) -- (4,1.15) node[midway,below] {\colorlabelsize $\edgeof{1}$};
                    \draw (7.75,3.2) -- (13.5,1) node[midway,above] {\colorlabelsize $\edgeof{\catorder\shortminus1}$};

                \end{scope}

                \coordinate[label=left:$b)$] (A) at (-2,2);

                \begin{scope}[shift={(0,-2)}]

                    \draw (-1,-1) rectangle (1,1);
                    \node[anchor=center] (A) at (0,0) {\corelabelsize $\hypercoreof{0}$};
                    \draw (0,-1)--(0,-2.5) node[midway,right] {\colorlabelsize $\catvariableof{0}$};

                    \draw (3,-1) rectangle (5,1);
                    \node[anchor=center] (A) at (4,0) {\corelabelsize $\hypercoreof{1}$};
                    \draw (4,-1)--(4,-2.5) node[midway,right] {\colorlabelsize $\catvariableof{1}$};

                    \node[anchor=center] (text) at (8,0) {$\hdots$};

                    \draw (11,-1) rectangle (13,1);
                    \node[anchor=center] (A) at (12,0) {\corelabelsize $\hypercoreof{\catorder\shortminus1}$};
                    \draw (12,-1)--(12,-2.5) node[midway,right] {\colorlabelsize $\catvariableof{\catorder\shortminus1}$};

                    \drawvariabledot{6}{4}
                    \node[anchor=south] (text) at (6,4) {\colorlabelsize $\decvariable$};

                    \draw (6,4) to[bend right= 20] (0,1);
                    \draw (6,4) to[bend right= 10] (4,1);
                    \draw (6,4) to[bend right= -20] (12,1);

                \end{scope}

            \end{tikzpicture}
        \end{center}
        \caption{Hypergraph to a $\cpformat$ format (see \exaref{exa:cpFormat}).
        a) Node-centric design.
        b) Corresponding tensor network on the edges of the hypergraph.}\label{fig:CPHypergraph}
    \end{figure}

\end{example}

\begin{example}[The $\ttformat$ format]\label{exa:ttFormat}
    The Tensor-Train ($\ttformat$) format (see \cite{oseledets_tensor-train_2011}) corresponds in our notation to a hypergraph (see \figref{fig:TTHypergraph}) defined by
    \begin{itemize}
        \item nodes $\catvariableof{[\catorder]}$ and hidden variables $\decvariableof{[\catorder-1]}$, each decorated by a dimension $\catdimof{[\catorder]}$ and $\decdimof{[\catorder-1]}$,
        \item edges
        \begin{align*}
            \big\{\edgeof{0}=\{\catvariableof{0},\decvariableof{0}\}\big\} \cup
            \big\{\edgeof{\catenumerator}=\{\decvariableof{\catenumerator-1},\catvariableof{\catenumerator},\decvariableof{\catenumerator}\} \wcols \catenumerator\in\{1,\ldots,\catorder-2\}\big\} \cup
            \big\{\edgeof{\catorder-1}=\{\decvariableof{\catorder-2},\catvariableof{\catorder-1}\}\big\}
        \end{align*}
        each decorated by a tensor of order 3 (respectively 2 for $\catenumerator\in\{0,\catorder-1\}$).
    \end{itemize}
    \begin{figure}
        \begin{center}
            \begin{tikzpicture}[scale=0.35,thick]

                \begin{scope}[shift={(-22,-2)}]
                    \coordinate[label=left:$a)$] (A) at (-2,4);

                    \node[circle, draw, thick, fill=\nodegrayscale, minimum size = \nodeminsize] (A) at (0,0) {};
                    \node[anchor=center] (A) at (0,0) {\corelabelsize $\catvariableof{0}$};
                    \node[circle, draw, thick, fill=\nodegrayscale, minimum size = \nodeminsize] (A) at (4,0) {};
                    \node[anchor=center] (A) at (4,0) {\corelabelsize $\catvariableof{1}$};

                    \coordinate[label=below:$\hdots $] (A) at (10.5,4);
                    \coordinate[label=below:$\hdots $] (A) at (10.5,0);

                    \node[circle, draw, thick, fill=\nodegrayscale, minimum size = \nodeminsize] (A) at (2,4) {};
                    \node[anchor=center] (A) at (2,4) {\corelabelsize $\decvariableof{0}$};
                    \node[circle, draw, thick, fill=\nodegrayscale, minimum size = \nodeminsize] (A) at (6,4) {};
                    \node[anchor=center] (A) at (6,4) {\corelabelsize $\decvariableof{1}$};

                    \draw (0.5,1) -- (1.5,3) node[midway,left] {\colorlabelsize $e_0$};

                    \node[anchor=south] (A) at (4,2) {\colorlabelsize $e_1$};
                    \draw (4,2)  to[bend right=-30] (2.5,3);
                    \draw (4,2)  to[bend left=-30] (5.5,3);
                    \draw (4,2) -- (4,1.1);

                    \node[anchor=south] (A) at (8,2) {\colorlabelsize $e_2$};
                    \draw (8,2)  to[bend right=-30] (6.5,3);
                    \draw (8,2)  to[bend left=-30] (9.5,3);
                    \draw (8,2) -- (8,1.1);

                    \begin{scope}[shift={(3,0)}]

                        \node[anchor=south] (A) at (10,2) {\colorlabelsize $e_{\catorder-2}$};
                        \draw (10,2)  to[bend right=-30] (8.5,3);
                        \draw (10,2)  to[bend left=-30] (11.5,3);
                        \draw (10,2) -- (10,1.1);

                        \draw (13.5,1) -- (12.5,3) node[midway,right] {\colorlabelsize $e_{\catorder-1}$};

                        \node[circle, draw, thick, fill=\nodegrayscale, minimum size = \nodeminsize] (A) at (12,4) {};
                        \node[anchor=center] (text) at (12,4) {\corelabelsize $\decvariableof{\catorder\shortminus2}$};

                        \node[circle, draw, thick, fill=\nodegrayscale, minimum size = \nodeminsize] (A) at (14,0) {};
                        \node[] (text) at (14,0) {\corelabelsize $\catvariableof{\catorder\shortminus1}$};
                    \end{scope}
                \end{scope}

                \coordinate[label=left:$b)$] (A) at (-2,2);

                \draw (-1,-1) rectangle (1,1);
                \node[anchor=center] (A) at (0,0) {\corelabelsize $\hypercoreof{0}$};
                \draw (0,-1)--(0,-2.5) node[midway,right] {\colorlabelsize $\catvariableof{0}$};

                \draw (1,0) -- (3,0) node[midway,above] {\colorlabelsize $\decvariableof{0}$};

                \draw (3,-1) rectangle (5,1);
                \node[anchor=center] (A) at (4,0) {\corelabelsize $\hypercoreof{1}$};
                \draw (4,-1)--(4,-2.5) node[midway,right] {\colorlabelsize $\catvariableof{1}$};

                \draw (5,0) -- (6.5,0) node[midway,above] {\colorlabelsize $\decvariableof{1}$};

                \node[anchor=center] (text) at (8,0) {$\hdots$};

                \draw (9.5,0) -- (11,0) node[midway,above] {\colorlabelsize $\decvariableof{\catorder\shortminus2}$};;
                \draw (11,-1) rectangle (13,1);
                \node[anchor=center] (A) at (12,0) {\corelabelsize $\hypercoreof{\catorder\shortminus1}$};
                \draw (12,-1)--(12,-2.5) node[midway,right] {\colorlabelsize $\catvariableof{\catorder\shortminus1}$};

            \end{tikzpicture}
        \end{center}
        \caption{Hypergraph to a $\ttformat$ format (see \exaref{exa:ttFormat}).
        a) Node-centric design.
        b) Corresponding tensor network on the edges of the hypergraph.}\label{fig:TTHypergraph}
    \end{figure}

\end{example}

\subsection{Generic contractions}

Let us now exploit our graphical approach to tensor networks in the definition of contractions.

\begin{definition}
    \label{def:contraction}
    Let $\tnetof{\graph}$ be a tensor network on a decorated hypergraph $\graph=(\nodes,\edges)$.
    For any subset $\secnodes\subset\nodes$ we define the contraction of $\tnetof{\graph}$ with open variables $\catvariableof{\secnodes}$ to be the tensor (for an example see \figref{fig:contraction})
    \begin{align*}
        \contractionof{\tnetof{\graph}}{\secnodevariables} \in \bigotimes_{\node\in\secnodes} \rr^{\catdimof{\node}}
    \end{align*}
    with coordinates at indices $\catindexof{\secnodes}\in\bigtimes_{\node\in\secnodes}[\catdimof{\node}]$ by
    \begin{align*}
        \contractionof{\tnetof{\graph}}{\indexedcatvariableof{\secnodes}} =
        \sum_{\catindexof{\setwithout{\nodes}{\secnodes}} \in\,\nodestatesof{\setwithout{\nodes}{\secnodes}}}
        \left( \prod_{\edge\in\edges}\hypercoreofat{\edge}{\indexedcatvariableof{\edge}} \right) \, .
    \end{align*}
\end{definition}

When an open variable $\catvariable$ does not appear in any tensor in a contraction, we define the contraction as a tensor product with the trivial tensor $\onesat{\catvariable}$ (see \exaref{exa:diracDelta}).
To ease notation, we often omit the set notation by brackets $\{\cdot\}$.

\begin{figure}
    \begin{center}
        \begin{tikzpicture}[scale=0.35,thick]

    \draw (-5,-1) rectangle (9,-3);
    \node[anchor=center] (text) at (2,-2) {\corelabelsize $\contractionof{\hypercoreof{\edge_0},\hypercoreof{\edge_1},\hypercoreof{\edge_2}}{\catvariableof{1},\catvariableof{3}}$};
    \draw (0,-3)--(0,-5) node[midway,left] {\colorlabelsize $\catvariableof{1}$};
    \draw (4,-3)--(4,-5) node[midway,left] {\colorlabelsize $\catvariableof{3}$};

    \node[anchor=center] (text) at (11.5,-2) {${=}$};

    \begin{scope}
        [shift={(15,0)}]

        \draw (-1,-1) rectangle (5,-3);
        \node[anchor=center] (text) at (2,-2) {\corelabelsize $\hypercoreof{\edge_0}$};
        \draw (0,-3)--(0,-4) node[midway,right] {\colorlabelsize $\catvariableof{0}$};
        \drawvariabledot{0}{-4}

        \draw (2,-3)--(2,-5) node[midway,right] {\colorlabelsize $\catvariableof{1}$};
        \draw (4,-3)--(4,-5) node[midway,right] {\colorlabelsize $\catvariableof{2}$};

        \draw (6,-1) rectangle (10,-3);
        \node[anchor=center] (text) at (8,-2) {\corelabelsize $\hypercoreof{\edge_2}$};
        \draw (7,-3)--(7,-5) node[midway,right] {\colorlabelsize $\catvariableof{2}$};
        \draw (9,-3)--(9,-5) node[midway,right] {\colorlabelsize $\catvariableof{3}$};

        \draw (1,-7) rectangle (5,-9);
        \node[anchor=center] (text) at (3,-8) {\corelabelsize $\hypercoreof{\edge_1}$};
        \draw (2,-5)--(2,-7); 
        \draw (4,-5) to[bend right=20]  (7,-6); 
        \draw (4,-7) to[bend right=-20]  (7,-6);

        \draw[fill] (2,-6) circle (\dotsize);
        \draw (2,-6) to[bend right=20] (-1,-8); 
        \node[anchor=center] (text) at (-2,-8) {\colorlabelsize $\catvariableof{1}$};

        \draw[fill] (7,-6) circle (\dotsize);
        \draw (7,-5) -- (7,-6);

    \end{scope}

\end{tikzpicture}
    \end{center}
    \caption{
        Graphical depiction of a tensor network contraction with the open variables $\catvariableof{1},\catvariableof{3}$.
        Open variables are depicted by those without a dot at the end of the line.
    }\label{fig:contraction}
\end{figure}
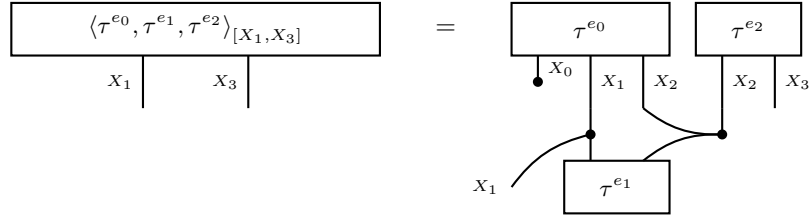

\begin{example}[Tensor product]
    \label{exa:tensorProduct}
    The simplest contraction is the tensor product, which maps a pair of two tensors with distinct variables onto a third tensor and has an interpretation by coordinate-wise products.
    Such a contraction corresponds with a tensor network of two tensors with disjoint variables.
    Let there be two tensors
    \begin{align*}
        \hypercoreat{\shortcatvariables} \in \facspace  \andspace \sechypercoreat{\headvariables} \in \bigotimes_{\selindexin}\rr^{\headdimof{\selindex}}
    \end{align*}
    with disjoint tuples of categorical variables assigned to their axes.
    Then their tensor product is the tensor
    \begin{align*}
        \contractionof{\hypercoreat{\shortcatvariables},\sechypercoreat{\secshortcatvariables}}{\shortcatvariables,\headvariables}
        \in \left(\facspace\right) \otimes \left(\bigotimes_{\selindexin}\rr^{\headdimof{\selindex}}\right)
    \end{align*}
    with coordinates to tuples of $\shortcatindices\in\facstates$ and $\shortheadindices\in\bigtimes_{\selindexin}[\headdimof{\selindex}]$ as
    \begin{align*}
        & \contractionof{\hypercoreat{\shortcatvariables},\sechypercoreat{\secshortcatvariables}}{\indexedshortcatvariables,\headvariables=\shortheadindices} \\
        &\quad\quad \coloneqq  \hypercoreat{\indexedshortcatvariables}\cdot \sechypercoreat{\headvariables=\shortheadindices} \, .
    \end{align*}
\end{example}

\subsection{Normalizations}

Based on generic contractions, we now introduce the normalization of tensors, which introduces certain contraints on tensors to be depicted by directed hyperedges.

\begin{definition}
    \label{def:normalization}
    The normalization of a tensor $\hypercorewithnodes$ on incoming nodes $\innodes\subset\nodes$ and outgoing nodes $\outnodes\subset\setwithout{\nodes}{\innodes}$ is the tensor $\normalizationofwrt{\hypercorewithnodes}{\catvariableof{\outnodes}}{\catvariableof{\innodes}}$ defined for $\catindexof{\innodes}$ as
    \begin{align*}
        \normalizationofwrt{\hypercorewithnodes}{\catvariableof{\outnodes}}{\indexedcatvariableof{\innodes}}
        = \begin{cases}
              \frac{\contractionof{\hypercore}{\catvariableof{\outnodes},\indexedcatvariableof{\innodes}}}{\contractionof{\hypercore}{\indexedcatvariableof{\innodes}}} & \ifspace \contractionof{\hypercore}{\indexedcatvariableof{\innodes}} \neq 0 \\
              \frac{1}{\prod_{\node\in\outnodes}\catdimof{\node}}\onesat{\catvariableof{\outnodes}} & \text{else}
        \end{cases} \, .
    \end{align*}
    We say that $\hypercorewithnodes$ is normalized with incoming nodes $\innodes\subset\nodes$, if
    \begin{align*}
        \hypercorewithnodes
        = \normalizationofwrt{\hypercorewithnodes}{\catvariableof{\setwithout{\nodes}{\innodes}}}{\catvariableof{\innodes}} \, .
    \end{align*}
\end{definition}

In our graphical tensor notation, we depict normalized tensors by directed hyperedges (a), which are decorated by directed tensors (b), for example when $\catvariableof{\innodes}=(\catvariableof{2},\catvariableof{3})$ and $\catvariableof{\setwithout{\nodes}{\innodes}}=(\catvariableof{0},\catvariableof{1})$:
\begin{center}
    \begin{tikzpicture}[scale=0.35,thick, xscale=1.1] 

    \node[anchor=center] (text) at (-2,0) {$a)$};

    \node [circle, draw, thick, fill=\nodegrayscale, minimum size = \nodeminsize] (P1) at (0,-3) {\colorlabelsize $\catvariableof{0}$};
    \node [circle, draw, thick, fill=\nodegrayscale, minimum size = \nodeminsize] (P2) at (3,-3) {\colorlabelsize $\catvariableof{1}$};

    \node [circle, draw, thick, fill=\nodegrayscale, minimum size = \nodeminsize] (P3) at (6,-3)  {\colorlabelsize $\catvariableof{2}$};

    \node [circle, draw, thick, fill=\nodegrayscale, minimum size = \nodeminsize] (P4) at (9,-3)  {\colorlabelsize $\catvariableof{3}$};

    \node[anchor=center] (text) at (9,-3) {\colorlabelsize $\catvariableof{3}$};

    \draw[midarrow]
    (4.5,0) to[bend right=25] (P1);
    \draw[midarrow]
    (4.5,0) to[bend right=10] (P2);
    \draw[midarrow]
    (P3) to[bend right=10] (4.5,0);
    \draw[midarrow]
    (P4) to[bend right=25] (4.5,0);

    \node[anchor=center] (text) at (4.5,0.5) {$\edge$};

    \begin{scope}[shift={(20,0)}]

        \node[anchor=center] (text) at (-2,0) {$b)$};

        \draw (-1,-1) rectangle (7,-3);
        \node[anchor=center] (text) at (3,-2) {\corelabelsize $\hypercoreofat{\edge}{\catvariableof{0},\catvariableof{1},\catvariableof{2},\catvariableof{3}}$};

        \draw[midarrow]  (0,-3) -- (0,-5) node[midway,left] {\colorlabelsize $\catvariableof{0}$};
        \draw[midarrow]
        (2,-3)--(2,-5) node[midway,left] {\colorlabelsize $\catvariableof{1}$};
        \draw[midarrow]
        (4,-5)--(4,-3) node[midway,left] {\colorlabelsize $\catvariableof{2}$};
        \draw[midarrow]
        (6,-5)--(6,-3) node[midway,right] {\colorlabelsize $\catvariableof{3}$};
    \end{scope}

\end{tikzpicture}
\end{center}

\subsection{Function encoding and \ComputationActivationNetworks{}}

Towards presenting the function encoding schemes, we define one-hot encodings mapping the states of variables to basis tensors.

\begin{definition}[One-hot encoding]
    \label{def:onehotenc}
    To any variable $\catvariable$ taking values in $[\catdim]$, the one-hot encoding of any state $\catindex\in[\catdim]$ is the vector with coordinates
    \begin{align*}
        \onehotmapofat{\catindex}{\catvariable=\tilde{\catindex}}
        \coloneqq \begin{cases}
                      1 & \ifspace \catindex=\tilde{\catindex} \\
                      0 & \text{else} \, .
        \end{cases}
    \end{align*}
    To any tuple $\shortcatvariables$ of variables taking values in $\facstates$, the one-hot encoding of a state tuple $\shortcatindices$ is the tensor product
    \begin{align*}
        \onehotmapofat{\shortcatindices}{\shortcatvariables}
        \coloneqq \bigotimes_{\catenumeratorin} \onehotmapofat{\catindexof{\atomenumerator}}{\catvariableof{\atomenumerator}} \, .
    \end{align*}
\end{definition}

We now use one-hot encodings to encode functions between state sets.

\begin{definition}[Basis encoding of maps between state sets]
    \label{def:functionRepresentation}
    Let there be two sets of variables $\shortcatvariables$ and $\headvariables$, and let there be a map
    \begin{align*}
        \statesetfunction\defcols\facstates \rightarrow  \bigtimes_{\selindexin}[\headdimof{\selindex}]
    \end{align*}
    between their state sets.
    Then, the basis encoding of $\statesetfunction$ is a tensor
    \begin{align*}
        \bencodingofat{\statesetfunction}{\headvariables,\shortcatvariables}
        \in \left(\bigotimes_{\selindexin}\rr^{\headdimof{\selindex}}\right) \otimes \left(\facspace\right)
    \end{align*}
    defined by
    \begin{align*}
        \bencodingofat{\statesetfunction}{\headvariables,\shortcatvariables}
        = \sum_{\shortcatindices\in\facstates}
        \onehotmapofat{\statesetfunctionev{\shortcatindices}}{\headvariables} \otimes  \onehotmapofat{\shortcatindices}{\shortcatvariables} \, .
    \end{align*}
\end{definition}
Basis encodings are normalized tensors and are thus depicted as decorations of directed edges in hypergraphs:
\begin{center}
    \begin{tikzpicture}[scale=0.4,thick] 

    \draw[->-] (0,-1)--(0,1) node[midway,left] {\colorlabelsize $\headvariableof{0}$};
    \draw[->-] (1.5,-1)--(1.5,1) node[midway,left] {\colorlabelsize $\headvariableof{1}$};
    \node[anchor=center] (text) at (3,0) {$\cdots$};
    \draw[->-] (4,-1)--(4,1) node[midway,right] {\colorlabelsize $\headvariableof{\seccatorder\shortminus1}$};

    \draw (-1,-1) rectangle (5,-3);
    \node[anchor=center] (text) at (2,-2) {\corelabelsize $\bencodingof{\statesetfunction}$};
    \draw[-<-] (0,-3)--(0,-5) node[midway,left] {\colorlabelsize $\catvariableof{0}$};
    \draw[-<-] (1.5,-3)--(1.5,-5) node[midway,left] {\colorlabelsize $\catvariableof{1}$};
    \node[anchor=center] (text) at (3,-4) {$\cdots$};
    \draw[-<-] (4,-3)--(4,-5) node[midway,right] {\colorlabelsize $\catvariableof{\atomorder\shortminus1}$};

\end{tikzpicture}
\end{center}

We further generalize basis encodings to arbitrary functions between finite sets by the use of bijective image enumeration maps.
Given an arbitrary set $\arbset$, we say a map
\begin{align*}
    \indexinterpretation \defcols
    \bigtimes_{\catenumeratorin}[\catdimof{\catenumerator}] \rightarrow \arbset
\end{align*}
is an enumeration map of $\arbset$ by $\catorder$ variables $\catvariableof{\catenumerator}$, taking values in $\catdimof{\catenumerator}$.
Given a function $\exfunction:\inset\rightarrow\outset$ between arbitrary sets and enumerating maps $\indexinterpretationof{\insymbol}$ and $\indexinterpretationof{\outsymbol}$ for both sets, we define the basis encoding of $\exfunction$ as
\begin{align*}
    \bencodingofat{\exfunction}{\headvariables,\shortcatvariables}
    = \sum_{\arbelement\in\inset} \onehotmapofat{\invindexinterpretationofat{\outsymbol}{\exfunction(\arbelement)}}{\headvariables}
    \otimes \onehotmapofat{\invindexinterpretationofat{\insymbol}{\arbelement}}{\shortcatvariables} \, ,
\end{align*}
where $\shortcatvariables,\headvariables$ are variables taking values in $[\cardof{\inset}]$ and $[\cardof{\outset}]$.
In \exaref{exa:madicRepresentation} we present index enumeration maps for summations in $\catdim$-adic integer representations.
Based on these concepts, we define the most general tensor network architecture to be applied in the rest of this work.

\begin{definition}[\ComputationActivationNetwork{} (\CompActNets{})]
    \label{def:compActNets}
    Let there be a function $\sstat : \facstates \allowbreak \rightarrow \parspace$ with basis encoding $\bencodingofat{\sstat}{\headvariables,\shortcatvariables}$, where $\headvariables$ is a tuple of variables to an enumeration map of the image of $\sstat$.
    Let there further be a hypergraph $\graph=(\nodes,\edges)$ with nodes $\nodes$ containing $[\seldim]$.
    We define the by $\sstat$ computable and by $\graph$ activated family of distributions by
    \begin{align*}
        \realizabledistsof{\sstat,\graph}
        = \left\{ \normalizationof{\bencodingofat{\sstat}{\headvariables,\shortcatvariables},\contractionof{\acttensor}{\headvariables} 
        }{\shortcatvariables}
              \wcols \acttensorat{\headvariableof{\nodes}} \in \tnsetof{\graph} \right\} \, .
    \end{align*}
    We refer to any member $\probat{\shortcatvariables}\in\realizabledistsof{\sstat,\graph}$ as a \emph{\ComputationActivationNetwork{}} (or shorter as a \emph{\CompActNet{}}).
    We call $\bencsstatwith$ (and any decomposition of it) the \emph{computation network} and $\acttensorat{\headvariableof{\nodes}}$ the \emph{activation network}.
\end{definition}

The elementary activated networks are representable by an elementary activation tensor with respect to the graph
\begin{align*}
    \elgraph = \big(\nodes,\{\{\node\}\wcols\nodein\}\big)
\end{align*}
and we denote such networks by $\realizabledistsof{\sstat,\elgraph}$.
Any \CompActNet{} is representable with respect to the maximal hypergraph
\begin{align*}
    \maxgraph  = \big(\nodes,\{\nodes\}\big) \, .
\end{align*}
We therefore have for any graph that $\realizabledistsof{\sstat,\graph}\subset\realizabledistsof{\sstat,\maxgraph}$.

\section{The probabilistic paradigm}\label{sec:probPar}

In the following we investigate tensor network decomposition mechanisms of probability distributions.
After introducing probability distributions as tensors and independencies as decomposition schemes, we derive tensor network decompositions based on conditional independencies (applying a classical theorem of Hammersley-Clifford, see \cite{clifford_markov_1971}) to motivate graphical models.
Furthermore, we present the Fisher-Neyman Factorization Theorem as providing decompositions in the presence of sufficient statistics.

\subsection{Basic concepts}

As defined next, distributions $\probtensor$ over a discrete state space can be represented by tensors, where each entry corresponds to the probability of a corresponding state.

\begin{definition}[Joint probability distribution]
    \label{def:probabilityDistribution} 
    Let there be for each $\catenumeratorin$ a categorical variable $\catvariableof{\catenumerator}$ taking values in $[\catdimof{\catenumerator}]$.
    A joint probability distribution of these categorical variables is a tensor
    \begin{align*}
        \probwith \in \facspace
    \end{align*}
    which coordinates are non-negative, that is for any $\shortcatindicesin$ it holds
    \begin{align*}
        \probat{\indexedshortcatvariables} \geq 0 \, ,
    \end{align*}
    and which is normalized with no incoming variables, that is
    \begin{align*}
        \contraction{\probat{\shortcatvariables}}
        = 1 \, .
    \end{align*}
    Let $\thirdcatvariable$ be another variable taking values in a possibly infinite set $\valof{\thirdcatvariable}$.
    Then, a tensor $\condprobat{\shortcatvariables}{\thirdcatvariable}$ is a family of joint probability distributions if, for any $\thirdcatindex\in\valof{\thirdcatvariable}$, the slice $\condprobat{\shortcatvariables}{\thirdcatvariable=\thirdcatindex}$ is a joint probability distribution.
\end{definition}

\begin{example}[Family of independent coin tosses]
    Consider tossing a coin with head probability $\thirdcatindex\in[0,1]$ and repeating the experiment independently $\catorder\in\nn$ times.
    We define a variable $\thirdcatvariable$ taking values in $\mathrm{val}(\thirdcatvariable)=[0,1]$ and denote by $\shortcatvariables$ $\catorder$ boolean variables.
    Then, the family of coin toss distributions is the tensor $\condprobat{\shortcatvariables}{\thirdcatvariable}$ with coordinates $\shortcatindices\in\atomstates$ and $\thirdcatindex\in[0,1]$ defined by
    \begin{align*}
        \condprobat{\indexedshortcatvariables}{\thirdcatvariable=\thirdcatindex}
        = \prod_{\catenumeratorin} \thirdcatindex^{\catindexof{\catenumerator}} (1-\thirdcatindex)^{1-\catindexof{\catenumerator}}
        = \thirdcatindex^{\sum_{\catenumeratorin} \catindexof{\catenumerator}} (1-\thirdcatindex)^{\catorder - \sum_{\catenumeratorin} \catindexof{\catenumerator}} \, .
    \end{align*}
    Note that by the binomial theorem we have $\contraction{\probat{\shortcatvariables,\thirdcatvariable=\thirdcatindex}} =1$ for each slice with respect to $\thirdcatindex\in[0,1]$. Therefore, $\probat{\shortcatvariables,\thirdcatvariable}$ is indeed a family of probability distributions.
    For $\catorder=2$ we have more explicitly for any $\thirdcatindex\in[0,1]$ that
    \begin{center}
        \begin{tikzpicture}[scale=1]
            \node[anchor=east] (A) at (-2,0) {$\condprobat{\catvariableof{[2]}}{\thirdcatvariable=\thirdcatindex}\,=$};

            \node (A) at (1,0) {
                $\begin{bmatrix}
                (1-\thirdcatindex)
                     ^2 & \thirdcatindex \cdot (1-\thirdcatindex) \\
                     \thirdcatindex \cdot (1-\thirdcatindex) & \thirdcatindex^2
                \end{bmatrix}$
            };
            \draw[<-,dashed] (-1.3,-0.275) node[right] {\tiny $1$} -- (-1.3,0.275) node [midway, left] {\tiny $\catvariableof{0}$} node[right] {\tiny $0$};
            \draw[->,dashed] (0,0.85) node[below] {\tiny $0$} -- (2,0.85) node [midway, above] {\tiny $\catvariableof{1}$} node[below] {\tiny $1$};
            \node[anchor=east] (A) at (3.5,-0.55) {$\cdot$};
        \end{tikzpicture}
    \end{center}
\end{example}

A basic inference operation on probability distributions is the computation of marginal and conditional distributions.

\begin{definition}
    \label{def:marginalConditionalDistribution}
    For any distribution $\probat{\catvariableof{0},\catvariableof{1}}$ the marginal distribution is the contraction (see \defref{def:contraction})
    \begin{align*}
        \probat{\catvariableof{0}}
        \coloneqq \contractionof{\probat{\catvariableof{0},\catvariableof{1}}}{\catvariableof{0}},
    \end{align*}
    which is depicted by the diagram
    \begin{center}
        \begin{tikzpicture}[scale=0.3,thick] 

    \draw (-19,-1) rectangle (-15,-3);
    \node[anchor=center] (text) at (-17,-2) {\corelabelsize $\margprobat{\exrandom}$};
    \draw[midarrow]  (-17,-3)--(-17,-5) node[midway,left] {\colorlabelsize $\exrandom$};

    \node[anchor=center] (text) at (-13,-2) {${=}$};

    \draw (-11,-1) rectangle (-5,-3);
    \node[anchor=center] (text) at (-8,-2) {\corelabelsize $\probat{\exrandom,\secexrandom}$};
    \draw[midarrow]  (-10,-3)--(-10,-5) node[midway,left] {\colorlabelsize $\exrandom$};
    \draw[midarrow]  (-6,-3)--(-6,-5) node[midway,left] {\colorlabelsize $\secexrandom$};
    \drawvariabledot{-6}{-5}

\end{tikzpicture}
    \end{center}
    The conditional distribution of $\catvariableof{0}$ on $\catvariableof{1}$ is the normalization (see \defref{def:normalization})
    \begin{align*}
        \condprobat{\catvariableof{0}}{\catvariableof{1}}
        \coloneqq \normalizationofwrt{\probat{\catvariableof{0},\catvariableof{1}}}{\catvariableof{0}}{\catvariableof{1}} \, .
    \end{align*}
\end{definition}

For $\catindexof{1}\in[\catdimof{1}]$ with $\contraction{\probat{\catvariableof{0},\indexedcatvariableof{1}}}$ we depict the normalization by
\begin{center}
    \begin{tikzpicture}[scale=0.3, thick] 

    \begin{scope}
        [shift={(-13,0)}]

        \draw (-22,-1) rectangle (-14,-3);
        \node[anchor=center] (text) at (-18,-2) {\corelabelsize $\condprobat{\catvariableof{0}}{\catvariableof{1}=\catindexof{1}}$};
        \draw[->-]  (-18,-3)--(-18,-5) node[midway,left] {\colorlabelsize $\catvariableof{0}$};

        \node[anchor=center] (text) at (-12,-2) {${\coloneqq}$};

    \end{scope}

    \begin{scope}
        [shift={(-11,6)}]

        \draw (-11,-1) rectangle (-5,-3);
        \node[anchor=center] (text) at (-8,-2) {\corelabelsize $\probat{\catvariableof{0},\catvariableof{1}}$};
        \draw[->-]  (-10,-3)--(-10,-5) node[midway,left] {\colorlabelsize $\catvariableof{0}$};
        \draw[->-]  (-6,-3)--(-6,-5) node[midway,left] {\colorlabelsize $\catvariableof{1}$};
        \draw[] (-7,-5) rectangle (-5,-7);
        \node[anchor=center] (text) at (-6,-6) {\corelabelsize $\onehotmapof{\catindexof{1}}$};

    \end{scope}

    \draw (-23,-2) -- (-15,-2);

    \begin{scope}
        [shift={(-11,-2)}]

        \draw (-11,-1) rectangle (-5,-3);
        \node[anchor=center] (text) at (-8,-2) {\corelabelsize $\probat{\catvariableof{0},\catvariableof{1}}$};
        \draw[->-]  (-10,-3)--(-10,-5) node[midway,left] {\colorlabelsize $\catvariableof{0}$};
        \drawvariabledot{-10}{-5}
        \draw[->-]  (-6,-3)--(-6,-5) node[midway,left] {\colorlabelsize $\catvariableof{1}$};
        \draw[] (-7,-5) rectangle (-5,-7);
        \node[anchor=center] (text) at (-6,-6) {\corelabelsize $\onehotmapof{\catindexof{1}}$};

    \end{scope}

    \begin{scope}
        [shift={(10,0)}]

        \node[anchor=center] (text) at (-23,-2) {${\eqqcolon}$};
        
        \draw (-21,-1) rectangle (-15,-3);
        \node[anchor=center] (text) at (-18,-2) {\corelabelsize $\condprobat{\catvariableof{0}|\catvariableof{1}}$};
        \draw[->-]  (-20,-3)--(-20,-5) node[midway,left] {\colorlabelsize $\catvariableof{0}$};

        \draw[-<-]  (-16,-3)--(-16,-5) node[midway,left] {\colorlabelsize $\catvariableof{1}$};
        \draw[] (-15,-5) rectangle (-17,-7);
        \node[anchor=center] (text) at (-16,-6) {\corelabelsize $\onehotmapof{\catindexof{1}}$};

    \end{scope}

\end{tikzpicture}
\end{center}

\subsection{Factorization into graphical models}

The number of coordinates in a tensor representation of probability distributions is the product
\begin{align*}
    \prod_{\catenumeratorin}\catdimof{\catenumerator} \, .
\end{align*}
It therefore scales exponentially in the number of coordinates.
To find efficient representation schemes of probability distributions by tensor networks, we need to exploit additional properties of the distribution.
Independence leads to severe sparsifications of conditional probabilities and is hence the key assumption to gain sparse decompositions of probability distributions.

\begin{definition}[Independence]
    \label{def:independence} 
    We say that $\catvariableof{0}$ is independent of $\catvariableof{1}$ with respect to a distribution $\probat{\catvariableof{0},\catvariableof{1}}$ if the distribution is the tensor product of the marginal distributions, that is
    \begin{align*}
        \probat{\catvariableof{0},\catvariableof{1}}
        = \probat{\catvariableof{0}} \otimes \probat{\catvariableof{1}} \, .
    \end{align*}
    In this case we write $\independent{\catvariableof{0}}{\catvariableof{1}}$.
\end{definition}

Thus, independence appears directly as a tensor–product decomposition of probability distribution.
Using tensor network diagrams, we depict this property by
\begin{center}
    \begin{tikzpicture}[scale=0.3,thick] 

    \draw (0,1) rectangle (7,-1);
    \node[anchor=center] (text) at (3.5,0) {\corelabelsize $\probat{\exrandom,\secexrandom}$};
    \draw[->-] (1,-1) -- (1,-3) node[midway, left] {\colorlabelsize $\exrandom$};
    \draw[->-] (6,-1) -- (6,-3) node[midway, left] {\colorlabelsize $\secexrandom$};

    \node[anchor=center] (text) at (9,0) {\corelabelsize ${=}$};

    \begin{scope}[shift={(11,0)}]

        \draw (0,1) rectangle (7,-1);
        \node[anchor=center] (text) at (3.5,0) {\corelabelsize $\probat{\exrandom,\secexrandom}$};
        \draw[->-] (1,-1) -- (1,-3) node[midway, left] {\colorlabelsize $\exrandom$};
        \draw[->-] (6,-1) -- (6,-3) node[midway, left] {\colorlabelsize $\secexrandom$};
        \drawvariabledot{6}{-3}

    \end{scope}

    \node[anchor=center] (text) at (20,0) {\corelabelsize $\otimes$};

    \begin{scope}[shift={(22,0)}]

        \draw (0,1) rectangle (7,-1);
        \node[anchor=center] (text) at (3.5,0) {\corelabelsize $\probat{\exrandom,\secexrandom}$};
        \draw[->-] (1,-1) -- (1,-3) node[midway, left] {\colorlabelsize $\exrandom$};
        \drawvariabledot{1}{-3}
        \draw[->-] (6,-1) -- (6,-3) node[midway, left] {\colorlabelsize $\secexrandom$};

    \end{scope}

    \node[anchor=center] (text) at (31,0) {\corelabelsize ${=}$};

    \begin{scope}[shift={(33,0)}]

        \draw (0,1) rectangle (4,-1);
        \node[anchor=center] (text) at (2,0) {\corelabelsize $\margprobat{\exrandom}$};
        \draw[->-] (2,-1) -- (2,-3) node[midway, left] {\colorlabelsize $\exrandom$};

        \node[anchor=center] (text) at (6,0) {\corelabelsize $\otimes$};

        \draw (8,1) rectangle (12,-1);
        \node[anchor=center] (text) at (10,0) {\corelabelsize $\margprobat{\secexrandom}$};
        \draw[->-] (10,-1) -- (10,-3) node[midway, left] {\colorlabelsize $\secexrandom$};

    \end{scope}


\end{tikzpicture} 
\end{center}
Note that the assumption of independence reduces the degrees of freedom from $(\exranddim\cdot\secexranddim)-1$ to $(\exranddim-1)+(\secexranddim-1)$.
The decomposition into marginal distributions furthermore exploits this reduced freedom and provides an efficient storage.
Having a joint distribution of multiple variables whose disjoint subsets are independent, we can iteratively apply the decomposition scheme.
As a result, we can reduce the scaling of the degrees of freedom from exponential to linear by the assumption of independence.

As we observed, independence is a strong assumption, which is often too restrictive.
Less demanding is conditional independence, which still implies efficient tensor network decomposition schemes.
We introduce conditional independence as independence of variables with respect to conditional distributions.

\begin{definition}[Conditional independence]
    \label{def:condIndependence} 
    Assume a joint distribution of variables $\catvariableof{0}$, $\catvariableof{1}$ and $\catvariableof{2}$.
    We say that $\catvariableof{0}$ is independent of $\catvariableof{1}$ conditioned on $\catvariableof{2}$ if
    \begin{align*}
        \condprobof{\catvariableof{0},\catvariableof{1}}{\catvariableof{2}}
        = \contractionof{
            \condprobof{\catvariableof{0}}{\catvariableof{2}},\condprobof{\catvariableof{1}}{\catvariableof{2}}
        }{\catvariableof{0},\catvariableof{1},\catvariableof{2}} \, .
    \end{align*}
    In this case we write $\condindependent{\catvariableof{0}}{\catvariableof{1}}{\catvariableof{2}}$.
\end{definition}

Conditional independence stated in \defref{def:condIndependence} has a close connection with independence stated in \defref{def:independence}.
To be more precise, $\catvariableof{0}$ is independent of $\catvariableof{1}$ conditioned on $\catvariableof{2}$ if and only if $\catvariableof{0}$ is independent of $\catvariableof{1}$ with respect to any slice $\condprobof{\catvariableof{0},\catvariableof{1}}{\indexedcatvariableof{2}}$ of the conditional distribution $\condprobof{\catvariableof{0},\catvariableof{1}}{\catvariableof{2}}$.

We can further exploit conditional independence to find tensor network decompositions of probabilities as we show in the next corollary.
\begin{corollary}
    \label{cor:secCriterionCondIndepencence}
    Let $\probat{\catvariableof{0},\catvariableof{1},\catvariableof{2}}$ be a joint distribution.
    If and only if $\catvariableof{0}$ is independent of $\catvariableof{1}$ conditioned on $\catvariableof{2}$, the distribution satisfies
    \begin{align*}
        \probat{\catvariableof{0},\catvariableof{1},\catvariableof{2}}
        = \contractionof{\condprobof{\catvariableof{0}}{\catvariableof{2}},\condprobof{\catvariableof{1}}{\catvariableof{2}},\margprobat{\catvariableof{2}}}{\catvariableof{0},\catvariableof{1},\catvariableof{2}} \, .
    \end{align*}
    In a diagrammatic notation, this is depicted by
    \begin{center}
        \begin{tikzpicture}[scale=0.3,thick] 

    \draw (-2,1) rectangle (7,-1);
    \node[anchor=center] (text) at (2.5,0) {\corelabelsize $\probat{\exrandom,\secexrandom,\thirdexrandom}$};
    \draw[->-] (-1,-1) -- (-1,-3) node[midway, left] {\colorlabelsize $\exrandom$};
    \draw[->-] (2.5,-1) -- (2.5,-3) node[midway, left] {\colorlabelsize $\secexrandom$};
    \draw[->-] (6,-1) -- (6,-3) node[midway, left] {\colorlabelsize $\thirdexrandom$};

    \node[anchor=center] (text) at (9,0) {\corelabelsize ${=}$};

    \draw (11,1) rectangle (18,-1);
    \node[anchor=center] (text) at (14.5,0) {\corelabelsize $\condprobof{\exrandom}{\thirdexrandom}$};
    \draw[->-] (12,-1) -- (12,-3) node[midway, left] {\colorlabelsize $\exrandom$};
    \draw[-<-] (17,-1) -- (17,-3) node[midway, left] {\colorlabelsize $\thirdexrandom$};

    \draw (21,1) rectangle (25,-1);
    \node[anchor=center] (text) at (23,0) {\corelabelsize $\probat{\thirdexrandom}$};
    \draw[->-] (23,-1) -- (23,-3) node[midway, left] {\colorlabelsize $\thirdexrandom$};

    \draw (23,-3) -- (23,-5);
    \draw[fill] (23,-5) circle (\dotsize);
    \draw (23,-5) -- (23,-7) node[midway, left] {\colorlabelsize $\thirdexrandom$};
    \draw (17,-3) to[bend right=40] (23,-5);
    \draw (29,-3) to[bend right=-40] (23,-5);

    \draw (28,1) rectangle (35,-1);
    \node[anchor=center] (text) at (31.5,0) {\corelabelsize $\condprobof{\secexrandom}{\thirdexrandom}$};
    \draw[-<-] (29,-1) -- (29,-3) node[midway, left] {\colorlabelsize $\thirdexrandom$};
    \draw[->-] (34,-1) -- (34,-3) node[midway, left] {\colorlabelsize $\secexrandom$};


\end{tikzpicture}
    \end{center}
\end{corollary}

This conditional independence pattern is the basic local building block that is generalized in Markov networks, which we define in the following.

\begin{definition}[Markov network]
    \label{def:markovNetwork}
    Let $\tnetof{\graph}$ be a tensor network of non-negative tensors decorating a hypergraph $\graph$.
    Then the Markov network $\probof{\graph}$ to $\tnetof{\graph}$ is the probability distribution of $\nodevariables$ defined by the tensor
    \begin{align*}
        \probofat{\graph}{\nodevariables} = \frac{
            \contractionof{\{\hypercoreof{\edge} \wcols \edgein\}}{\nodevariables}
        }{
            \contraction{\{\hypercoreof{\edge} \wcols \edgein\}}
        } = \normalizationof{\tnetof{\graph}}{\nodevariables} \, .
    \end{align*}
    We call the denominator
    \begin{align*}
        \partitionfunctionof{\tnetof{\graph}} = \contraction{\{\hypercoreof{\edge} \wcols \edgein\}}
    \end{align*}
    the partition function of the tensor network $\tnetof{\graph}$.
\end{definition}

We define graphical models based on hypergraphs to establish a direct connection with tensor networks decorating the hypergraph.
In a more canonical way, Markov networks are instead defined by graphs, where instead of the edges the cliques are decorated by factor tensors (see for example \cite{koller_probabilistic_2009}).
Following this alternative description, the graphs of the tensor networks are dual to the graphs of the graphical models \cite{robeva_duality_2019,glasser_expressive_2019}.

We can interpret the factors $\hypercorewith$ as activation cores placed on the hyperedges $\edge$ of the graph.
The global activation tensor (and hence the joint distribution) is obtained by contracting this activation network and normalizing by its partition function.

While so far we have defined Markov networks as decomposed probability distributions, we now want to derive assumptions on a distribution, assuring that such decompositions exist.
The sets of conditional independencies encoded by a hypergraph are captured by its separation properties, as we define next.

\begin{definition}[Separation of hypergraph]
    A path in a hypergraph is a sequence of nodes $\node_{\catenumerator}$ for $\catenumeratorin$, such that for any $\catenumerator\in[\atomorder-1]$ we find a hyperedge $\edgein$ such that $(\node_{\catenumerator}, \node_{\catenumerator+1})\subset \edge$.
    Given disjoint subsets $\nodesa$, $\nodesb$, $\nodesc$ of nodes in a hypergraph $\graph$, we say that $\nodesc$ separates $\nodesa$ and $\nodesb$ with respect to $\graph$ when any path starting at a node in $\nodesa$ and ending in a node in $\nodesb$ contains a node in $\nodesc$.
\end{definition}

To characterize Markov networks in terms of conditional independencies, we need to further define the property of clique-capturing.
This property establishes a correspondence of hyperedges with maximal cliques in the more canonical graph-based definition of Markov networks \cite{koller_probabilistic_2009}.

\begin{definition}[Clique-capturing hypergraph]
    \label{def:ccHypergraph}
    We call a hypergraph $\graph$ \emph{clique-capturing}, if the following holds:
    Each subset $\secnodes\subset\nodes$, which fulfills that for any $\nodea,\nodeb\in\secnodes$ with $\nodea \neq \nodeb$, there is a hyperedge $\edgein$ with $\nodea,\nodeb\in\edge$, is contained in a hyperedge.
\end{definition}

We are now ready to state the Hammersley-Clifford theorem characterizing the sets of Markov networks on a hypergraph by conditional independence.

\begin{theorem}[Hammersley-Clifford factorization theorem]
    \label{the:factorizationHammersleyClifford}
    Let there be a positive probability distribution $\probwithnodes$ and a clique-capturing hypergraph $\graph=(\nodes,\edges)$.
    Then the following are equivalent:
    \begin{itemize}
        \item[i] The distribution $\probwithnodes$ is representable by a Markov network on $\graph$, that is for each edge $\edgein$ there is a tensor $\hypercoreofat{\edge}{\catvariableof{\edge}}$ such that
        \begin{align*}
            \probwithnodes = \normalizationof{\{\hypercoreofat{\edge}{\catvariableof{\edge}}\wcols\edgein\}}{\nodevariables} \, .
        \end{align*}
        \item[ii] For any subsets $\nodesa,\nodesb,\nodesc\subset\nodes$ such that $\nodesc$ separates $\nodesa$ from $\nodesb$, we have
        \begin{align*}
            \condindependent{\catvariableof{\nodesa}}{\catvariableof{\nodesb}}{\catvariableof{\nodesc}} \, .
        \end{align*}
    \end{itemize}
\end{theorem}
\begin{proof}
    This is shown in Appendix~\secref{sec:proofFactorizationTheorems}.
\end{proof}

By \theref{the:factorizationHammersleyClifford} the conditional independence structure of $\probwithnodes$ determines a global tensor network decomposition of $\probwithnodes$.
Note that the assumption of a positive distribution is required (i.e. for all $\shortcatindices$ we have $\probat{\indexedshortcatvariables}>0$).
The assumption of positivity is however not required in our characterization of independencies and conditional independencies by the existence of corresponding tensor decompositions (see \defref{def:independence} and \defref{def:condIndependence}).

\begin{example}[Independent boolean variables]
    \label{exa:coinTossHC}
    Let there be $\catorder$ boolean variables $\shortcatvariables$, which are i.i.d. drawn from a positive distribution $\probat{\catvariable}$.
    From the pairwise independencies of $\catvariableof{\catenumerator}$ it follows with the Hammersley-Clifford Factorization \theref{the:factorizationHammersleyClifford} that the distribution is representable by an elementary tensor network, that is
    \begin{align*}
        \probwith = \bigotimes_{\catenumeratorin} \probofat{\catenumerator}{\catvariableof{\catenumerator}} \, .
    \end{align*}
    The corresponding hypergraph is the elementary graph, with respect to which any two disjoint subsets of nodes are separated (see \figref{fig:ELDecomposition}).
    \begin{figure}
        \begin{center}
            \begin{tikzpicture}[scale=0.35,thick]
                \begin{scope}[shift={(-19,-2)}]
                    \coordinate[label=left:$a)$] (A) at (-2,2);

                    \node[circle, draw, thick, fill=\nodegrayscale, minimum size = \nodeminsize] (A) at (0,0) {};
                    \node[anchor=center] (A) at (0,0) {\corelabelsize $\catvariableof{0}$};
                    \node[anchor=center] (A) at (0,1.75) {\corelabelsize $\edgeof{0}$};

                    \node[circle, draw, thick, fill=\nodegrayscale, minimum size = \nodeminsize] (A) at (4,0) {};
                    \node[anchor=center] (A) at (4,0) {\corelabelsize $\catvariableof{1}$};
                    \node[anchor=center] (A) at (4,1.75) {\corelabelsize $\edgeof{1}$};

                    \coordinate[label=below:$\hdots $] (A) at (7,0.5);

                    \node[circle, draw, thick, fill=\nodegrayscale, minimum size = \nodeminsize] (A) at (10,0) {};
                    \node[] (text) at (10,0) {\corelabelsize $\catvariableof{\catorder\shortminus1}$};
                    \node[anchor=center] (A) at (10,1.75) {\corelabelsize $\edgeof{\catorder\shortminus1}$};
                \end{scope}

                \coordinate[label=left:$b)$] (A) at (-4,0);

                \begin{scope}[shift={(5,-2)}]
                    \draw (-3,-1) rectangle (-9,1);
                    \node[anchor=center] at (-6,0) {\corelabelsize $\probtensor$};
                    \draw (-4,-1)--(-4,-2.5) node[midway,right] {\colorlabelsize $\catvariableof{\catorder\shortminus1}$};
                    \node[anchor=center] (text) at (-5,-2.25) {$\cdots$};
                    \draw (-7,-1)--(-7,-2.5) node[midway,right] {\colorlabelsize $\catvariableof{1}$};
                    \draw (-8,-1)--(-8,-2.5) node[midway,left] {\colorlabelsize $\catvariableof{0}$};

                    \node[anchor=center] (A) at (-1.5,0) {${=}$};

                    \draw (0,-1) rectangle (2,1);
                    \node[anchor=center] (A) at (1,0) {\corelabelsize $\probof{0}$};
                    \draw (1,-1)--(1,-2.5) node[midway,right] {\colorlabelsize $\catvariableof{0}$};

                    \draw (3,-1) rectangle (5,1);
                    \node[anchor=center] (A) at (4,0) {\corelabelsize $\probof{1}$};
                    \draw (4,-1)--(4,-2.5) node[midway,right] {\colorlabelsize $\catvariableof{1}$};

                    \node[anchor=center] (text) at (7,0) {$\hdots$};

                    \draw (9,-1) rectangle (11,1);
                    \node[anchor=center] (A) at (10,0) {\corelabelsize $\probof{\catorder\shortminus1}$};
                    \draw (10,-1)--(10,-2.5) node[midway,right] {\colorlabelsize $\catvariableof{\catorder\shortminus1}$};
                \end{scope}
            \end{tikzpicture}
        \end{center}
        \caption{Decomposition of a probability distribution with independent variables (see \exaref{exa:coinTossHC}).
        The independencies are captured by the elementary hypergraph a), whose edges contain single nodes.
        The corresponding tensor $\probwith$ is then represented by a Markov network on the elementary hypergraph, where each factor is the marginal distribution of the corresponding variable as visualized in b).}\label{fig:ELDecomposition}
    \end{figure}
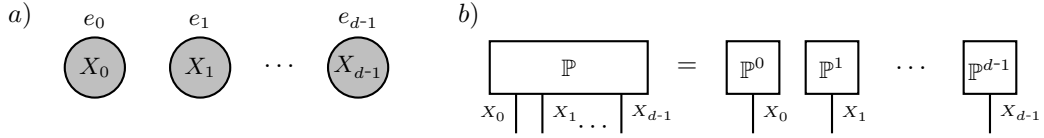
\end{example}

\begin{example}\label{exa:studentHC}
    We consider a classical example of a graphical model (see \cite[Example~4.3]{koller_probabilistic_2009}):
    A student of intelligence ($\catvariableof{I}$) and SAT score ($\catvariableof{S}$) is assigned a test of difficulty ($\catvariableof{D}$), for which he gets a grade ($\catvariableof{G}$) depending on which he gets a recommendation letter ($\catvariableof{L}$) by its teacher.
    We make the following modelling assumptions:
    \begin{itemize}
        \item "The SAT score depends only on the students intelligence": $\condindependent{\catvariableof{S}}{\catvariableof{\{D,G,L\}}}{\catvariableof{I}}$.
        \item "The recommendation letter depends only on the grade": $\condindependent{\catvariableof{L}}{\catvariableof{\{D,I,S\}}}{\catvariableof{G}}$.
    \end{itemize}
    The associated hypergraph capturing these conditional independencies is drawn in \figref{fig:studentHypergraph} a).
\end{example}
\begin{figure}[t]
    \begin{center}
        \begin{tikzpicture}

            \node[anchor=center] (A) at (-2,2.2) {$a)$};

            \node[anchor=center] (A) at (1,0.4) {$\edgeof{0}$};
            \draw[thick] (0,0) -- (1,2/3);
            \draw[thick] (2,0) -- (1,2/3);
            \draw[thick] (1,2) -- (1,2/3);

            \draw[thick] (2,0) -- (3,2) node[midway,right] {$\edgeof{1}$};
            \draw[thick] (1,2) -- (-1,2) node[midway,above] {$\edgeof{2}$};

            \node[circle, draw, thick, fill=\nodegrayscale, minimum size = \nodeminsize] (A) at (0,0) {};
            \node[anchor=center] (A) at (0,0) {\corelabelsize $\catvariableof{D}$};

            \node[circle, draw, thick, fill=\nodegrayscale, minimum size = \nodeminsize] (A) at (2,0) {};
            \node[anchor=center] (A) at (2,0) {\corelabelsize $\catvariableof{I}$};

            \node[circle, draw, thick, fill=\nodegrayscale, minimum size = \nodeminsize] (A) at (1,2) {};
            \node[anchor=center] (A) at (1,2) {\corelabelsize $\catvariableof{G}$};

            \node[circle, draw, thick, fill=\nodegrayscale, minimum size = \nodeminsize] (A) at (3,2) {};
            \node[anchor=center] (A) at (3,2) {\corelabelsize $\catvariableof{S}$};

            \node[circle, draw, thick, fill=\nodegrayscale, minimum size = \nodeminsize] (A) at (-1,2) {};
            \node[anchor=center] (A) at (-1,2) {\corelabelsize $\catvariableof{L}$};

            \begin{scope}[shift={(7,0)}]
                \node[anchor=center] (A) at (-2,2.2) {$b)$};

                \draw[thick] (0,0) -- (2,0);
                \draw[thick] (0,0) -- (1,2);
                \draw[thick] (2,0) -- (1,2);

                \draw[thick, dashed, rounded corners=15pt] (1,3) -- (-0.8,-0.5) -- (2.8,-0.5) -- cycle;

                \node (of) at (0.5,-0.25) {};
                \draw[thick, dashed, rounded corners=10pt]  ($(3.25,2.5)+(of)$) -- ($(3.25,2.5)-(of)$)  -- ($(1.75,-0.5)-(of)$) -- ($(1.75,-0.5)+(of)$) -- cycle;
                \draw[thick] (2,0) -- (3,2); 

                \node (of) at (0,0.5) {};
                \draw[thick, dashed, rounded corners=10pt]  ($(1.5,2)+(of)$) -- ($(1.5,2)-(of)$)  -- ($(-1.5,2)-(of)$) -- ($(-1.5,2)+(of)$) -- cycle;
                \draw[thick] (1,2) -- (-1,2); 

                \node[circle, draw, thick, fill=\nodegrayscale, minimum size = \nodeminsize] (A) at (0,0) {};
                \node[anchor=center] (D) at (0,0) {\corelabelsize $\catvariableof{D}$};

                \node[circle, draw, thick, fill=\nodegrayscale, minimum size = \nodeminsize] (A) at (2,0) {};
                \node[anchor=center] (I) at (2,0) {\corelabelsize $\catvariableof{I}$};

                \node[circle, draw, thick, fill=\nodegrayscale, minimum size = \nodeminsize] (A) at (1,2) {};
                \node[anchor=center] (G) at (1,2) {\corelabelsize $\catvariableof{G}$};

                \node[circle, draw, thick, fill=\nodegrayscale, minimum size = \nodeminsize] (A) at (3,2) {};
                \node[anchor=center] (S) at (3,2) {\corelabelsize $\catvariableof{S}$};

                \node[circle, draw, thick, fill=\nodegrayscale, minimum size = \nodeminsize] (A) at (-1,2) {};
                \node[anchor=center] (L) at (-1,2) {\corelabelsize $\catvariableof{L}$};

            \end{scope}
        \end{tikzpicture}
    \end{center}
    \caption{
        Hypergraph a) capturing the conditional independencies of the student example.
        The cliques of the node adjacency graph are highlighted in b) and coincide with hyperedges of the hypergraph.
        The hypergraph is therefore clique-capturing (see \defref{def:ccHypergraph}).
    }\label{fig:studentHypergraph}
\end{figure}
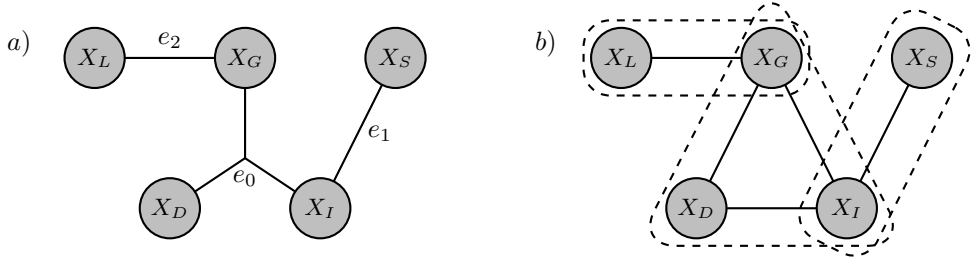

\subsection{Factorization based on sufficient statistics}

Let us now introduce sufficient statistics towards studying further tensor network decompositions of probability distributions.

\begin{definition}
    \label{def:sufStatistic}
    Let $\probat{\catvariable,\thirdcatvariable}$ be a joint distribution of the $\catdim$-dimensional variable $\catvariable$ and the $\thirdcatdim$-dimensional variable $\thirdcatvariable$ and let
    \begin{align*}
        \sstat \defcols [\catdim] \rightarrow [\headdim]
    \end{align*}
    be a statistic.
    We are interested in the distribution $\probat{\catvariable,\thirdcatvariable,\headvariableof{\sstat}}=\contractionof{\probat{\catvariable,\thirdcatvariable},\bencodingofat{\sstat}{\headvariableof{\sstat},\catvariable}}{\catvariable,\thirdcatvariable,\headvariableof{\sstat}}$.
    We say that $\sstat$ is a sufficient statistic for $\thirdcatvariable$ if $\catvariable$ is independent of $\thirdcatvariable$ conditioned on $\headvariableof{\sstat}$.
\end{definition}

\begin{example}[Sufficient statistics for the probability]\label{exa:sufStatProb}
    Let $\thirdcatvariable$ be the value $\probat{\indexedshortcatvariables}$ when drawing $\shortcatvariables$ from $\probwith$.
    Then $\sstat$ is a sufficient statistic for $\thirdcatvariable=\probwith$ if for all $\headindex$ in the image of $\sstat$ we have
    \begin{align*}
        \condprobof{\indexedshortcatvariables}{\headvariableof{\sstat}=\headindex} =
        \begin{cases}
            \frac{1}{\cardof{\{\shortcatindices \wcols \sstatat{\shortcatindices}=\headindex\}}} & \ifspace \sstat(\shortcatindices)=\headindex \\
            0 & \text{else}
        \end{cases} \, .
    \end{align*}
    When knowing the value $\sstatat{\shortcatindices}$ of the sufficient statistic at a given index $\shortcatindices$, we then also know the probability $\probat{\indexedshortcatvariables}$.
    The function $\sstat$ is thus a sufficient statistic for $\thirdcatvariable=\probwith$ if and only if there is a tensor $\acttensorwith$ with
    \begin{align*}
        \probat{\shortcatvariables}
        = \contractionof{\bencodingofat{\sstat}{\headvariables,\shortcatvariables},\acttensorwith}{\shortcatvariables} \, .
    \end{align*}
\end{example}

\exaref{exa:sufStatProb} hints at a connection between sufficient statistics and decompositions into \CompActNets{}.
More generally, such decompositions are provided by the Fisher-Neyman Factorization Theorem.

\begin{theorem}[Fisher-Neyman Factorization Theorem] 
    \label{the:factorizationFisherNeyman}
    Let $\probtensor$ be a joint distribution of variables $\catvariable,\thirdcatvariable$ with values $\valof{\catvariable},\,\valof{\thirdcatvariable}$.
    Let there further be a finite set $\valof{\headvariableof{\sstat}}$, which is enumerated with a variable $\headvariableof{\sstat}$.
    Then $\sstat\defcols \valof{\catvariable} \rightarrow \valof{\headvariableof{\sstat}}$ is a sufficient statistic for $\thirdcatvariable$ if and only if there are tensors $\basemeasureat{\catvariable}$ and $\acttensorat{\headvariableof{\sstat},\thirdcatvariable}$ such that
    \begin{align*}
        \probat{\catvariable,\thirdcatvariable}
        = \contractionof{
            \acttensorat{\headvariableof{\sstat},\thirdcatvariable},\bencodingofat{\sstat}{\headvariableof{\sstat},\catvariable},\basemeasureat{\catvariable}
        }{\catvariable,\thirdcatvariable} \, .
    \end{align*}
    We depict this equation diagrammatically by
    \begin{center}
        \begin{tikzpicture}[scale = 0.35, thick]

    \draw (1,-2) rectangle (8,-4);
    \node[anchor=center] (text) at (4.5,-3) {\corelabelsize $\probtensor$};
    \draw (2.5,-4) -- (2.5,-6) node[midway, left] {\colorlabelsize $\catvariable$};
    \draw (6.5,-4) -- (6.5,-6) node[midway, right] {\colorlabelsize $\thirdcatvariable$};

    \node[anchor=center] (text) at (10,-3) {${=}$};

    \begin{scope}[shift={(15,-1)}]

        \draw (-3,-1.5) rectangle (-1,0.5);
        \node[anchor=center] (text) at (-2,-0.5) {\corelabelsize $\basemeasure$};
        \draw (-2,-1.5) -- (-2,-4);
        \draw (-2,-4) -- (-2,-5) node[midway, left] {\colorlabelsize $\catvariable$};

        \drawvariabledot{-2}{-2.5}
        \draw[] (-2,-2.5) -- (0,-2.5);
        \draw[->-] (-1.5,-2.5) -- (0,-2.5);
        \draw (0,-1.5) rectangle (2,-3.5);
        \node[anchor=center] (text) at (1,-2.5) {\corelabelsize $\bencodingof{\sstat}$};

        \draw[->-] (2,-2.5) -- (4,-2.5);
        \drawvariabledot{3.5}{-2.5}
        \draw[] (2,-2.5) -- (5,-2.5) node[midway, below] {\colorlabelsize $\headvariableof{\sstat}$};

        \draw (5,-1.5) rectangle (7,-3.5);
        \node[anchor=center] (text) at (6,-2.5) {\corelabelsize $\acttensor$};
        \draw (6,-3.5) -- (6,-5) node[midway, right] {\colorlabelsize $\thirdcatvariable$};
    \end{scope}


\end{tikzpicture}
    \end{center}
\end{theorem}
\begin{proof}
    Shown in the appendix, see \theref{the:generalFactorizationFisherNeyman}.
\end{proof}

Note that the definition of sufficient statistic does not make use of the marginal distribution $\probat{\thirdcatvariable}$.
We therefore can define sufficient statistics also for families of distributions $\condprobat{\catvariable}{\thirdcatvariable}$ with respect to arbitrary non-degenerate marginal distributions $\probat{\thirdcatvariable}$.
We then use \theref{the:factorizationFisherNeyman} to embed such families in \CompActNets{}.

\begin{corollary}
    Let $\condprobof{\shortcatvariables}{\thirdcatvariable}$ be an arbitrary family of distributions of $\shortcatvariables$ and $\sstat$ a sufficient statistic for $\thirdcatvariable$.
    Then, there is a tensor $\basemeasurewith$ and an activation tensor $\acttensorat{\headvariables,\thirdcatvariable}$ such that for any $\thirdcatindex\in\valof{\thirdcatvariable}$ we have
    \begin{align*}
        \condprobat{\shortcatvariables}{\thirdcatvariable=\thirdcatindex} \in\cansof{\sstat,\maxgraph, \acttensor} \, . 
    \end{align*}
\end{corollary}

The Factorization Theorem of Fisher-Neyman provides the fundamental motivation for the \CompActNets{} architecture.
Any decomposition of $\bencodingofat{\sstat}{\headvariables,\shortcatvariables}$ is called a \emph{computation network} and common to all members of a family with sufficient statistic $\sstat$.
The activation tensor $\acttensorat{\headvariables,\indexedthirdcatvariable}$, whose decomposition is called the \emph{activation network}, is specific to each member of the family.
We now show in two examples how families of distributions can be represented in \CompActNets{} by sufficient statistics.

\begin{example}[Order statistic for boolean variables]
    \label{exa:coinTossFN}
    Let there be $\catorder$ boolean variables $\shortcatvariables$ and a family $\condprobat{\shortcatvariables}{\thirdcatvariable}$ of distributions.
    The order statistic assigns the ordered tuple to each tuple $\shortcatindices$. The ordered tuple effectively counts the number of $1$ coordinates in the tuple $\shortcatindices$, that is the statistic
    \begin{align*}
        \sstatof{+} \defcols \facstates \rightarrow [\seldim] \, , \quad
        \sstatof{+}(\shortcatindices) = \cardof{\{\catenumerator \wcols \catindexof{\catenumerator}=1\}} \, .
    \end{align*}
    When the order statistic is sufficient for $\thirdcatvariable$, the detailed order of the outcomes is uninformative about the member $\thirdcatindex$ from which the random variables have been drawn.
    By the Fisher-Neyman Factorization \theref{the:factorizationFisherNeyman}, $\sstatof{+}$ is a sufficient statistic if and only if there are tensors $\basemeasurewith$ and $\acttensorat{\headvariableof{+},\Theta}$ such that for each $\theta\in\Theta$
    \begin{align*}
        \condprobat{\shortcatvariables}{\indexedthirdcatvariable}
        = \contractionof{
            \acttensorat{\headvariableof{+},\indexedthirdcatvariable},\bencodingofat{\sstatof{+}}{\headvariableof{+},\shortcatvariables},\basemeasurewith}{\shortcatvariables} \, .
    \end{align*}
    For each member $\thirdcatindex$ of the family, the probability of each sequence $\shortcatindices$ is thus the product of a base measure factor $\basemeasureat{\indexedshortcatvariables}$ and a factor $\acttensorat{\headvariableof{+}=+(\shortcatindices),\indexedthirdcatvariable}$ depending only on the count $+(\shortcatindices)$ of $1$ coordinates in $\shortcatindices$.
    We later continue this example in \exaref{exa:coinTossExp}, where further interpretations to the case of i.i.d. variables are provided.
\end{example}

\begin{example}[Graphical models as a special case of \CompActNets{}]
    For graphical models we take the \emph{identity statistic}
    \begin{align*}
        \identity\big(\shortcatindices\big)
        = \shortcatindices
    \end{align*}
    so that the image coordinates coincide with the variables and there are no non-trivial computation cores.
    The associated basis encoding is just the identity tensor
    \begin{align*}
        \bencodingofat{\identity}{\headvariableof{[\catorder]},\shortcatvariables}
        = \identityat{\shortcatvariables,\headvariableof{[\catorder]}} \, .
    \end{align*}
    Therefore, for any activation tensor $\acttensorwith$ we obtain
    \begin{align*}
        \probwith
        = \normalizationof{\acttensorwith,\bencodingofat{\identity}{\headvariableof{[\catorder]},\shortcatvariables}}{\shortcatvariables}
        = \normalizationof{\acttensorat{\shortcatvariables}}{\shortcatvariables}\, .
    \end{align*}
    Put differently, in the graphical–model case the activation tensor coincides with the joint distribution tensor.
    In this setting, structural properties of the distribution such as (conditional) independences can be read off as algebraic factorization patterns of the activation (and hence joint) tensor.
\end{example}

\subsection{Exponential families}

We now show that exponential families are specific instances of \CompActNets{}, whose activation tensors have elementary decompositions.
The importance of exponential families in statistics stems from their universal properties.
A classical theorem by Pitman, Koopman and Darmois (see \cite{pitman_sufficient_1936}) states, that whenever a family exhibits constant support and a finite sufficient statistic for arbitrary large data sets, then it is in an exponential family.
For a discussion of further universal properties of exponential families such as the existence of priors and entropy maximizers, see \cite{murphy_probabilistic_2022}.

\begin{definition}[Exponential family]\label{def:expFamily}
    Given a base measure $\basemeasure$ and a statistic
    \begin{align*}
        \sstat:\facstates\rightarrow\parspace
    \end{align*}
    we enumerate for each coordinate $\selindexin$ the image $\imageof{\sstatcoordinateof{\selindex}}$ by an interpretation map
    \begin{align*}
        \indexinterpretationof{\selindex} \defcols
        [\cardof{\imageof{\sstatcoordinateof{\selindex}}}] \rightarrow \imageof{\sstatcoordinateof{\selindex}} \, .
    \end{align*}
    For any canonical parameter vector $\canparamwithin$, we build the activation cores $\softactlegwith$ for each coordinate $\headindexof{\selindex}\in[\cardof{\imageof{\sstatcoordinateof{\selindex}}}]$ by
    \begin{align*}
        \softactlegat{\indexedheadvariableof{\selindex}}
        = \expof{\canparamat{\indexedselvariable} \cdot \indexinterpretationofat{\selindex}{\headindexof{\selindex}}} \,
    \end{align*}
    and define the distribution 
    \begin{align*}
        \expdistwith =
        \normalizationof{\{\basemeasurewith\} \cup \{\bencodingofat{\sstatcoordinateof{\selindex}}{\headvariableof{\selindex},\shortcatvariables} \wcols \selindexin\}\cup\{\softactlegwith \wcols \selindexin\}}{\shortcatvariables} \, .
    \end{align*}
    We then call the tensor $\probfamilyofat{\sstat,\basemeasure}{\shortcatvariables}{\Theta}$ with $\valof{\Theta}=\parspace$ and slices for $\canparam\in\parspace$ given by
    \begin{align*}
        \probfamilyofat{\sstat,\basemeasure}{\shortcatvariables}{\Theta=\canparam}
        = \expdistwith
    \end{align*}
    the exponential family to the statistic $\sstat$ and the base measure $\basemeasure$.
\end{definition}

To see that \defref{def:expFamily} is consistent with the typical definition of exponential families (see \cite{brown_fundamentals_1987}), note that for each $\canparam\in\parspace$ the normalization amounts to the division by
\begin{align*}
    \partitionfunctionof{\canparam}
    = \contraction{\{\basemeasurewith\} \cup \{\bencodingofat{\sstatcoordinateof{\selindex}}{\headvariableof{\selindex},\shortcatvariables} \wcols \selindexin\}\cup\{\softactlegwith \wcols \selindexin\}} \, ,
\end{align*}
a quantity which is referred to as the partition function.
Then, we have for each coordinate $\shortcatindicesin$ that
\begin{align*}
    &\expdistat{\indexedshortcatvariables}\\
    &= \frac{1}{\partitionfunctionof{\canparam}}
    \contractionof{\{\basemeasurewith\} \cup \{\bencodingofat{\sstatcoordinateof{\selindex}}{\headvariableof{\selindex},\shortcatvariables} \wcols \selindexin\}\cup\{\softactlegwith \wcols \selindexin\}}{\indexedshortcatvariables} \\
    &= \frac{1}{\partitionfunctionof{\canparam}} \cdot \basemeasureat{\indexedshortcatvariables} \cdot \expof{\sum_{\selindexin} \canparamat{\indexedselvariable} \cdot \sstatcoordinateofat{\selindex}{\indexedshortcatvariables}} \, .
\end{align*}

Note that by construction each member of an exponential family is an element in a \CompActNet{} with elementary activation cores, that is
\begin{align*}
    \probfamilyofat{\sstat,\basemeasure}{\shortcatvariables}{\Theta=\canparam}
    \in \cansof{\sstat,\elgraph,\basemeasure} \, .
\end{align*}

\begin{figure}
    \begin{center}
        \begin{tikzpicture}[scale=0.35,thick,xscale=1.2] 

    \begin{scope}
    [shift={(-13,0)}]
        \draw (-2,-1) rectangle (6,-3);
        \node[anchor=center] (text) at (2,-2) {\corelabelsize $ \expdist$}; 
        \draw[-<-] (0,-3)--(0,-5) node[midway,left] {\colorlabelsize $\catvariableof{0}$};
        \draw[-<-] (1.5,-3)--(1.5,-5) node[midway,left] {\colorlabelsize $\catvariableof{1}$};
        \node[anchor=center] (text) at (3,-4) {$\cdots$};
        \draw[-<-] (4,-3)--(4,-5) node[midway,right] {\colorlabelsize $\catvariableof{\atomorder\shortminus1}$};

        \node[anchor=center] (text) at (9,-2) {$= \,\,\frac{1}{\partitionfunctionof{\canparam}}\,\cdot$};
    \end{scope}

    \draw (-1.25,1) rectangle (1.25,3);
    \node[anchor=center] (text) at (0,2) {\corelabelsize $\softactsymbolof{0,\canparam}$};

    \draw (2.75,1) rectangle (5.25,3);
    \node[anchor=center] (text) at (4,2) {\corelabelsize $\softactsymbolof{\seldim\shortminus 1,\canparam}$};

    \draw[->-] (0,-1)--(0,0);
    \node[left] (text) at (0,0) {\colorlabelsize $\headvariableof{0}$};
    \draw[] (0,0)--(0,1);
    \drawvariabledot{0}{0}
    \node[anchor=center] (text) at (2,0) {$\cdots$};

    \draw[->-] (4,-1)--(4,0);
    \node[right] (text) at (4,0) {\colorlabelsize $\headvariableof{\seccatorder\shortminus1}$};
    \draw[] (4,0)--(4,1);
    \drawvariabledot{4}{0}

    \draw (-1,-1) rectangle (5,-3);
    \node[anchor=center] (text) at (2,-2) {\corelabelsize $\bencodingof{\sstat}$};
    \draw[-<-] (0,-3)--(0,-5) node[midway,left] {\colorlabelsize $\catvariableof{0}$};
    \draw[-<-] (1.5,-3)--(1.5,-5) node[midway,left] {\colorlabelsize $\catvariableof{1}$};
    \node[anchor=center] (text) at (3,-4) {$\cdots$};
    \draw[-<-] (4,-3)--(4,-5) node[midway,right] {\colorlabelsize $\catvariableof{\atomorder\shortminus1}$};

    \begin{scope}
    [shift={(0,-4)}]
        \draw[] (0,1)--(0,-6);
        \node[below] (text) at (0,-6) {\colorlabelsize $\catvariableof{0}$};
        \drawvariabledot{0}{-5}
        \draw[] (1.5,1)--(1.5,-6);
        \node[below] (text) at (1.5,-6) {\colorlabelsize $\catvariableof{1}$};
        \drawvariabledot{1.5}{-3}
        \node[anchor=center] (text) at (3,0) {$\cdots$};
        \node[anchor=center] (text) at (3,-5.5) {$\cdots$};
        \draw[] (4,1)--(4,-6);
        \node[below] (text) at (4,-6) {\colorlabelsize $\catvariableof{\atomorder\shortminus1}$};
        \drawvariabledot{4}{-2}

        \draw[] (0,-5) -- (6,-5);
        \draw[] (1.5,-3) -- (6,-3);
        \node[anchor=center] (text) at (5,-3.75) {$\vdots$};
        \draw[] (4,-2) -- (6,-2);
        \draw (6,-1) rectangle (9, -6);
        \node[anchor=center] (text) at (7.5,-3.5) {\corelabelsize $\basemeasure$};

        \node[anchor=center] (text) at (10,-4.5) {$.$};

    \end{scope}

%
%
%

\end{tikzpicture}
    \end{center}
    \caption{
        Tensor Network diagram of a member of an exponential family $\probfamilyofat{\sstat,\basemeasure}{\shortcatvariables}{\Theta=\canparam}$ before normalization as an \CompActNet{} with elementary activation, that is an element in $\cansof{\sstat,\elgraph,\basemeasure}$.}\label{fig:expdistElementary}
\end{figure}

\begin{example}[Exponential family of coin tosses]
    \label{exa:coinTossExp}
    Recall the family of distributions of boolean $\shortcatvariables$ from \exaref{exa:coinTossFN}, which has the order statistic $\sstatof{+}$ as a sufficient statistic.
    We now in addition assume that the variables $\shortcatvariables$ are i.i.d. with respect to any member of the family (see \exaref{exa:coinTossFN}).
    For the variables to be i.i.d., we need $\basemeasurewith=\oneswith$ and can thus choose a representation such that for $\thirdcatindex\in\valof{\thirdcatvariable}$
    \begin{align*}
        \condprobat{\shortcatvariables}{\indexedthirdcatvariable}
        = \contractionof{\bencodingofat{\sstatof{+}}{\headvariableof{+},\shortcatvariables},\acttensorat{\headvariableof{+},\indexedthirdcatvariable}}{\shortcatvariables}
    \end{align*}
    where for each $\catenumerator\in[\catorder+1]$
    \begin{align*}
        \acttensorat{\headvariableof{+}=\catenumerator,\indexedthirdcatvariable}
        = (1-\thirdcatindex)^{\catorder-\catenumerator} \cdot \thirdcatindex^{\catenumerator} \, .
    \end{align*}
    The marginal distribution $\condprobof{\headvariableof{+}}{\indexedthirdcatvariable}$ is then the binominal distribution $B(\catorder,\thirdcatindex)$.
    When excluding the case of $\thirdcatindex\in\{0,1\}$, the family is a subset of the exponential family with the head count statistic, where each member is reparametrized by
    \begin{align*}
        \canparam \coloneqq \lnof{\frac{\thirdcatindex}{1-\thirdcatindex}} \, .
    \end{align*}
    To see that this is true, we observe that the coordinate $\headindexof{+}\in[\catorder+1]$ of the activation tensor of $\expdistof{(\sstatof{+},\canparam,\trivbm)}$ is
    \begin{align*}
        \softacttensorat{\indexedheadvariableof{+}}
        = \expof{\headindexof{+}\cdot\canparam}
        = \frac{\thirdcatindex^{\headindexof{+}}}{(1-\thirdcatindex)^{\headindexof{+}}} \, .
    \end{align*}
    Now, with $\partitionfunctionof{\canparam}=\frac{1}{(1-\thirdcatindex)^{\catorder}}$ we have for any $\shortcatindices$ with $\sum_{\catenumeratorin}\catindexof{\catenumerator}=\headindexof{+}$ that
    \begin{align*}
        \frac{1}{\partitionfunctionof{\canparam}} \cdot \contraction{\bencodingofat{\sstatof{+}}{\headvariableof{+},\indexedshortcatvariables},\softacttensorat{\headvariableof{+}}}
        = (1-\thirdcatindex)^{\catorder} \cdot \frac{\thirdcatindex^{\headindexof{+}}}{(1-\thirdcatindex)^{\headindexof{+}}}
        =  \thirdcatindex^{\headindexof{+}} \cdot (1-\thirdcatindex)^{\catorder-\headindexof{+}}  \, .
    \end{align*}
    Comparing with the activation tensor $\acttensorat{\headvariableof{+}}$ above, we note that $\partitionfunctionof{\canparam}$ is the partition function of the exponential family and $\expdistofat{(\sstatof{+},\canparam,\trivbm)}{\shortcatvariables}$ coincides with the member $\condprobat{\shortcatvariables}{\indexedthirdcatvariable}$.
    We further observe that since the statistic $\sstatof{+}$ decomposes as a sum of terms depending on single variables only, we have a decomposition of the corresponding \CompActNet{} by
    \begin{center}
        \begin{tikzpicture}[scale=0.35,thick]
            \begin{scope}[shift={(-5,0)}]
                \draw (0.5,2.5) rectangle (-12.5,4.5);
                \node[anchor=center] at (-6,3.5) {\corelabelsize $
                \begin{bmatrix}
                    1 & \expof{\canparam} & \cdots & \expof{\catorder\cdot\canparam}
                \end{bmatrix}$};
                \draw[->-] (-6,1)--(-6,2.5) node[midway,right] {\colorlabelsize $\headvariableof{+}$};
                \draw (-3,-1) rectangle (-9,1);
                \node[anchor=center] at (-6,0) {\corelabelsize $\bencodingof{\sstatof{+}}$};
                \draw[-<-] (-4,-1)--(-4,-2.5) node[midway,right] {\colorlabelsize $\catvariableof{\catorder\shortminus1}$};
                \node[anchor=center] at (-5.5,-2.25) {$\cdots$};
                \draw[-<-] (-7,-1)--(-7,-2.5) node[midway,right] {\colorlabelsize $\catvariableof{1}$};
                \draw[-<-] (-8,-1)--(-8,-2.5) node[midway,left] {\colorlabelsize $\catvariableof{0}$};
            \end{scope}

            \node[anchor=center] (A) at (-3.5,0) {${=}$};

            \draw (-1.5,-1) rectangle (3.5,1);
            \node[anchor=center] (A) at (1,0) {\corelabelsize $
            \begin{bmatrix}
                1 & \expof{\canparam}
            \end{bmatrix}$};
            \draw (1,-1)--(1,-2.5) node[midway,right] {\colorlabelsize $\catvariableof{0}$};

            \draw (4.5,-1) rectangle (9.5,1);
            \node[anchor=center] (A) at (7,0) {\corelabelsize $
            \begin{bmatrix}
                1 & \expof{\canparam}
            \end{bmatrix}$};
            \draw (7,-1)--(7,-2.5) node[midway,right] {\colorlabelsize $\catvariableof{1}$};

            \node[anchor=center] (A) at (11,0) {${\cdots}$};

            \draw (12.5,-1) rectangle (17.5,1);
            \node[anchor=center] (A) at (15,0) {\corelabelsize $
            \begin{bmatrix}
                1 & \expof{\canparam}
            \end{bmatrix}$};
            \draw (15,-1)--(15,-2.5) node[midway,right] {\colorlabelsize $\catvariableof{\catorder\shortminus1}$};

        \end{tikzpicture}
    \end{center}
    This reproduces the fact that distributions of independent variables are representable by elementary tensors (see \exaref{exa:coinTossHC}).
\end{example}


\subsection{Efficient contractions by message passing}

Contractions of tensor networks are generally hard to solve.
Here, we investigate message passing algorithms, which decompose global contractions into a sequence of local contractions, whose results are passed as messages through the tensor network.
The resulting algorithm is called the Tree Belief Propagation (see \algoref{alg:treeBeliefPropagation}).
While various scheduling strategies for the message passing exist, we focus on the case of tree hypergraphs, for which exactness and efficiency can be shown.
We denote $\dirovedges$ to be all tuples $(\sedge,\redge)$ of hyperedges $\sedge,\redge\in\edges$ such that $\sedge\neq\redge$ and $\sedge\cap\redge\neq\varnothing$.
For our purposes we call a hypergraph $\graph$ a tree when the graph with nodes by the hyperedges $\edges$ and edges by $\dirovedges$ is a tree (for an example see \figref{fig:studentMessagePassingDirections}).

\begin{algorithm}[hbt!]
    \caption{Tree Belief Propagation}\label{alg:treeBeliefPropagation}
    \begin{algorithmic}
        \Require Tensor network $\extnet$ on a hypergraph $\graph$
        \Ensure Messages $\{\messagewith\wcols(\sedge,\redge)\in\dirovedges\}$
        \iosepline
        \State Initialize a message scheduler $\scheduler=\{(\sedge,\redge) \in \dirovedges \wcols \sedge \text{ a leaf in the tree } (\edges,\dirovedges)\}]$
        \While{$\scheduler$ not empty}
            \State Pop a $(\sedge,\redge)$ pair from $\scheduler$
            \State Compute the message
            \begin{align*}
                \messagewith
                = \contractionof{\{\hypercoreofat{\sedge}{\catvariableof{\sedge}}\}
                    \cup \{\mesfromtoat{\secsedge}{\sedge}{\catvariableof{\secsedge\cap\sedge}} \wcols (\secsedge,\sedge)\in\dirovedges \ncond \secsedge\neq \redge\}
                }{\catvariableof{\sedge\cap \redge}}
            \end{align*}
            \State Update $\scheduler$ by all messages $(\redge,\thirdsedge)$ which have not yet been sent, if all messages $(\secsedge,\redge)$ with $\secsedge\neq\thirdsedge$ have been sent.
        \EndWhile
        \State \Return Messages $\{\messagewith\wcols(\sedge,\redge)\in\dirovedges\}$
    \end{algorithmic}
\end{algorithm}

For an implementation of \algoref{alg:treeBeliefPropagation} in the \python{} package \tnreason{}, see \secref{sec:propAlgsImplementation}.

The following theorem states that the contraction of a whole tensor network can be replaced by local contractions with messages.
Since contracting the whole network can be infeasible, this shows that calculating the messages with the \algoref{alg:treeBeliefPropagation} can be advantageous.

\begin{theorem}
    \label{the:treeBeliefPropagationExactness}
    Let $\extnet$ be a tensor network on a tree hypergraph $\graph$ (i.e. the graph $(\edges,\dirovedges)$ is a tree).
    The messages in the tree belief propagation \algoref{alg:treeBeliefPropagation} are contracted to local marginals, meaning that for each $\sedge\in\edges$ we have
    \begin{align*}
        \contractionof{\extnet}{\catvariableof{\sedge}}
        =\contractionof{\{\hypercoreofat{\sedge}{\catvariableof{\sedge}}\}\cup
            \{\mesfromtowith{\secsedge}{\sedge} \wcols (\secsedge,\sedge)\in\dirovedges\}}{\catvariableof{\sedge}} \, .
    \end{align*}
\end{theorem}

We show \theref{the:treeBeliefPropagationExactness} based on the following lemma.
We denote for each pair $(\sedge,\redge)$ the subset $\preedgeset\subset\edges$ as the subset of edges $\edgein$, for which each path in $(\edges,\dirovedges)$ to $\redge$ passes through $\sedge$. The tree hypergraph property makes this definition equivalent to an existing path through $\sedge$, which is used in the proof of the following lemma.
Note that by construction $\sedge\in\preedgeset$.

\begin{lemma}
    \label{lem:messageAsContraction}
    For any tensor network on a tree hypergraph, \algoref{alg:treeBeliefPropagation} terminates in the tree-based implementation and returns final messages
    \begin{align*}
        \messagewith
        = \contractionof{\{\edgehypercorewith\wcols\edge\in\preedgeset\}}{\catvariableof{\sedge\cap\redge}}\,.
    \end{align*}
\end{lemma}
\begin{proof}
    We show this property by induction over the size of the edge sets $\preedgeset$ for pairs $(\sedge,\redge)\in\dirovedges$, such that $\cardof{\preedgeset}\leq n$.
    Note that since always $\sedge\in\preedgeset$ we have that $n\geq1$.

    $n=1$: In this case we have $\preedgeset=\{\sedge\}$ and $\sedge$ is a leaf of the tree-hypergraph $\graph$.
    The claimed message property holds thus by definition.

    $n\rightarrow n+1$: Assume that the message obeys the claimed property for edge sets with cardinality up to $n$.
    If there is no edge set with cardinality $n+1$, the property holds also for those with cardinality up to $n+1$.
    If there is an edge set $\preedgeset$ with size $n+1$, we have
    \begin{align*}
        \preedgeset
        = \{\sedge\} \cup \left(\bigcup_{\secsedge \in\dirovedges} \preedgesetwrt{\secsedge}{\sedge}\right) \, .
    \end{align*}
    The message $\mesfromto{\sedge}{\redge}$ is sent once all messages $\mesfromto{\secsedge}{\sedge}$ to $(\secsedge,\sedge)\in\preedgeset$ arrived.
    By definition we have that
    \begin{align*}
        \mesfromtowith{\sedge}{\redge}
        = \contractionof{\{\hypercoreofat{\sedge}{\catvariableof{\sedge}}\}
            \cup \{\mesfromtowith{\secsedge}{\sedge} \wcols (\secsedge,\sedge) \in\dirovedges \ncond \secsedge \neq \redge \}
        }{\catvariableof{\sedge\cap \redge}}\,.
    \end{align*}
    Now we use the induction assumption on $\preedgesetwrt{\secsedge}{\sedge}$ (since its cardinality is at most $n$) and get
    \begin{align*}
        \mesfromtoat{\sedge}{\redge}{\catvariableof{\sedge\cap \redge}}
        &= \contractionof{
            \{\hypercoreofat{\sedge}{\catvariableof{\sedge}}\} \cup
            \left(\bigcup_{(\secsedge,\sedge)\in\dirovedges \ncond \secsedge\neq \redge}
                \contractionof{\{\hypercoreofat{\thirdsedge}{\catvariableof{\thirdsedge}} \wcols  \thirdsedge \in \preedgesetwrt{\secsedge}{\sedge}\}}{\catvariableof{\secsedge\cap \sedge}} \right)
        }{\catvariableof{\sedge\cap \redge}} \\
        &= \contractionof{
            \{\hypercoreofat{\sedge}{\catvariableof{\sedge}}\} \cup
            \left(\bigcup_{(\secsedge,\sedge)\in\dirovedges \ncond \secsedge\neq \redge} \{\hypercoreofat{\thirdsedge}{\catvariableof{\thirdsedge}} \wcols  \thirdsedge \in \preedgesetwrt{\secsedge}{\sedge}\} \right)
        }{\catvariableof{\sedge\cap \redge}} \\
        &= \contractionof{\{\edgehypercorewith\wcols\edge\in\preedgeset\}}{\catvariableof{\sedge\cap\redge}}\,.
    \end{align*}
    Here, we used the commutation of contraction property in the second equation, which is justified by the assumed tree property of the hypergraph.
    Thus, the message property holds also for any edge sets of size $n+1$.

    By induction, the claimed message property therefore holds for all final messages.
\end{proof}

\begin{proof}[Proof of \theref{the:treeBeliefPropagationExactness}]
    Since the hypergraph is by assumption a tree, we can partition $\edges$ into disjoint subsets $\{\sedge\}$ and $\preedgesetwrt{\secsedge}{\sedge}$ for $(\secsedge,\sedge)\in\dirovedges$.
    We then have
    \begin{align*}
        \contractionof{\extnet}{\catvariableof{\sedge}}
        =& \contractionof{\{\hypercoreofat{\sedge}{\catvariableof{\sedge}}\}
            \cup \left\{\contractionof{\hypercoreofat{\edge}{\catvariableof{\edge}}\wcols\edge\in\preedgesetwrt{\secsedge}{\sedge}}{\catvariableof{\edge\cap\secsedge}} \wcols (\secsedge,\sedge)\in\dirovedges \right\}
        }{\catvariableof{\sedge}} \\
        =& \contractionof{\{\hypercoreofat{\sedge}{\catvariableof{\sedge}}\} \cup \{\mesfromtowith{\secsedge}{\sedge} \wcols (\secsedge,\sedge)\in\dirovedges\}}{\catvariableof{\sedge}},
    \end{align*}
    where we used \lemref{lem:messageAsContraction} in the second equation.
\end{proof}

We illustrate the usage of \algoref{alg:treeBeliefPropagation} on the Markov network of \exaref{exa:studentHC}.

\begin{example}[Continuation of \exaref{exa:studentHC}]
    \label{exa:studentBP}
    We exemplify the Belief Propagation \algoref{alg:treeBeliefPropagation} on the Markov network in the student example (see \exaref{exa:studentHC}).
    The directions of the messages result from the hyperedge overlaps (see \figref{fig:studentMessagePassingDirections} a) and the resulting directions $\dirovedges$ are sketched in \figref{fig:studentMessagePassingDirections} b).
    The messages to $\{(\edgeof{2},\edgeof{0}),(\edgeof{0},\edgeof{2})\}$ are vectors of $\catvariableof{G}$ and the messages $\{(\edgeof{0},\edgeof{1}),(\edgeof{1},\edgeof{0})\}$ are vectors of $\catvariableof{I}$.

    Since the hyperedges are minimally connected, we can implement \algoref{alg:treeBeliefPropagation} by a tree scheduler $\scheduler$:
    \begin{itemize}
        \item The scheduler is initialized with messages from leafs, in our example $\{(\edgeof{2},\edgeof{0}),(\edgeof{1},\edgeof{0})\}$.
        \item Each message is placed exactly once on $\scheduler$, when at a hyperedge all but the reverse message have been received.
        In our example, after execution of $(\edgeof{2},\edgeof{0})$ the message $(\edgeof{0},\edgeof{1})$ is placed on $\scheduler$ and after execution of $(\edgeof{1},\edgeof{0})$ the message $(\edgeof{0},\edgeof{2})$.
    \end{itemize}
    In this implementation, \algoref{alg:treeBeliefPropagation} terminates after $\cardof{\dirovedges}=4$ iterations of the \whileSymbol{} loop.
    The exact marginals of the edge variables are then
    \begin{align*}
        \probat{\catvariableof{L},\catvariableof{G}}
        &= \normalizationof{
            \hypercoreofat{\edgeof{2}}{\catvariableof{L},\catvariableof{G}},
            \mesfromto{\edgeof{0}}{\edgeof{2}}\left[\catvariableof{G}\right]
        }{\catvariableof{L},\catvariableof{G}}, \\
        \probat{\catvariableof{G},\catvariableof{D},\catvariableof{I}}
        &= \normalizationof{\hypercoreofat{\edgeof{0}}{\catvariableof{G},\catvariableof{D},\catvariableof{I}},
            \mesfromto{\edgeof{2}}{\edgeof{0}}\left[\catvariableof{G}\right],
            \mesfromto{\edgeof{1}}{\edgeof{0}}\left[\catvariableof{I}\right]
        }{\catvariableof{G},\catvariableof{D},\catvariableof{I}}, \\
        \probat{\catvariableof{I},\catvariableof{S}}
        &= \normalizationof{
            \hypercoreofat{\edgeof{1}}{\catvariableof{I},\catvariableof{S}},
            \mesfromto{\edgeof{0}}{\edgeof{1}}\left[\catvariableof{I}\right]
        }{\catvariableof{I},\catvariableof{S}} \, .
    \end{align*}

\end{example}

\begin{figure}[t]
    \begin{center}
        \begin{tikzpicture}
            \node[anchor=center] (A) at (-0.5,3) {\corelabelsize $a)$};

            \node[circle, draw, thick, fill=\nodegrayscale, minimum size = \nodeminsize] (A) at (0,2) {};
            \node[anchor=center] (A) at (0,2) {\corelabelsize $\catvariableof{L}$};
            \node[circle, draw, thick, fill=\nodegrayscale, minimum size = \nodeminsize] (A) at (1,2) {};
            \node[anchor=center] (A) at (1,2) {\corelabelsize $\catvariableof{G}$};
            \node[circle, draw, thick, fill=\nodegrayscale, minimum size = \nodeminsize] (A) at (2,2) {};
            \node[anchor=center] (A) at (2,2) {\corelabelsize $\catvariableof{D}$};
            \node[circle, draw, thick, fill=\nodegrayscale, minimum size = \nodeminsize] (A) at (3,2) {};
            \node[anchor=center] (A) at (3,2) {\corelabelsize $\catvariableof{I}$};
            \node[circle, draw, thick, fill=\nodegrayscale, minimum size = \nodeminsize] (A) at (4,2) {};
            \node[anchor=center] (A) at (4,2) {\corelabelsize $\catvariableof{S}$};

            \node (of) at (0,0.5) {};
            \draw[thick, dashed, rounded corners=10pt]  ($(1.5,2)+(of)$) -- ($(1.5,2)-(of)$)  -- ($(-0.5,2)-(of)$) -- ($(-0.5,2)+(of)$) -- cycle;
            \node[anchor=center] (A) at (-1,2) {\corelabelsize $\edgeof{2}$};

            \draw[thick, dashed, rounded corners=10pt]  ($(2.5,2)+(of)$) -- ($(2.5,2)-(of)$)  -- ($(4.5,2)-(of)$) -- ($(4.5,2)+(of)$) -- cycle;
            \node[anchor=center] (A) at (5,2) {\corelabelsize $\edgeof{1}$};
            \node (of) at (0,0.75) {};
            \draw[thick, dashed, rounded corners=10pt]  ($(0.5,2)+(of)$) -- ($(0.5,2)-(of)$)  -- ($(3.5,2)-(of)$) -- ($(3.5,2)+(of)$) -- cycle;
            \node[anchor=center] (A) at (2,1) {\corelabelsize $\edgeof{0}$};

            \begin{scope}[shift={(7,0)}]
                \node[anchor=center] (A) at (-0.5,3) {\corelabelsize $b)$};

                \node[circle, draw, thick, minimum size = \nodeminsize] (A) at (0,2) {};
                \node[anchor=center] (E0) at (0,2) {\corelabelsize $\edgeof{2}$};

                \node[circle, draw, thick, minimum size = \nodeminsize] (A) at (2,2) {};
                \node[anchor=center] (E1) at (2,2) {\corelabelsize $\edgeof{0}$};

                \node[circle, draw, thick, fill=white, minimum size = \nodeminsize] (A) at (4,2) {};
                \node[anchor=center] (E2) at (4,2) {\corelabelsize $\edgeof{1}$};

                \draw[->-] (E0) to[bend right = 40] (E1);
                \node[anchor=center] at (1,2.75) {\corelabelsize $(\edgeof{2},\edgeof{0})$};
                \draw[->-] (E1) to[bend right = 40] (E0);
                \node[anchor=center] at (1,1.25) {\corelabelsize $(\edgeof{2},\edgeof{0})$};
                \draw[->-] (E1) to[bend right = 40] (E2);
                \node[anchor=center] at (3,2.75) {\corelabelsize $(\edgeof{0},\edgeof{1})$};
                \draw[->-] (E2) to[bend right = 40] (E1);
                \node[anchor=center] at (3,1.25) {\corelabelsize $(\edgeof{1},\edgeof{0})$};
            \end{scope}
        \end{tikzpicture}
    \end{center}
    \caption{a) Sketch of the overlap of the edges, resulting in the message directions b) $\dirovedges=\{(\edgeof{2},\edgeof{0}),(\edgeof{0},\edgeof{2}),(\edgeof{0},\edgeof{1}),(\edgeof{1},\edgeof{0})\}$.}\label{fig:studentMessagePassingDirections}
\end{figure}
    \section{The neural paradigm}\label{sec:neurPar}

The neural paradigm of artificial intelligence exploits the decomposition of functions into neurons, which are aligned in a directed acyclic graph.
We show in this section how functions decomposable into neurons can be represented by tensor networks.
To this end, we formalize discrete neural models by decomposition graphs and formally prove the corresponding decomposition of their basis encodings.

\subsection{Function decomposition}

As a main principle of tensor decompositions, we now show that basis encodings of composition functions are contractions of the basis encodings of their components.

\begin{lemma}
    \label{lem:formulaDecomp}
    Let $\formulaat{\shortcatvariables}$ be a composition of a $\seldim$-ary connective function $\chainingfunction$ and functions $\formulaofat{\selindex}{\shortcatvariables}$, where $\selindexin$, that is for $\shortcatindices\in\atomstates$ we have
    \begin{align*}
        \formula(\shortcatindices)
        = \chainingfunction\left(\formulaofat{0}{\shortcatindices}, \dots, \formulaofat{\seldim-1}{\shortcatindices}\right) \, .
    \end{align*}
    Then, we have (see \figref{fig:functionDecomposition})
    \begin{align*}
        \bencodingofat{\formula}{\headvariableof{\formula},\shortcatvariables}
        = \contractionof{
            \{\bencodingofat{\chainingfunction}{\headvariableof{\formula},\headvariables}\}
            \cup \{\bencodingofat{\formulaof{\selindex}}{\headvariableof{\selindex},\shortcatvariables} \wcols \selindexin\}
        }{\headvariableof{\formula},\shortcatvariables} \, .
    \end{align*}
\end{lemma}
\begin{proof}
    For any $\shortcatindicesin$ we have
    \begin{align*}
        &\contractionof{
            \{\bencodingofat{\chainingfunction}{\headvariableof{\formula},\headvariables}\}
            \cup \{\bencodingofat{\formulaof{\selindex}}{\headvariableof{\selindex},\shortcatvariables} \wcols \selindexin\}
        }{\headvariableof{\formula},\indexedshortcatvariables} \\
        &\quad\quad= \contractionof{
            \{\bencodingofat{\chainingfunction}{\headvariableof{\formula},\headvariables}\}
            \cup \{\bencodingofat{\formulaof{\selindex}}{\headvariableof{\selindex},\indexedshortcatvariables} \wcols \selindexin\}
        }{\headvariableof{\formula}} \\
        &\quad\quad= \contractionof{
            \{\bencodingofat{\chainingfunction}{\headvariableof{\formula},\headvariables}\}
            \cup \{\onehotmapofat{\formulaof{\selindex}(\shortcatindices)}{\headvariableof{\selindex}} \wcols \selindexin\}
        }{\headvariableof{\formula}} \\
        &\quad\quad= \onehotmapofat{\formula(\shortcatindices)}{\headvariableof{\formula}}\\
        &\quad\quad= \bencodingofat{\formula}{\headvariableof{\formula},\indexedshortcatvariables} \, .
    \end{align*}
    Thus, the tensors on both sides of the equation coincide in all slices to $\shortcatvariables$ and are thus equal.
\end{proof}

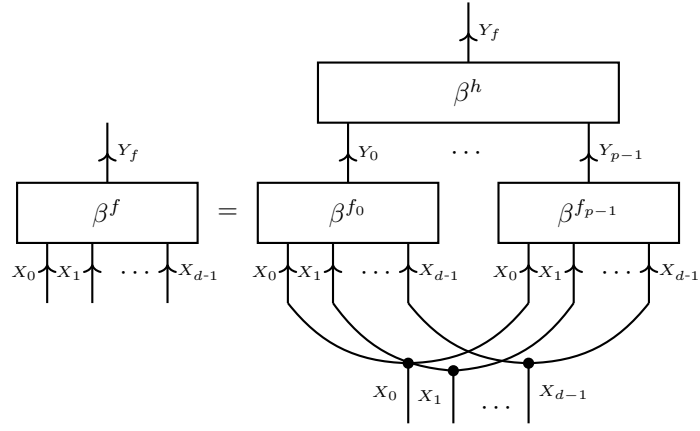
\begin{figure}[t]
    \begin{center}
        \begin{tikzpicture}[scale=0.4, thick]
    \begin{scope}


        \begin{scope}
        [shift={(-8,0)}]
            \draw[->-] (5.5,-9)--(5.5,-7) node[midway,right] {\colorlabelsize $\headvariableof{\exformula}$};
            \drawatomcore{3.5}{-8}{$\bencodingof{\exformula}$}
            \drawatomindices{3.5}{-12}
        \end{scope}

        \node[anchor=center] (text) at (1.5,-10) {${=}$};

        \draw[->-] (9.5,-5)--(9.5,-3) node[midway,right] {\colorlabelsize $\headvariableof{\exformula}$};

        \node[anchor=center] (text) at (9.5,-6) {$\bencodingof{\chainingfunction}$};
        \draw (4.5,-7) rectangle (14.5,-5);

        \draw[->-] (5.5,-9)--(5.5,-7) node[midway,right] {\colorlabelsize $\headvariableof{0}$};

        \node[anchor=center] at (9.5,-8) {$\cdots$};

        \drawatomcore{3.5}{-8}{$\bencodingof{\formulaof{0}}$}
        \drawatomindices{3.5}{-12}

        \begin{scope}
        [shift={(8,0)}]

            \draw[->-] (5.5,-9)--(5.5,-7) node[midway,right] {\colorlabelsize $\headvariableof{\seldim-1}$};

            \drawatomcore{3.5}{-8}{$\bencodingof{\formulaof{\seldim-1}}$}
            \drawatomindices{3.5}{-12}

        \end{scope}

        \draw[fill] (7.5,-15) circle (\dotsize);
        \draw[] (7.5,-15) to[bend left=25] (3.5,-13);
        \draw[] (7.5,-15) to[bend right=25] (11.5,-13);

        \draw[fill] (9,-15.25) circle (\dotsize);
        \draw[] (9,-15.25) to[bend left=25] (5,-13);
        \draw[] (9,-15.25) to[bend right=25] (13,-13);

        \draw[fill] (11.5,-15) circle (\dotsize);
        \draw[] (11.5,-15) to[bend left=25] (7.5,-13);
        \draw[] (11.5,-15) to[bend right=25] (15.5,-13);

        \draw[] (7.5,-15)--(7.5,-17) node[midway,left] {\colorlabelsize $\catvariableof{0}$};
        \draw[] (9,-15.25)--(9,-17) node[midway,left] {\colorlabelsize $\catvariableof{1}$};
        \node[anchor=center] (text) at (10.5,-16.5) {$\cdots$};
        \draw[] (11.5,-15)--(11.5,-17) node[midway,right] {\colorlabelsize $\catvariableof{\atomorder-1}$};

    \end{scope}
\end{tikzpicture}
    \end{center}
    \caption{Tensor network decomposition of the basis encoding of a function $\exformula$, which is the composition of the functions $\formulaof{0},\ldots,\formulaof{\seldim-1}$ with a function $\chainingfunction$.}
    \label{fig:functionDecomposition}
\end{figure}

\noindent We now define a more generic decomposition of discrete functions.

\begin{definition}
    \label{def:decompositionHypergraph}
    A \emph{decomposition hypergraph} is a directed acyclic hypergraph $\graph=(\nodes,\edges)$ such that the following holds.
    \begin{itemize}
        \item Each node $\nodein$ is decorated by a set $\arbsetof{\node}$ of finite cardinality $\catdimof{\node}$, a variable $\catvariableof{\node}$, and an index interpretation function
        \begin{align*}
            \indexinterpretationof{\node} \defcols [\catdimof{\node}] \rightarrow \arbsetof{\node} \, .
        \end{align*}
        \item Each directed hyperedge $(\incomingnodes,\outgoingnodes)$ has at least one outgoing node, that is $\outgoingnodes\neq\varnothing$, and is decorated by an activation function 
        \begin{align*}
            \secexfunctionof{\edge} \defcols
            \bigtimes_{\node\in\incomingnodes} \arbsetof{\node}
            \rightarrow \bigtimes_{\node\in\outgoingnodes} \arbsetof{\node} \, .
        \end{align*}
        \item Each node $\nodein$ appears at most once as an outgoing node. 
        \item The nodes not appearing as an outgoing node are enumerated by $\node^{\insymbol}_{[\atomorder]}$.
        We abbreviate the corresponding variables by $\catvariableof{\node^{\insymbol}_{[\atomorder]}}=\shortcatvariables$. 
        \item The nodes not appearing as an incoming node are enumerated by $\node^{\outsymbol}_{[\seldim]}$.
        We abbreviate the variables by $\catvariableof{\node^{\outsymbol}_{[\selindex]}}=\headvariables$. 
    \end{itemize}
    We assign for each $\catenumeratorin$ restriction functions
    \begin{align*}
        \exfunctionof{\node^{\insymbol}_\catenumerator}
        \defcols \bigtimes_{\seccatenumerator\in[\catorder]} \arbsetof{\node^{\insymbol}_{\seccatenumerator}} \rightarrow \arbsetof{\node^{\insymbol}_\catenumerator}
        \quad,\quad  \restrictionofto{\catindexof{[\catorder]}}{\catenumerator} = \catindexof{\catenumerator}
    \end{align*}
    to the nodes $\node^{\insymbol}_{[\atomorder]}$ and recursively assign to each further node $\node$ a node function 
    \begin{align*}
        \exfunctionof{\node} \wcols \bigtimes_{\catenumeratorin} \arbsetof{\node^{\insymbol}_\catenumerator} \rightarrow \arbsetof{\node}
        \quad,\quad
        \exfunctionof{\node}(\catindexof{[\catorder]})
        = \restrictionofto{\secexfunctionof{\edge_\node}
        \left(
            \bigtimes_{\secnode\in\incomingnodes}\exfunctionof{\secnode}(\catindexof{[\catorder]})
        \right)}{\node} \, ,
    \end{align*}
    where $\edgeof{\node}$ is to each $\node\in\incomingnodes$ the unique hyperedge with outgoing nodes $\{\node\}$.
    We then call the function
    \begin{align*}
        \exfunctionof{\graph} \wcols \bigtimes_{\catenumeratorin} \arbsetof{\node^{\insymbol}_\catenumerator} \rightarrow \bigtimes_{\selindexin} \arbsetof{\node^{\outsymbol}_\selindex}
        \quad,\quad
        \exfunctionof{\graph} = \bigtimes_{\selindexin} \exfunctionof{\node^{\outsymbol}_\selindex}
    \end{align*}
    the composition function to the decomposition hypergraph $\graph$.
\end{definition}

The neural paradigm in AI can be modeled by the existence of decomposition hypergraphs for functions on large sets.
We now show how decomposition hypergraphs enable the sparse representation of composition functions by tensor networks.

\begin{theorem}
    \label{the:functionDecompositionRep}
    For any decomposition hypergraph $\graph$ with composition formula $\exfunctionof{\graph}$, we have
    \begin{align*}
        \bencodingofat{\exfunctionof{\graph}}{\headvariables,\shortcatvariables}
        = \contractionof{\left\{\bencodingofat{\secexfunctionof{\edge}}{\catvariableof{\outgoingnodes},\catvariableof{\incomingnodes}} \wcols \edge=(\incomingnodes,\outgoingnodes)\in\edges\right\}
        }{\headvariables,\shortcatvariables} \, .
    \end{align*}
\end{theorem}
\begin{proof}
    We show by induction over the number of edges in $\graph$ that, for any $\shortcatindicesin$ and $\node\in\nodes$, we have
    \begin{align}
        \label{eq:indHypNeurThm}
        \contractionof{\left\{\bencodingofat{\secexfunctionof{\edge}}{\catvariableof{\outgoingnodes},\catvariableof{\incomingnodes}} \wcols \edge=(\incomingnodes,\outgoingnodes)\in\edges\right\}
        }{\catvariableof{\nodes\setexcept{\node^{\insymbol}_{[\atomorder]}}},\indexedshortcatvariables}
        = \bigotimes_{\node\in\nodes\setexcept{\node^{\insymbol}_{[\atomorder]}}} \onehotmapofat{\exfunctionof{\node}(\shortcatindices)}{\catvariableof{\node}} \, .
    \end{align}

    $\cardof{\edges}=1$: If there is a single edge $\edge=(\incomingnodes,\outgoingnodes)$ in $\edges$, we have $\shortcatvariables=\catvariableof{\incomingnodes}$ and $\nodes\setexcept{\node^{\insymbol}_{[\atomorder]}}=\outgoingnodes$.
    In this case \eqref{eq:indHypNeurThm} holds since
    \begin{align*}
        \bencodingofat{\secexfunctionof{\edge}}{\catvariableof{\outgoingnodes},\indexedcatvariableof{\incomingnodes}}
        = \bigotimes_{\node\in\nodes\setexcept{\node^{\insymbol}_{[\atomorder]}}} \onehotmapofat{\exfunctionof{\node}(\shortcatindices)}{\catvariableof{\node}} \, .
    \end{align*}

    $(\cardof{\edges}=n)\Rightarrow(\cardof{\edges}=n+1)$:
    Let us now assume that \eqref{eq:indHypNeurThm} holds for all graphs with $\cardof{\edges}\leq n$ and let $\graph$ be a graph with $\cardof{\edges}=n+1$.
    We find an edge $\secedge=(\incomingnodes,\outgoingnodes)$ such that all nodes in $\incomingnodes$ are only appearing as outgoing nodes in other edges.
    \begin{align*}
        &\contractionof{\left\{\bencodingofat{\secexfunctionof{\edge}}{\catvariableof{\edge}} \wcols \edge\in\edges\right\}
        }{\catvariableof{\nodes\setexcept{\node^{\insymbol}_{[\atomorder]}}},\indexedshortcatvariables} \\
        &\quad=
        \contractionof{
            \bencodingofat{\secedge}{\catvariableof{\outgoingnodes},\catvariableof{\incomingnodes}},
            \contractionof{\left\{\bencodingofat{\secexfunctionof{\edge}}{\catvariableof{\edge}} \wcols \edge\in\edges\setexcept{\{\secedge\}}\right\}
            }{\catvariableof{\nodes\setexcept{\{\node^{\insymbol}_{[\atomorder]}\cup\outgoingnodes\}}},\indexedshortcatvariables}
        }{\catvariableof{\nodes\setexcept{\node^{\insymbol}_{[\atomorder]}}},\indexedshortcatvariables} \\
        &\quad=
        \contractionof{
            \bencodingofat{\secedge}{\catvariableof{\outgoingnodes},\catvariableof{\incomingnodes}},
            \bigotimes_{\node\in\nodes\setexcept{\{\node^{\insymbol}_{[\atomorder]}\cup\outgoingnodes\}}}
            \onehotmapofat{\exfunctionof{\node}(\shortcatindices)}{\catvariableof{\node}}
        }{\catvariableof{\nodes\setexcept{\node^{\insymbol}_{[\atomorder]}}},\indexedshortcatvariables} \\
        &\quad=
        \bigotimes_{\node\in\nodes\setexcept{\node^{\insymbol}_{[\atomorder]}}}
        \onehotmapofat{\exfunctionof{\node}(\shortcatindices)}{\catvariableof{\node}} \, .
    \end{align*}
    Here, in the second equation we used the induction hypothesis on the subgraph $(\nodes,\edges\setexcept{\{\secedge\}})$ with $\cardof{\edges\setexcept{\{\secedge\}}}=n$ and a generic contraction property of basis encodings in the third equation.

    Thus, \eqref{eq:indHypNeurThm} holds always and we have for any $\shortcatindicesin$ that
    \begin{align*}
        \contractionof{\left\{\bencodingofat{\secexfunctionof{\edge}}{\catvariableof{\outgoingnodes},\catvariableof{\incomingnodes}} \wcols \edge=(\incomingnodes,\outgoingnodes)\in\edges\right\}
        }{\headvariables,\indexedshortcatvariables}
        &= \contractionof{
            \bigotimes_{\node\in\nodes\setexcept{\node^{\insymbol}_{[\atomorder]}}} \onehotmapofat{\exfunctionof{\node}(\shortcatindices)}{\catvariableof{\node}}
        }{\headvariables,\shortcatvariables} \\
        & = \bigotimes_{\node\in\node^{\outsymbol}_{[\seldim]}}
        \onehotmapofat{\exfunctionof{\node}(\shortcatindices)}{\catvariableof{\node}} \\
        & = \onehotmapofat{\exfunctionof{\graph}}{\headvariables} \\
        & = \bencodingofat{\exfunctionof{\graph}}{\headvariables,\indexedshortcatvariables} \, .
    \end{align*}
    Keeping $\shortcatvariables$ open, the claim is established.
\end{proof}

When neurons have tunable parameters, we can discretize those by sets $\arbsetof{\catenumerator}$ and understand them as additional input variables.

\begin{example}[Sum of integers in $\catdim$-adic representation]
    \label{exa:madicRepresentation}
    We develop a tensor network representation of integer summations on the set $[\catdim^{\catorder}]=\{0,\ldots,\catdim^{\catorder}-1\}$, where $\catdim,\catorder\in\nn$,
    \begin{align*}
        + \defcols [\catdim^\catorder] \times [\catdim^\catorder] \rightarrow [\catdim^{\catorder+1}]
        \quad,\quad
        +(i,\tilde{i}) = i+\tilde{i} \,,
    \end{align*}
    which have a $\catdim$-adic representation of length $\catorder$.
    We define an index interpretation map
    \begin{align*}
        \indexinterpretation \defcols \bigtimes_{\catenumeratorin}[\catdim] \rightarrow [\catdim^{\catorder}]
        \quad,\quad
        \indexinterpretationat{\shortcatindices}
        = \sum_{\catenumeratorin} \catindexof{\catenumerator} \cdot \catdim^{\catenumerator} \, ,
    \end{align*}
    which enables the parameterization of $[\catdim^{\catorder}]$ as the states of $\catorder$ categorical variables $\shortcatvariables$ of dimension $\catdim$.
    We analogously represent a second set $[\catdim^{\catorder}]$ by variables $\tildecatvariableof{[\catorder]}$ and the set $[\catdim^{\catorder+1}]$ of possible sums by $\headvariableof{[\catorder+1]}$.
    The basis encoding of the sum is then
    \begin{align*}
        \bencodingofat{+}{\headvariableof{[\catorder+1]},\shortcatvariables,\seccatvariableof{[\catorder]}}
        = \sum_{\shortcatindices,\tildecatindexof{[\catorder]}} \onehotmapofat{\invindexinterpretationat{
            \indexinterpretationat{\shortcatindices} + \indexinterpretationat{\tildecatindexof{[\catorder]}}
        }}{\headvariableof{[\catorder+1]}}
        \otimes \onehotmapofat{\shortcatindices}{\shortcatvariables}
        \otimes \onehotmapofat{\tildecatindexof{[\catorder]}}{\tildecatvariableof{[\catorder]}} \, .
    \end{align*}
    Note that the tensor space of $\bencodingof{+}$ is of dimension $\catdim^{3\cdot\catorder+1}$ increasing exponentially in $\catorder$.
    Feasible representation of this tensor for large $\catorder$ therefore requires tensor network decompositions, which we now provide based on a decomposition hypergraph.
    The targeted function to be decomposed is the representation of the integer sum by
    \begin{align*}
        \marysumsymbol \defcols \left(\bigtimes_{\catenumeratorin}[\catdim]\right) \times \left(\bigtimes_{\catenumeratorin}[\catdim]\right) \rightarrow \bigtimes_{\catenumerator\in[\catorder+1]}[\catdim]
        \quad,\quad
        \marysumsymbol(\shortcatindices,\tildecatindexof{[\catorder]}) =
        \invindexinterpretationat{\indexinterpretationat{\shortcatindices}+\indexinterpretationat{\tildecatindexof{[\catorder]}}} \, .
    \end{align*}
    We build a decomposition hypergraph $\graph=(\nodes,\edges)$ (see \defref{def:decompositionHypergraph}) consisting of $4\cdot \catorder$ nodes (see \figref{fig:decompositionMarySum}a) .
    The nodes carry the $(3\cdot\catorder +1)$ variables $\shortcatvariables,\seccatvariableof{[\catorder]},\headvariableof{[\catorder+1]}$ of dimension $\catdim$ constructed above and $\catorder-1$ auxiliary variables $\thirdcatvariableof{[\catorder-1]}$ of dimension $2$ representing carry bits.
    The directed hyperedges of $\graph$ are
    \begin{align*}
        \edges
        =&\left\{(\{\catvariableof{0},\tildecatvariableof{0}\},\{\headvariableof{0},\thirdcatvariableof{0}\})\right\}
        \cup \left\{(\{\thirdcatvariableof{\catenumerator-1},\catvariableof{\catenumerator},\tildecatvariableof{\catenumerator}\},\{\headvariableof{\catenumerator},\thirdcatvariableof{\catenumerator}\})
                 \wcols \catenumerator\in\{1,\ldots,\catorder-2\}\right\} \\
        &\cup \left\{(\{\thirdcatvariableof{\catorder-2},\catvariableof{\catorder-1},\tildecatvariableof{\catorder-1}\},\{\headvariableof{\catorder-1},\headvariableof{\catorder}\})\right\}
    \end{align*}
    and are decorated by local summation functions
    \begin{align*}
        \modsumsymbol \defcols [2] \times [\catdim] \times [\catdim] \rightarrow [\catdim] \times [2]
        \quad,\quad
        \modsumsymbol(\thirdcatindex,\catindex,\tildecatindex)
        = \left((\thirdcatindex + \catindex + \tildecatindex) \modspace\catdim,
              \left\lfloor\frac{\thirdcatindex + \catindex + \tildecatindex}{\catdim}\right\rfloor\right) \, .
    \end{align*}
    Since to the first hyperedge we do not have a carry bit, the decorating function is the restriction of the first argument to $0$.

    It is known that the composition of the local summations $\modsumsymbol$ is the global summation $\marysumsymbol$ of integers in $\catdim$-adic representation.
    Thus, the composition function $\exfunctionof{\graph}$ is $\marysumsymbol$.
    By \theref{the:functionDecompositionRep} we have a decomposition of the basis encoding to $\exfunctionof{\graph}$ (see \figref{fig:decompositionMarySum}b) as
    \begin{align*}
        \bencodingofat{\marysumsymbol}{\headvariableof{[\catorder+1]},\shortcatvariables,\tildecatvariableof{[\catorder]}}
        = \breakablecontractionof{
            &\{\bencodingofat{\modsumsymbol,0}{\headvariableof{0},\thirdcatvariableof{0},\catvariableof{0},\tildecatvariableof{0}}\} \cup \\
            &\{\bencodingofat{\modsumsymbol,\catenumerator}{
            \headvariableof{\catenumerator},\thirdcatvariableof{\catenumerator},\catvariableof{\catenumerator},\tildecatvariableof{\catenumerator},\thirdcatvariableof{\catenumerator-1}
            }\wcols \catenumerator\in\{1,\ldots,\catorder-2\}\} \cup \\
            &\{\bencodingofat{\modsumsymbol,\catorder-2}{\headvariableof{\catorder-1},\headvariableof{\catorder},\catvariableof{\catorder-1},\tildecatvariableof{\catorder-1},\thirdcatvariableof{\catorder-2}}\}
        }{\headvariableof{[\catorder+1]},\shortcatvariables,\tildecatvariableof{[\catorder]}} \, .
    \end{align*}

    \begin{figure}[t]
        \begin{center}
            \begin{tikzpicture}[scale=0.35, thick]

                \begin{scope}[shift={(-0.5,10)}]

                    \node[anchor=east]  at (-9,4) {$a)$};

                    \node (of) at (0,1.35) {};

                    \draw[thick, dashed, rounded corners=10pt]  ($(-5,3)+(of)$) -- ($(-5,3)-(of)$)  -- ($(28.5,3)-(of)$) -- ($(28.5,3)+(of)$) -- cycle;
                    \node[anchor=center] (A) at (-3,3) {\corelabelsize $\node^{\outsymbol}$};

                    \draw[thick, dashed, rounded corners=10pt]  ($(-5,-3)+(of)$) -- ($(-5,-3)-(of)$)  -- ($(28.5,-3)-(of)$) -- ($(28.5,-3)+(of)$) -- cycle;
                    \node[anchor=center] (A) at (-3,-3) {\corelabelsize $\node^{\insymbol}$};

                    \node[circle, draw, thick, fill=\nodegrayscale, minimum size = \nodeminsize] (Y0) at (1.5,3) {};
                    \node[] (text) at (1.5,3) {\corelabelsize $\headvariableof{0}$};
                    \node[circle, draw, thick, fill=\nodegrayscale, minimum size = \nodeminsize] (X00) at (0,-3) {};
                    \node[] (text) at (0,-3) {\corelabelsize $\catvariableof{0}$};
                    \node[circle, draw, thick, fill=\nodegrayscale, minimum size = \nodeminsize] (X10) at (3,-3) {};
                    \node[] (text) at (3,-3) {\corelabelsize $\tildecatvariableof{0}$};
                    \node[circle, draw, thick, fill=\nodegrayscale, minimum size = \nodeminsize] (Z0) at (5,0) {};
                    \node[] (text) at (5,0) {\corelabelsize $\thirdcatvariableof{0}$};

                    \coordinate (m0) at (1.5,0);
                    \node[anchor=east]  at (1.5,0) {$\edgeof{0}$};
                    \draw[->-] (X00) -- (m0);
                    \draw[->-] (X10) -- (m0);
                    \draw[->-] (m0) -- (Y0);
                    \draw[->-] (m0) -- (Z0);

                    \node[circle, draw, thick, fill=\nodegrayscale, minimum size = \nodeminsize] (Y1) at (8.5,3) {};
                    \node[] (text) at (8.5,3) {\corelabelsize $\headvariableof{1}$};
                    \node[circle, draw, thick, fill=\nodegrayscale, minimum size = \nodeminsize] (X01) at (7,-3) {};
                    \node[] (text) at (7,-3) {\corelabelsize $\catvariableof{1}$};
                    \node[circle, draw, thick, fill=\nodegrayscale, minimum size = \nodeminsize] (X11) at (10,-3) {};
                    \node[] (text) at (10,-3) {\corelabelsize $\tildecatvariableof{1}$};
                    \node[circle, draw, thick, fill=\nodegrayscale, minimum size = \nodeminsize] (Z1) at (12,0) {};
                    \node[] (text) at (12,0) {\corelabelsize $\thirdcatvariableof{1}$};

                    \coordinate (m1) at (8.5,0);
                    \node[anchor=east]  at (8.5,0.75) {$\edgeof{1}$};
                    \draw[->-] (Z0) -- (m1);
                    \draw[->-] (X01) -- (m1);
                    \draw[->-] (X11) -- (m1);
                    \draw[->-] (m1) -- (Y1);
                    \draw[->-] (m1) -- (Z1);

                    \draw[->-] (Z1) -- (15,0) node[anchor=west]{$\cdots$};

                    \begin{scope}[shift={(15,0)}]
                        \node[circle, draw, thick, fill=\nodegrayscale, minimum size = \nodeminsize] (Z0) at (5,0) {};
                        \node[] (text) at (5,0) {\corelabelsize $\thirdcatvariableof{\catorder\shortminus2}$};

                        \node[circle, draw, thick, fill=\nodegrayscale, minimum size = \nodeminsize] (Y1) at (8.5,3) {};
                        \node[] (text) at (8.5,3) {\corelabelsize $\headvariableof{\catorder\shortminus 1}$};
                        \node[circle, draw, thick, fill=\nodegrayscale, minimum size = \nodeminsize] (X01) at (7,-3) {};
                        \node[] (text) at (7,-3) {\corelabelsize $\catvariableof{\catorder\shortminus1}$};
                        \node[circle, draw, thick, fill=\nodegrayscale, minimum size = \nodeminsize] (X11) at (10,-3) {};
                        \node[] (text) at (10,-3) {\corelabelsize $\tildecatvariableof{\catorder\shortminus1}$};
                        \node[circle, draw, thick, fill=\nodegrayscale, minimum size = \nodeminsize] (Z1) at (12,3) {};
                        \node[] (text) at (12,3) {\corelabelsize $\headvariableof{\catorder}$};

                        \coordinate (me) at (8.5,0);
                        \node[anchor=east]  at (8.5,1) {$\edgeof{\catorder\shortminus 1}$};
                        \draw[->-] (2,0)  -- (Z0);
                        \draw[->-] (Z0) -- (me);
                        \draw[->-] (X01) -- (me);
                        \draw[->-] (X11) -- (me);
                        \draw[->-] (me) -- (Y1);
                        \draw[->-] (me) -- (Z1);
                    \end{scope}

                \end{scope}

                \node[anchor=east]  at (-9,4) {$b)$};

                \begin{scope}[shift={(-8,0)}]
                    \draw (-2,-1) rectangle (4,1);
                    \draw[->-] (-1.25,1)--(-1.25,2.5) node[midway,left] {\colorlabelsize $\headvariableof{0}$};
                    \draw[->-] (3.25,1)--(3.25,2.5) node[midway,right] {\colorlabelsize $\headvariableof{\catorder\shortminus1}$};
                    \node[anchor=center] (A) at (1,2.5) {\corelabelsize $\cdots$};
                    \node[anchor=center] (A) at (1,0) {\corelabelsize $\bencodingof{\exfunctionof{\graph}}$};
                    \draw[-<-] (-1.5,-1)--(-1.5,-2.5) node[midway,left] {\colorlabelsize $\catvariableof{0}$};
                    \draw[-<-] (-1,-1)--(-1,-2.5) node[midway,right] {\colorlabelsize $\tildecatvariableof{0}$};
                    \draw[-<-] (3,-1)--(3,-2.5) node[midway,left] {\colorlabelsize $\catvariableof{\catorder\shortminus1}$};
                    \draw[-<-] (3.5,-1)--(3.5,-2.5) node[midway,right] {\colorlabelsize $\tildecatvariableof{\catorder\shortminus1}$};
                    \node[anchor=center] (A) at (1,-2.5) {\corelabelsize $\cdots$};
                \end{scope}

                \node[anchor=east]  at (-1.5,0) {$=$};

                \draw (-1,-1) rectangle (3,1);
                \node[anchor=center] (A) at (1,0) {\corelabelsize $\bencodingof{\modsumsymbol,0}$};
                \draw[->-] (1,1)--(1,3) node[midway,left] {\colorlabelsize $\headvariableof{0}$};
                \draw[-<-] (0,-1)--(0,-2.5) node[midway,left] {\colorlabelsize $\catvariableof{0}$};
                \draw[-<-] (2,-1)--(2,-2.5) node[midway,right] {\colorlabelsize $\tildecatvariableof{0}$};
                \draw[->-] (3,0)--(6,0) node[midway,above] {\colorlabelsize $\thirdcatvariableof{0}$};

                \begin{scope}[shift={(7,0)}]
                    \draw (-1,-1) rectangle (3,1);
                    \node[anchor=center] (A) at (1,0) {\corelabelsize $\bencodingof{\modsumsymbol,1}$};
                    \draw[->-] (1,1)--(1,3) node[midway,left] {\colorlabelsize $\headvariableof{1}$};
                    \draw[-<-] (0,-1)--(0,-2.5) node[midway,left] {\colorlabelsize $\catvariableof{1}$};
                    \draw[-<-] (2,-1)--(2,-2.5) node[midway,right] {\colorlabelsize $\tildecatvariableof{1}$};
                    \draw[->-] (3,0)--(6,0) node[midway,above] {\colorlabelsize $\thirdcatvariableof{1}$};
                \end{scope}

                \node[anchor=center] at (15.5,0) {$\cdots$};

                \begin{scope}[shift={(22,0)}]
                    \draw[->-] (-4,0)--(-1,0) node[midway,above] {\colorlabelsize $\thirdcatvariableof{\catorder\shortminus 2}$};
                    \draw (-1,-1) rectangle (3,1);
                    \node[anchor=center] (A) at (1,0) {\corelabelsize $\bencodingof{\modsumsymbol,\catorder-1}$};
                    \draw[->-] (1,1)--(1,3) node[midway,left] {\colorlabelsize $\headvariableof{\catorder\shortminus1}$};
                    \draw[-<-] (0,-1)--(0,-2.5) node[midway,left] {\colorlabelsize $\catvariableof{\catorder\shortminus1}$};
                    \draw[-<-] (2,-1)--(2,-2.5) node[midway,right] {\colorlabelsize $\tildecatvariableof{\catorder\shortminus1}$};
                    \draw (3,0)--(4.5,0) -- (4.5,1);
                    \draw[->-] (4.5,1)--(4.5,3) node[midway,left] {\colorlabelsize $\headvariableof{\catorder}$};
                \end{scope}

            \end{tikzpicture}
        \end{center}
        \caption{Example of a decomposition hypergraph to the sum of integers (see \exaref{exa:madicRepresentation}).
        a) Hypergraph of directed edges $\edgeof{\catenumerator}$ for $\catenumeratorin$, each decorated by an integer summation $+$ preparing an index $\headvariableof{\catenumerator}$ of the resulting sum.
        b) Corresponding tensor network decomposition of the basis encoded composition function, which is the sum of integers in $\catdim$-adic representation.}
        \label{fig:decompositionMarySum}
    \end{figure}

\end{example}

\subsection{Function evaluation by message passing}

We are now concerned with an efficient inference algorithm based on tensor network contractions. To evaluate a function given as a tensor network decomposition of its basis encoding, the whole network has to be contracted. As this can be infeasible for large networks, a message passing algorithm based on local contractions can be applied, compare \algoref{alg:treeBeliefPropagation} for a message passing algorithm for tensor networks on a on tree hypergraph.

\begin{algorithm}[hbt!]
    \caption{Directed Belief Propagation}\label{alg:directedBeliefPropagation}
    \begin{algorithmic}
        \Require Tensor network $\extnet$ on a directed hypergraph $\graph$
        \Ensure Messages $\{\messagewith\wcols(\sedge,\redge)\in\dirovedges\}$
        \iosepline
        \State Prepare directed message directions
        \begin{align*}
            \dirovedges = \left\{
                              \big((\incomingnodes_{0},\outgoingnodes_{0}),(\incomingnodes_{1},\outgoingnodes_{1})\big) \wcols
                              \incomingnodes_{0} \cap (\incomingnodes_{1},\outgoingnodes_{1}) = \varnothing
                              \ncond
                              \outgoingnodes_{1} \cap (\incomingnodes_{0},\outgoingnodes_{0}) = \varnothing
                              \ncond
                              \outgoingnodes_{0} \cap \incomingnodes_{1} \neq \varnothing
            \right\}
        \end{align*}
        \State Initialize a message queue $\scheduler=\{(\secsedge,\sedge) \wcols \secsedge \text{ has empty incoming nodes} \}$ 
        \While{$\scheduler$ not empty}
            \State Pop a $(\sedge,\redge)$ pair from $\scheduler$
            \State Update the message
            \begin{align*}
                \messagewith
                = \contractionof{\{\hypercoreofat{\sedge}{\catvariableof{\sedge}}\}
                    \cup \{\mesfromtoat{\secsedge}{\sedge}{\catvariableof{\secsedge\cap\sedge}} \wcols (\secsedge,\sedge)\in\dirovedges \ncond \secsedge\neq \redge\}
                }{\catvariableof{\sedge\cap \redge}}
            \end{align*}
            \State Update $\scheduler$ by all messages $(\redge,\thirdsedge)$ which have not yet been sent, if all messages $(\secsedge,\redge)$ have been sent.
        \EndWhile
        \State \Return Messages $\{\messagewith\wcols(\secsedge,\sedge)\in\dirovedges\}$
    \end{algorithmic}
\end{algorithm}

We apply the Directed Belief Propagation, Algorithm \algoref{alg:directedBeliefPropagation}, on a decomposition hypergraph, where we add hyperedges to each leaf node and assign one-hot encodings of input states.
We then show that the messages are the one-hot encodings to the evaluations of the node functions.

\begin{theorem}
    \label{the:directedBeliefPropagationExactness}
    Let $\graph$ be a decomposition graph and let us add hyperedges containing single input nodes, which are decorated by one-hot encodings.
    Then the messages computed in \algoref{alg:directedBeliefPropagation} are characterized by
    \begin{align*}
        \mesfromtowith{\sedge}{\redge}
        = \bigotimes_{\node\in\sedge\cap\redge} \onehotmapofat{\exfunctionof{\node}(\shortcatindices)}{\catvariableof{\node}} \, .
    \end{align*}
\end{theorem}
\begin{proof}
    We show the theorem inductively over the messages computed in \algoref{alg:directedBeliefPropagation}.
    The first message is sent from an input edge $\{[\catenumerator]\}$ to an edge $\edge$ of the decomposition graph and is by assumption the one-hot encoding of an input state $\onehotmapofat{\catindexof{\catenumerator}}{\catvariableof{\catenumerator}}$.

    We now assume that all previous messages satisfy the claimed equation at an arbitrary stage of the algorithm.
    The message computed in the \whileSymbol{} loop is then a contraction of one-hot encodings with basis encodings and
    \begin{align*}
        \mesfromtowith{\sedge}{\redge}
        &=\contractionof{\{\bencodingofat{\secexfunctionof{\sedge}}{\catvariableof{\outgoingnodes},\catvariableof{\incomingnodes}}\} \cup
            \{\mesfromtowith{\secsedge}{\sedge} \wcols (\secsedge,\sedge)\in\dirovedges\}
        }{\catvariableof{\sedge\cap\redge}} \\
        &=\contractionof{\{\bencodingofat{\secexfunctionof{\sedge}}{\catvariableof{\outgoingnodes},\catvariableof{\incomingnodes}}\} \cup
            \{\onehotmapofat{\exfunctionof{\node}(\shortcatindices)}{\catvariableof{\node}} \wcols \node\in\incomingnodes\}
        }{\catvariableof{\sedge\cap\redge}} \\
        &= \bigotimes_{\node\in\sedge\cap\redge} \onehotmapofat{\exfunctionof{\node}(\shortcatindices)}{\catvariableof{\node}} \, .
    \end{align*}
    Thus, the new message is also the tensor product of the one-hot encodings of the evaluated node functions.
    By induction, the property is therefore true for all messages.
\end{proof}

We notice that we can interpret any directed acyclic hypergraph for which each node appears exactly once as an outgoing node and which is decorated by boolean and directed tensors $\extnet$.
Edges with empty incoming sets are carrying one-hot encodings of input states and all further edges carry functions.

\begin{example}[Continuation of \exaref{exa:madicRepresentation}]
    \label{exa:madicPropagation}
    We now show how \algoref{alg:directedBeliefPropagation} can be exploited to compute an efficient message passing algorithm for the digits of the $\catdim$-adic sum.
    Given two numbers in $\catdim$-adic representation by the tuples $\shortcatindices$ and $\tildecatindexof{[\catorder]}$, we add the hyperedges with empty incoming nodes and single outgoing node
    \begin{align*}
        \bigcup_{\catenumeratorin}\left\{(\varnothing,\{\catvariableof{\catenumerator}\}),(\varnothing,\{\tildecatvariableof{\catenumerator}\})\right\}
    \end{align*}
    to the hypergraph of \exaref{exa:madicRepresentation} and decorate them by the digit one-hot encodings $\onehotmapofat{\catindexof{\catenumerator}}{\catvariableof{\catenumerator}}$ and $\onehotmapofat{\tildecatindexof{\catenumerator}}{\tildecatvariableof{\catenumerator}}$ (see \figref{fig:propagationMary}).
    We then apply the Directed Belief Propagation \algoref{alg:directedBeliefPropagation}.
    The initial messages queue then consists of the messages from the digit encoding.
    As sketched in \figref{fig:propagationMary}, to each digit there are three messages (with the exception of the first being two), which can be scheduled in consecutive epochs $\messagesymbol^{(\catenumerator,[3])}$.
    We then have by \theref{the:directedBeliefPropagationExactness} for $\catenumerator\in[\catorder-1]$ that
    \begin{align*}
        \contractionof{
            \bencodingofat{\modsumsymbol,\catenumerator}{\headvariableof{\catenumerator},\thirdcatvariableof{\catenumerator},\catvariableof{\catenumerator},\tildecatvariableof{\catenumerator},\thirdcatvariableof{\catenumerator-1}},
            \messagesymbol^{(\catenumerator-1,2)}[\thirdcatvariableof{\catenumerator-1}],
            \messagesymbol^{(\catenumerator,0)}[\catvariableof{\catenumerator}],
            \messagesymbol^{(\catenumerator,1)}[\tildecatvariableof{\catenumerator}]
        }{\thirdcatvariableof{\catenumerator}}
        =
        \onehotmapofat{\thirdcatindexof{\catenumerator}}{\thirdcatvariableof{\catenumerator}}\,,
    \end{align*}
    where $\thirdcatindexof{\catenumerator}$ is the value of the $\catenumerator$-th carry bit.
    The $\catenumerator$-th digit of the sum $\headindexof{\catenumerator}$ can further be obtained by the contraction
    \begin{align*}
        \contractionof{
            \bencodingofat{\modsumsymbol,\catenumerator}{\headvariableof{\catenumerator},\thirdcatvariableof{\catenumerator},\catvariableof{\catenumerator},\tildecatvariableof{\catenumerator},\thirdcatvariableof{\catenumerator-1}},
            \messagesymbol^{(\catenumerator-1,2)}[\thirdcatvariableof{\catenumerator-1}],
            \messagesymbol^{(\catenumerator,0)}[\catvariableof{\catenumerator}],
            \messagesymbol^{(\catenumerator,1)}[\tildecatvariableof{\catenumerator}]
        }{\headvariableof{\catenumerator}}
        =
        \onehotmapofat{\headindexof{\catenumerator}}{\headvariableof{\catenumerator}} \, .
    \end{align*}
    Note that the hypergraph representing this instance is a tree and by \theref{the:treeBeliefPropagationExactness} also the message passing scheme of \algoref{alg:treeBeliefPropagation} is guaranteed to produce the exact contractions.
    We can exploit this fact for example in the efficient computation of averages of the summation digits, when we have an elementary distribution of input digits.
    We emphasize that the directed belief propagation \algoref{alg:treeBeliefPropagation} is exact even if the hypergraph fails to be a tree, provided that we have directed and boolean tensors..

    \begin{figure}[t]
        \begin{center}
            \begin{tikzpicture}[scale=0.5, thick]

                \draw (-1,-1) rectangle (3,1);
                \node[anchor=center] (A) at (1,0) {\corelabelsize $\bencodingof{\modsumsymbol,0}$};
                \draw[->-] (1,1)--(1,3) node[midway,left] {\colorlabelsize $\headvariableof{0}$};
                \draw[-<-] (0,-1)--(0,-2.5) node[midway,left] {\colorlabelsize $\catvariableof{0}$};
                \draw (-0.75,-2.5) rectangle (0.75,-4);
                \node[anchor=center] (A) at (0,-3.25) {\corelabelsize $\onehotmapof{\catindexof{0}}$};
                \draw[\newmessagecolor,dashed, ->] (-1.25,-3.25) to [bend right = -30] (-1.25,-1);
                \node[\newmessagecolor,anchor=center] (A) at (-1.75,-3.6) {\colorlabelsize $\messagesymbol^{(0,0)}$};
                \draw[-<-] (2,-1)--(2,-2.5) node[midway,right] {\colorlabelsize $\tildecatvariableof{0}$};
                \draw (1.25,-2.5) rectangle (2.75,-4);
                \node[anchor=center] (A) at (2,-3.25) {\corelabelsize $\onehotmapof{\tildecatindexof{0}}$};
                \draw[\newmessagecolor,dashed, ->] (3.25,-3.25) to [bend right = 30] (3.25,-1);
                \node[\newmessagecolor,anchor=center] (A) at (3.8,-3.6) {\colorlabelsize $\messagesymbol^{(0,1)}$};
                \draw[->-] (3,0)--(6,0) node[midway,above] {\colorlabelsize $\thirdcatvariableof{0}$};
                \draw[\newmessagecolor,dashed, ->] (3,1.25) to [bend right = -30] (6,1.25);
                \node[\newmessagecolor,anchor=center] (A) at (4.5,2.25) {\colorlabelsize $\messagesymbol^{(0,2)}$};

                \begin{scope}[shift={(7,0)}]
                    \draw (-1,-1) rectangle (3,1);
                    \node[anchor=center] (A) at (1,0) {\corelabelsize $\bencodingof{\modsumsymbol,1}$};
                    \draw[->-] (1,1)--(1,3) node[midway,left] {\colorlabelsize $\headvariableof{1}$};
                    \draw[-<-] (0,-1)--(0,-2.5) node[midway,left] {\colorlabelsize $\catvariableof{1}$};
                    \draw (-0.75,-2.5) rectangle (0.75,-4);
                    \node[anchor=center] (A) at (0,-3.25) {\corelabelsize $\onehotmapof{\catindexof{1}}$};
                    \draw[\newmessagecolor,dashed, ->] (-1.25,-3.25) to [bend right = -30] (-1.25,-1);
                    \node[\newmessagecolor,anchor=center] (A) at (-1.75,-3.6) {\colorlabelsize $\messagesymbol^{(1,0)}$};
                    \draw[-<-] (2,-1)--(2,-2.5) node[midway,right] {\colorlabelsize $\tildecatvariableof{1}$};
                    \draw (1.25,-2.5) rectangle (2.75,-4);
                    \node[anchor=center] (A) at (2,-3.25) {\corelabelsize $\onehotmapof{\tildecatindexof{1}}$};
                    \draw[\newmessagecolor,dashed, ->] (3.25,-3.25) to [bend right = 30] (3.25,-1);
                    \node[\newmessagecolor,anchor=center] (A) at (3.8,-3.6) {\colorlabelsize $\messagesymbol^{(1,1)}$};
                    \draw[->-] (3,0)--(6,0) node[midway,above] {\colorlabelsize $\thirdcatvariableof{1}$};
                    \draw[\newmessagecolor,dashed, ->] (3,1.25) to [bend right = -30] (6,1.25);
                    \node[\newmessagecolor,anchor=center] (A) at (4.5,2.25) {\colorlabelsize $\messagesymbol^{(1,2)}$};
                \end{scope}

                \node[anchor=center] at (15.5,0) {$\cdots$};

                \begin{scope}[shift={(22,0)}]
                    \draw[\newmessagecolor,dashed, ->] (-4,1.25) to [bend right = -30] (-1,1.25);
                    \node[\newmessagecolor,anchor=center] (A) at (-2.5,2.25) {\colorlabelsize $\messagesymbol^{(\catorder\shortminus2,2)}$};

                    \draw[->-] (-4,0)--(-1,0) node[midway,above] {\colorlabelsize $\thirdcatvariableof{\catorder\shortminus 2}$};
                    \draw (-1,-1) rectangle (3,1);
                    \node[anchor=center] (A) at (1,0) {\corelabelsize $\bencodingof{\modsumsymbol,\catorder\shortminus1}$};
                    \draw[->-] (1,1)--(1,3) node[midway,left] {\colorlabelsize $\headvariableof{\catorder\shortminus1}$};
                    \draw[-<-] (0,-1)--(0,-2.5) node[midway,left] {\colorlabelsize $\catvariableof{\catorder\shortminus1}$};
                    \draw (-0.75,-2.5) rectangle (0.75,-4);
                    \node[anchor=center] (A) at (0,-3.25) {\corelabelsize $\onehotmapof{\catindexof{\catorder\shortminus1}}$};
                    \draw[\newmessagecolor,dashed, ->] (-1.5,-3.25) to [bend right = -45] (-1.5,-0.75);
                    \node[\newmessagecolor,anchor=center] (A) at (-2.3,-3.8) {\colorlabelsize $\messagesymbol^{(\catorder\shortminus1,0)}$};
                    \draw[-<-] (2,-1)--(2,-2.5) node[midway,right] {\colorlabelsize $\tildecatvariableof{\catorder\shortminus1}$};
                    \draw (1.25,-2.5) rectangle (2.75,-4);
                    \draw[\newmessagecolor,dashed, ->] (3.5,-3.25) to [bend right = 45] (3.5,-0.75);
                    \node[\newmessagecolor,anchor=center] (A) at (4.3,-3.8) {\colorlabelsize $\messagesymbol^{(\catorder\shortminus1,1)}$};
                    \node[anchor=center] (A) at (2,-3.25) {\corelabelsize $\onehotmapof{\tildecatindexof{\catorder\shortminus1}}$};
                    \draw (3,0)--(4.5,0) -- (4.5,1);
                    \draw[->-] (4.5,1)--(4.5,3) node[midway,left] {\colorlabelsize $\headvariableof{\catorder}$};
                \end{scope}

            \end{tikzpicture}
        \end{center}
        \caption{Computation of the integer sum in $\catdim$-adic representation by the Directed Belief Propagation \algoref{alg:directedBeliefPropagation} (see \exaref{exa:madicPropagation}).
        The summands are represented by one-hot encodings of the digits $\shortcatindices$ and $\tildecatindexof{[\catorder]}$, from which the messages start.
        The $\catenumerator$-th digit (for $\catenumerator\in\{0,\ldots,\catorder-1\}$) of the sum is computed based on the first messages of the epoch labeled by $\messagesymbol^{(\catenumerator,[2])}$,
            The third message $\messagesymbol^{(\catenumerator,2)}$ in each epoch communicates the carry bit to the next digit summation core.
            In the last message epoch the digit $\catorder-1$ and $\catorder$ are computed based.
        }
        \label{fig:propagationMary}
    \end{figure}

\end{example}

%
\section{The logical paradigm}\label{sec:logPar}

A tensor-based representation of propositional logic is developed by defining formulas as boolean valued tensors, and showing how logical connectives and normal forms can be expressed as tensor contractions.

\subsection{Propositional semantics by boolean tensors}

Starting with the introduction of propositional formulas as boolean tensors their decomposition is discussed with respect to a basis encoding.

\begin{definition}
    \label{def:formulas}
    A \emph{propositional formula} $\formulaat{\shortcatvariables}$ depending on $\atomorder$ boolean variables $\catvariableof{\atomenumerator}$ is a tensor
    \begin{align*}
        \formulaat{\shortcatvariables} \in \bigotimes_{\atomenumeratorin} \rr^2
    \end{align*}
    with coordinates in $\ozset$.
    We call a state $\shortcatindices \in \atomstates$ a \emph{model} of a propositional formula $\formula$, if
    \begin{align*}
        \formulaat{\indexedshortcatvariables}=1 \, ,
    \end{align*}
    where we understand $1$ as a representation of $\mathrm{True}$ and $0$ of $\mathrm{False}$.
    If there is a model of a propositional formula, we say the formula is \emph{satisfiable}.
\end{definition}

\begin{example}
    \label{exa:propFormulaCoordinatewise}
    Let there be $\catorder=3$ boolean variables $\catvariableof{[3]}$ and a propositional formula
    \begin{align*}
        \formulaat{\catvariableof{[3]}}
        = (\catvariableof{0} \lor \catvariableof{1}) \land \lnot \catvariableof{2} \, .
    \end{align*}
    In a graphical depiction and in the coordinatewise representation this formula can be represented as
    \begin{center}
        \begin{tikzpicture}[scale=1]

            \begin{scope}[shift={(-4,-0.2)}]
                \node[anchor=east] (A) at (-0.25,0.2) {$\formulaat{\catvariableof{[3]}}\,=$};
                \draw (0,0) rectangle (1.6,0.8);
                \node[anchor=center] (A) at (0.8,0.4) {$\formula$};
                \draw (0.2,0) -- (0.2,-0.6) node[midway,left] {\tiny $\catvariableof{0}$};
                \draw (0.8,0) -- (0.8,-0.6) node[midway,left] {\tiny $\catvariableof{1}$};
                \draw (1.4,0) -- (1.4,-0.6) node[midway,left] {\tiny $\catvariableof{2}$};
            \end{scope}

            \node[anchor=east] (A) at (-1.5,0) {$=$};
            \node (A) at (0,0) {
                $\begin{bmatrix}
                     0 & 1 \\
                     1 & 1
                \end{bmatrix}$
            };
            \node (A) at (1.25,0.3) {
                $\begin{bmatrix}
                     0 & 0 \\
                     0 & 0
                \end{bmatrix}$
            };
            \draw[<-,dashed] (-0.9,-0.275) node[right] {\tiny $1$} -- (-0.9,0.275) node [midway, left] {\tiny $\catvariableof{0}$} node[right] {\tiny $0$};
            \draw[->,dashed] (-0.3,0.85) node[below] {\tiny $0$} -- (0.3,0.85) node [midway, above] {\tiny $\catvariableof{1}$} node[below] {\tiny $1$};
            \draw[->,dashed] (0,-0.85) node[above] {\tiny $0$} -- (1.25,-0.55) node [midway, below] {\tiny $\catvariableof{2}$} node[above] {\tiny $1$};

        \end{tikzpicture}
    \end{center}
    In the state set $\atomstates = \{0,1\}\times \{0,1\} \times \{0,1\}$ we have three models of the formula by the positions of the non-zero entries in the tensor, that is $\formulaat{\indexedcatvariableof{[3]}}=1$ if and only if
    \begin{align*}
        \catindexof{[3]}\in\{(1,0,0),(0,1,0),(1,1,0)\} \, .
    \end{align*}
    The formula $\formula$ is therefore satisfiable.
\end{example}

\paragraph{Model counts by contraction}
Each coordinate of a propositional formula is either $1$ or $0$, indicating whether the indexed state is a model of the formula or not.
In this way, the contraction $\contraction{\formula}$ counts the number of models of the propositional formula $\formula$.
One can therefore decide the satisfiability of a formula by testing for $\contraction{\formula}>0$.

\paragraph{CP decomposition}
We can decompose a formula into the sum of the one-hot encodings of its models:
\begin{center}
    \begin{tikzpicture}[scale=0.35, thick]

    \draw (-1,-1) rectangle (5,-3);
    \node[anchor=center] (text) at (2,-2) {\corelabelsize ${\exformula}$};
    \draw[] (0,-3)--(0,-5) node[midway,left] {\colorlabelsize $\catvariableof{0}$};
    \draw[] (1.5,-3)--(1.5,-5) node[midway,left] {\colorlabelsize $\catvariableof{1}$};
    \node[anchor=center] (text) at (3,-4) {$\cdots$};
    \draw[] (4,-3)--(4,-5) node[midway,right] {\colorlabelsize $\catvariableof{\atomorder\shortminus1}$};

    \node[anchor=center] (text) at (7,-2) {${=}$};

    \node[anchor=center] (text) at (12,-2.5) {$\sum\limits_{\shortcatindicesin}$};
    \node[anchor=center] (text) at (12,-4.25) {\colorlabelsize $\formulaat{\indexedshortcatvariables}=1$};

    \begin{scope}
    [shift={(19.5,1)}]

        \draw (-3,-2) rectangle (-1,-4);
        \node[anchor=center] (text) at (-2,-3) {\corelabelsize $\onehotmapof{\atomlegindexof{0}}$};
        \draw[->-] (-2,-4)--(-2,-6) node[midway,right] {\colorlabelsize $\catvariableof{0}$};

        \node[anchor=center] (text) at (1,-3) {\corelabelsize $\cdots$};

        \draw (3,-2) rectangle (5,-4);
        \node[anchor=center] (text) at (4,-3) {\corelabelsize $\onehotmapof{\atomlegindexof{\atomorder\shortminus1}}$};
        \draw[->-] (4,-4)--(4,-6) node[midway,right] {\colorlabelsize $\catvariableof{\atomorder\shortminus1}$};

    \end{scope}

    \node[anchor=center] (text) at (26,-2) {${=}$};

    \begin{scope}[shift={(29,-2.5)}]

        \draw (-1,-1) rectangle (1,1);
        \node[anchor=center] (A) at (0,0) {\corelabelsize $\hypercoreof{0}$};
        \draw (0,-1)--(0,-2.5) node[midway,right] {\colorlabelsize $\catvariableof{0}$};

        \draw (3,-1) rectangle (5,1);
        \node[anchor=center] (A) at (4,0) {\corelabelsize $\hypercoreof{1}$};
        \draw (4,-1)--(4,-2.5) node[midway,right] {\colorlabelsize $\catvariableof{1}$};

        \node[anchor=center] (text) at (8,0) {$\hdots$};

        \draw (11,-1) rectangle (13,1);
        \node[anchor=center] (A) at (12,0) {\corelabelsize $\hypercoreof{\catorder\shortminus1}$};
        \draw (12,-1)--(12,-2.5) node[midway,right] {\colorlabelsize $\catvariableof{\catorder\shortminus1}$};

        \drawvariabledot{6}{4}
        \node[anchor=south] (text) at (6,4) {\colorlabelsize $\decvariable$};

        \draw (6,4) to[bend right= 20] (0,1);
        \draw (6,4) to[bend right= 10] (4,1);
        \draw (6,4) to[bend right= -20] (12,1);

    \end{scope}

\end{tikzpicture}
\end{center}
As already depicted, one can exploit this summation to find a $\cpformat$ decomposition of the formula.
To this end, we enumerate the models $\shortcatindices^{\decindex}$ of $\formula$ by a decomposition variable $\decvariable$ with values $\decindex\in[\contraction{\formula}]$ and define, for $\catenumeratorin$, cores with slices
\begin{align*}
    \hypercoreofat{\catenumerator}{\catvariableof{\catenumerator},\indexeddecvariable}
    = \onehotmapofat{\catindexof{\catenumerator}^{\decindex}}{\catvariableof{\catenumerator}} \, .
\end{align*}

\begin{example}\label{exa:propFormulaBasCP}
    For the formula described in \exaref{exa:propFormulaCoordinatewise}, we have
    \begin{align*}
        \formulaat{\catvariableof{[3]}}
        &= \left(\tbasisat{\catvariableof{0}} \otimes \fbasisat{\catvariableof{1}} \otimes \fbasisat{\catvariableof{2}}\right)
        + (\fbasisat{\catvariableof{0}} \otimes \tbasisat{\catvariableof{1}} \otimes \fbasisat{\catvariableof{2}}) \\
        &\quad+ (\tbasisat{\catvariableof{0}} \otimes \tbasisat{\catvariableof{1}} \otimes \fbasisat{\catvariableof{2}}) \, .
    \end{align*}
    Note that we have $\contraction{\formulawith}=3$ and we can interpret this sum as a $\cpformat$ decomposition of $\formula$ with rank $3$.
    We use the decomposition to evaluate the formula $\formula$ at $\catindexof{[3]} = (1,1,0)$ and get
    \begin{align*}
        \formulaat{\indexedcatvariableof{[3]}}
        &= \left(\tbasisat{\catvariableof{0}=1} \otimes \fbasisat{\catvariableof{1}=1} \otimes \fbasisat{\catvariableof{2}=0}\right) \\
       &\quad + (\fbasisat{\catvariableof{0}=1} \otimes \tbasisat{\catvariableof{1}=1} \otimes \fbasisat{\catvariableof{2}=0}) \\
       &\quad + (\tbasisat{\catvariableof{0}=1} \otimes \tbasisat{\catvariableof{1}=1} \otimes \fbasisat{\catvariableof{2}=0}) \\
        &=  1\cdot 0 \cdot 1 + 0\cdot 1 \cdot 1 + 1 \cdot 1 \cdot 1 = 1 \, ,
    \end{align*}
    which verifies that $\catindexof{[3]} = (1,1,0)$ is a model of the formula $\formula$.
\end{example}

\paragraph{Basis encoding}
Representing booleans by elements in $\{0,1\}$ leads to the problem that the negation is an affine transformation and cannot be represented by multilinear tensors. 
To be able to express different kinds of connectives by contractions, booleans are encoded by one-hot encodings as defined in \defref{def:onehotenc}.
Propositional formulas $\formula$ can then be expressed by their basis encodings
\begin{align*}
    \bencodingofat{\formula}{\indexedheadvariableof{\formula},\indexedshortcatvariables}
    = \begin{cases}
          1 & \ifspace \formulaat{\indexedshortcatvariables} = \headindexof{\formula}\\
          0 & \text{else}
    \end{cases} \, .
\end{align*}
This basis encoding $\bencodingofat{\formula}{\headvariableof{\formula},\shortcatvariables} \in \{0,1\}^{2\times 2^d}$ encodes the formula itself and its negation in its slices, since
\begin{align*}
    \bencodingofat{\formula}{\headvariableof{\formula},\shortcatvariables}
    = \tbasisat{\headvariableof{\formula}} \otimes \formulaat{\shortcatvariables}
    + \fbasisat{\headvariableof{\formula}} \otimes \lnot\formulaat{\shortcatvariables} \, .
\end{align*}
In our graphical notation this property is visualized by
\begin{center}
    \begin{tikzpicture}[scale=0.35, thick] 

    \draw[->-] (2,-1)--(2,1) node[midway,right] {\colorlabelsize $\formulavar$};
    \draw (-1,-1) rectangle (5,-3);
    \node[anchor=center] (text) at (2,-2) {\corelabelsize $\bencodingof{\exformula}$};
    \draw[-<-] (0,-3)--(0,-5) node[midway,left] {\colorlabelsize $\catvariableof{0}$};
    \draw[-<-] (1.5,-3)--(1.5,-5) node[midway,left] {\colorlabelsize $\catvariableof{1}$};
    \node[anchor=center] (text) at (3,-4) {$\cdots$};
    \draw[-<-] (4,-3)--(4,-5) node[midway,right] {\colorlabelsize $\catvariableof{\atomorder\shortminus1}$};

    \node[anchor=center] (text) at (7,-2) {${=}$};

    \node[anchor=center] (text) at (10,-2.5) {${\sum\limits_{\shortcatindices\in\atomstates}}$};

    \begin{scope}
    [shift={(15.5,-0.5)}]

        \draw (-2,1) rectangle (4,-1);
        \node[anchor=center] (text) at (1,0) {\corelabelsize $\onehotmapof{\exformulaat{\indexedshortcatvariables}}$};
        \draw[->-] (1,1)--(1,2.5) node[midway,right] {\colorlabelsize $\formulavar$};

        \draw (-2,-2) rectangle (4,-4);
        \node[anchor=center] (text) at (1,-3) {\corelabelsize $\onehotmapof{\shortcatindices}$};

        \draw[->-] (-1.5,-4)--(-1.5,-5.5) node[midway,left] {\colorlabelsize $\catvariableof{0}$};
        \draw[->-] (0.5,-4)--(0.5,-5.5) node[midway,left] {\colorlabelsize $\catvariableof{1}$};
        \node[anchor=center] (text) at (2,-5) {$\cdots$};
        \draw[->-] (3.5,-4)--(3.5,-5.5) node[midway,right] {\colorlabelsize $\catvariableof{\atomorder\shortminus1}$};

    \end{scope}

    \node[anchor=center] (text) at (21.25,-2) {${=}$};

    \begin{scope} [shift={(25.5,-0.5)}]

        \draw (-2,1) rectangle (4,-1);
        \node[anchor=center] (text) at (1,0) {\corelabelsize $\onehotmapof{0}$};
        \draw[->-] (1,1)--(1,2.5) node[midway,right] {\colorlabelsize $\formulavar$};

        \draw (-2,-2) rectangle (4,-4);
        \node[anchor=center] (text) at (1,-3) {\corelabelsize $\lnot\formula$};

        \draw[] (-1.5,-4)--(-1.5,-5.5) node[midway,left] {\colorlabelsize $\catvariableof{0}$};
        \draw[] (0.5,-4)--(0.5,-5.5) node[midway,left] {\colorlabelsize $\catvariableof{1}$};
        \node[anchor=center] (text) at (2,-5) {$\cdots$};
        \draw[] (3.5,-4)--(3.5,-5.5) node[midway,right] {\colorlabelsize $\catvariableof{\atomorder\shortminus1}$};

    \end{scope}

    \node[anchor=center] (text) at (31.25,-2) {${+}$};

    \begin{scope} [shift={(35.5,-0.5)}]

        \draw (-2,1) rectangle (4,-1);
        \node[anchor=center] (text) at (1,0) {\corelabelsize $\onehotmapof{1}$};
        \draw[->-] (1,1)--(1,2.5) node[midway,right] {\colorlabelsize $\formulavar$};

        \draw (-2,-2) rectangle (4,-4);
        \node[anchor=center] (text) at (1,-3) {\corelabelsize $\formula$};

        \draw[] (-1.5,-4)--(-1.5,-5.5) node[midway,left] {\colorlabelsize $\catvariableof{0}$};
        \draw[] (0.5,-4)--(0.5,-5.5) node[midway,left] {\colorlabelsize $\catvariableof{1}$};
        \node[anchor=center] (text) at (2,-5) {$\cdots$};
        \draw[] (3.5,-4)--(3.5,-5.5) node[midway,right] {\colorlabelsize $\catvariableof{\atomorder\shortminus1}$};

    \end{scope}

\end{tikzpicture}
\end{center}
We further provide a more detailed example in coordinate sensitive notation in the following.
\begin{example}[Logical negation and conjunction]
    \label{exa:bencodingNegCon} 
    The basis encodings of the negation $\notucon: [2]\rightarrow [2]$ is the matrix
    \begin{center}
        \begin{tikzpicture}[scale=1]
            \node (A) at (-2.5,0) {$\bencodingofat{\lnot}{\headvariableof{\lnot},\catvariable}$=};
            \node (A) at (0,0) {
                $\begin{bmatrix}
                     0 & 1 \\
                     1 & 0
                \end{bmatrix}$
            };
            \draw[<-,dashed] (-0.9,-0.275) node[right] {\tiny $1$} -- (-0.9,0.275) node [midway, left] {\tiny $\catvariableof{0}$} node[right] {\tiny $0$};
            \draw[->,dashed] (-0.3,0.85) node[below] {\tiny $0$} -- (0.3,0.85) node [midway, above] {\tiny $\headvariableof{\lnot}$} node[below] {\tiny $1$};
        \end{tikzpicture}
    \end{center}
    The $2$-ary conjunctions $\land:  [2]\times[2] \rightarrow[2]$ is encoded by the order-$3$ tensor
    \begin{center}
        \begin{tikzpicture}[scale=1]
            \node (A) at (-4.5,0) {$\bencodingofat{\land}{\headvariableof{\land},\catvariableof{0},\catvariableof{1}}$=};

            \begin{scope}[shift={(0,0)}]
                \node(B) at (-2,0){
                    $\begin{bmatrix}
                         1 \\
                         0
                    \end{bmatrix}$
                };
                \draw[<-,dashed] (-2.5,-0.275) node[right] {\tiny $1$} -- (-2.5,0.275) node [midway, left] {\tiny $\headvariableof{\land}$} node[right] {\tiny $0$};
                \node (A) at (-1.6,0) {$\otimes$};
                \node (A) at (0,0) {
                    $\begin{bmatrix}
                         1 & 1 \\
                         1 & 0
                    \end{bmatrix}$
                };
                \draw[<-,dashed] (-0.9,-0.275) node[right] {\tiny $1$} -- (-0.9,0.275) node [midway, left] {\tiny $\catvariableof{0}$} node[right] {\tiny $0$};
                \draw[->,dashed] (-0.3,0.85) node[below] {\tiny $0$} -- (0.3,0.85) node [midway, above] {\tiny $\catvariableof{1}$} node[below] {\tiny $1$};
            \end{scope}

            \begin{scope}[shift={(4.25,0)}]

                \node[anchor=center] (A) at (-3.25,0) {$+$};

                \node(B) at (-2,0){
                    $\begin{bmatrix}
                         0 \\
                         1
                    \end{bmatrix}$
                };
                \draw[<-,dashed] (-2.5,-0.275) node[right] {\tiny $1$} -- (-2.5,0.275) node [midway, left] {\tiny $\headvariableof{\land}$} node[right] {\tiny $0$};
                \node (A) at (-1.6,0) {$\otimes$};
                \node (A) at (0,0) {
                    $\begin{bmatrix}
                         0 & 0 \\
                         0 & 1
                    \end{bmatrix}$
                };
                \draw[<-,dashed] (-0.9,-0.275) node[right] {\tiny $1$} -- (-0.9,0.275) node [midway, left] {\tiny $\catvariableof{0}$} node[right] {\tiny $0$};
                \draw[->,dashed] (-0.3,0.85) node[below] {\tiny $0$} -- (0.3,0.85) node [midway, above] {\tiny $\catvariableof{1}$} node[below] {\tiny $1$};
            \end{scope}

            \begin{scope}[shift={(7,0)}]
                \node[anchor=center] (A) at (-1.75,0) {$=$};

                \node (A) at (0,0) {
                    $\begin{bmatrix}
                         1 & 1 \\
                         1 & 0
                    \end{bmatrix}$
                };
                \node (A) at (1.25,0.3) {
                    $\begin{bmatrix}
                         0 & 0 \\
                         0 & 1
                    \end{bmatrix}$
                };
                \draw[<-,dashed] (-0.9,-0.275) node[right] {\tiny $1$} -- (-0.9,0.275) node [midway, left] {\tiny $\catvariableof{0}$} node[right] {\tiny $0$};
                \draw[->,dashed] (-0.3,0.85) node[below] {\tiny $0$} -- (0.3,0.85) node [midway, above] {\tiny $\catvariableof{1}$} node[below] {\tiny $1$};
                \draw[->,dashed] (0,-0.85) node[above] {\tiny $0$} -- (1.25,-0.55) node [midway, below] {\tiny $\headvariableof{\land}$} node[above] {\tiny $1$};
            \end{scope}
        \end{tikzpicture}
    \end{center}
    Furthermore, the $2$-ary disjunction $\lor:  [2]\times[2] \rightarrow[2]$ is encoded by the order-$3$ tensor
    \begin{center}
        \begin{tikzpicture}[scale=1]
            \node (A) at (-3,0) {$\bencodingofat{\lor}{\headvariableof{\lor},\catvariableof{0},\catvariableof{1}}$=};
            \node (A) at (0,0) {
                $\begin{bmatrix}
                     1 & 0 \\
                     0 & 0
                \end{bmatrix}$
            };
            \node (A) at (1.25,0.3) {
                $\begin{bmatrix}
                     0 & 1 \\
                     1 & 1
                \end{bmatrix}$
            };
            \draw[<-,dashed] (-0.9,-0.275) node[right] {\tiny $1$} -- (-0.9,0.275) node [midway, left] {\tiny $\catvariableof{0}$} node[right] {\tiny $0$};
            \draw[->,dashed] (-0.3,0.85) node[below] {\tiny $0$} -- (0.3,0.85) node [midway, above] {\tiny $\catvariableof{1}$} node[below] {\tiny $1$};
            \draw[->,dashed] (0,-0.85) node[above] {\tiny $0$} -- (1.25,-0.55) node [midway, below] {\tiny $\headvariableof{\lor}$} node[above] {\tiny $1$};
        \end{tikzpicture}
    \end{center}
\end{example}

\paragraph{Interpretation as \CompActNets{}}
The propositional formula and its negation can be represented by this tensor via
\begin{align*}
    \formulaat{\shortcatvariables}
    = \contractionof{\tbasisat{\headvariableof{\formula}},\bencodingofat{\formula}{\headvariableof{\formula},\shortcatvariables}}{\shortcatvariables}
    \andspace
    \lnot\formulaat{\shortcatvariables}
    = \contractionof{\fbasisat{\headvariableof{\formula}},\bencodingofat{\formula}{\headvariableof{\formula},\shortcatvariables}}{\shortcatvariables} \, .
\end{align*}
Both $\formula$ and $\lnot\formula$ are thus \ComputationActivationNetworks{} to the statistic $\{\formula\}$ and the hard activation tensor $\tbasisat{\headvariableof{\formula}}$, respectively $\fbasisat{\headvariableof{\formula}}$.

\subsection{Syntactic decomposition of propositional formulas}

Propositional formulas of concern often have a syntactic specification as composed functions.
We can therefore apply the neural paradigm to find efficient representations of them.

\begin{definition}[Syntactic decompositions]
    \label{def:syntacticalDecomposition}
    A syntactic decomposition of a propositional formula $\exformula$ is a decomposition hypergraph (see \defref{def:decompositionHypergraph}) such that all nodes are decorated with the dimension $\catdimof{\node}=2$ and the composition function $\exformula$.
\end{definition}

We thus have a tensor network representation of any propositional formula based on its syntactic decomposition, where the hypergraph of the syntactic decomposition equals the hypergraph of the representing tensor network.

\begin{example}
    \label{exa:propFormulaSyntax}
    For the formula $\formulaat{\catvariableof{[3]}} = (\catvariableof{0} \lor \catvariableof{1}) \land \lnot \catvariableof{2}$ from \exaref{exa:propFormulaCoordinatewise}, we have the following syntactical decomposition of its basis encoding:
    \begin{center}
        \begin{tikzpicture}[scale=0.4, yscale=-1, thick] 

            \draw[->-] (-2,1)--(-2,-1) node[midway,left] {\colorlabelsize $\catvariableof{0}$};
            \draw[->-] (0.5,1)--(0.5,-1) node[midway,right] {\colorlabelsize $\catvariableof{1}$};
            \draw[->-] (3,1)--(3,-1) node[midway,right] {\colorlabelsize $\catvariableof{2}$};
            \draw (-3,-1) rectangle (4, -3);
            \node[anchor=center] (text) at (0.5,-2) {\corelabelsize $\bencodingof{\exformula}$};
            \draw[->-] (0.5,-3)--(0.5,-5) node[midway,right] {\colorlabelsize $\headvariableof{\exformula}$};

            \node[anchor=center] (text) at (5,-2) {${=}$};

            \begin{scope}
            [shift={(7,0)}]

                \draw[->-] (0,1)--(0,-1) node[midway,left] {\colorlabelsize $\catvariableof{0}$};
                \draw[->-] (3,1)--(3,-1) node[midway,right] {\colorlabelsize $\catvariableof{1}$};
                \draw[->-] (6,1)--(6,-1) node[midway,right] {\colorlabelsize $\catvariableof{2}$};

                \draw (-1,-1) rectangle (4, -3);
                \node[anchor=center] (text) at (1.5,-2) {\corelabelsize $\bencodingof{\lor}$};

                \draw[->-] (1.5,-3) --(1.5,-5) node[midway,right]{\colorlabelsize $\headvariableof{0 \lor 1}$};

                \draw (5,-1) rectangle (7, -3);
                \node[anchor=center] (text) at (6,-2) {\corelabelsize $\bencodingof{\lnot}$};

                \draw[->-] (6,-3) --(6,-5) node[midway,right]{\colorlabelsize $\headvariableof{\lnot 2}$};

                \draw (0.5,-5) rectangle (6.5,-7);
                \node[anchor=center] (text) at (3.5,-6) {\corelabelsize $\bencodingof{\land}$};

                \draw[->-] (4,-7) -- (4,-8.5) node[right] {\colorlabelsize $\headvariableof{(0 \lor 1) \land \lnot 2}$};
                \drawvariabledot{4}{-8}
                \draw[] (4,-8) -- (4,-9);
                \draw (3,-9) rectangle (5,-11);
                \node[anchor=center] (text) at (4,-10) {$\tbasis$};

            \end{scope}

        \end{tikzpicture}
    \end{center}
\end{example}

\subsection{Entailment decision by contractions}

We have already seen that the contraction of a propositional formula counts its models.
This allows to define entailment between two propositional formulas as defined in the following.
To generalize the treatment, we no longer demand that the variables of a formula are of dimension $2$.
We further use $\lnot\formulawith=\oneswith-\formulawith$.

\begin{definition}[Entailment of propositional formulas]
    \label{def:logicalEntailment}
    Given two propositional formulas $\kb$ and $\formula$, we say that $\kb$ entails $\formula$, denoted by $\kb\models\formula$, if any model of $\kb$ is also a model of $\formula$, that is
    \begin{align*}
        \contraction{\kbwith,\lnot\formulawith}=0 \, .
    \end{align*}
    If $\kb\models\lnot\formula$ holds (i.e. $\contraction{\kbwith,\formulawith}$=0), we say that $\kb$ contradicts $\formula$.
\end{definition}

Classically (see e.g. \cite{russell_artificial_2021}) entailment in propositional logics is defined as the unsatisfiability of $\kb\land\lnot\exformula$.
This is equivalent to \defref{def:logicalEntailment} due to the equivalence of $\contraction{\kbwith,\lnot\formulawith}=0$ and $\contraction{(\kb \land (\lnot\exformula))[\shortcatvariables]}=0$, which is the unsatisfiability of $\kb\land\lnot\exformula$.


\begin{example}[$\sudokunum^2\,\times \,\sudokunum^2$ Sudoku]
    \label{exa:sudokuEntailment}
    We index the rows and the columns by tuples $(r0,r1)$ and $(co,c1)$, where $r0,r1,c0,c1\in[\sudokunum]$. The first index indicates the block and the second counts the row or column inside that block.
    For each $r0,r1,c0,c1\in[\sudokunum]$ and $i\in[\sudokunum^2]$ we then define an atomic variable $\catvariableof{r0,r1,c0,c1,i}\in\{0,1\}$ indicating whether in the row $(r0,r1)$ and column $(co,c1)$ the number $i$ is written.
    The Sudoku rules then amount to the formula
    \begin{align*}
        \sudokukbof{\sudokunum}  \coloneqq
        &\left( \bigwedge_{r0,r1,c0,c1\in[\sudokunum]} \left( \woneoplus_{i\in[\sudokunum^2]} \catvariableof{r0,r1,c0,c1,i} \right) \right) \land
        \left( \bigwedge_{r0,r1\in[\sudokunum], i\in[\sudokunum^2]} \left( \woneoplus_{c0,c1\in[\sudokunum]} \catvariableof{r0,r1,c0,c1,i} \right) \right) \land \\
        &\left( \bigwedge_{c0,c1\in[\sudokunum], i\in[\sudokunum^2]} \left( \woneoplus_{c0,c1\in[\sudokunum]} \catvariableof{r0,r1,c0,c1,i} \right) \right) \land
        \left( \bigwedge_{r0,c0\in[\sudokunum], i\in[\sudokunum^2]} \left( \woneoplus_{r1,c1\in[\sudokunum]} \catvariableof{r0,r1,c0,c1,i} \right) \right) \, ,
    \end{align*}
    where $\woneoplus$ is the $\sudokunum^2$-ary exclusive or connective (that is $1$ if and only if exactly one of the arguments is $1$).
    The four outer brackets in $\kb$ mark the constraints that at each position exactly one number is assigned, further that in each row each number is assigned once, and similar for the columns and the squares of the board.
    When solving a specific Sudoku instance, one typically knows from an initial board assignment $\sudokustartevidence$ a collection of atomic variables, which hold, and needs to find further atomic variables, which are entailed.
    This means, we need to decide for each $(r_0,r_1,c_0,c_1,i)\notin \sudokustartevidence$ whether the Sudoku rules and the initial board imply that the atomic variable $\catvariableof{r0,r1,c0,c1,i}$ (i.e. assignment to the board) is true
    \begin{align*}
        \sudokukbof{\sudokunum} \land \left(\bigwedge_{(r_0,r_1,c_0,c_1,i)\in \sudokustartevidence} \catvariableof{r0,r1,c0,c1,i} \right) \models \catvariableof{r0,r1,c0,c1,i}
    \end{align*}
    or false
    \begin{align*}
        \kb \land \left(\bigwedge_{(r_0,r_1,c_0,c_1,i)\in \sudokustartevidence} \catvariableof{r0,r1,c0,c1,i} \right) \models \lnot\catvariableof{r0,r1,c0,c1,i} \, .
    \end{align*}
    If and only if the Sudoku has a unique solution given the initial board assignment $\sudokustartevidence$, exactly one of these entailment statements holds for each $(r_0,r_1,c_0,c_1,i)\notin \sudokustartevidence$.
    Deciding which is equivalent to solving the Sudoku.

    For a more concrete example, let $n=2$ and
    \begin{align*}
        \sudokustartevidence = \{&(0,0,0,0,0),(0,0,1,0,2),(0,0,1,1,1), 
        (0,1,0,1,1), \\ 
        &(1,0,1,0,3), 
        (1,1,0,0,3),(1,1,0,1,2) 
        \} \, .
    \end{align*}
    We visualize this evidence by writing $i+1$ in a grid cell $(r0,r1,c0,c1)$ to indicate that $(r0,r1,c0,c1,i)\in \sudokustartevidence$:
    \begin{center}
        \begin{sudoku4x4}
            \matrix[sudokumatrix] (M) at (0,0) {
                1 & \ & 3 & 2 \\
                \ & 2 & \  & \  \\
                \ & \ & 4 & \ \\
                4 & 3 &  \ & \  \\
            };
            \draw[thick]([yshift=9.5pt,xshift=-0.6pt]M-1-2.east) -- ([yshift=-9.5pt,xshift=-0.6pt]M-4-2.east);
            \draw[thick]([xshift=-9.5pt,yshift=0.6pt]M-2-1.south) -- ([xshift=9.5pt,yshift=0.6pt]M-2-4.south);
        \end{sudoku4x4}
    \end{center}
    After deriving a sparse tensor network representations in \exaref{exa:sudokuDecomposition}, we demonstrate a solution algorithm to solve this instance in \exaref{exa:sudokuEntailment}.
\end{example}

\subsection{Efficient representation of knowledge bases}

We now investigate the representation of propositional knowledge bases $\kb=\{\formulaof{\selindex}\wcols\selindexin\}$, which are sets of propositional formulas $\formulaof{\selindex}$.
The conjunction of these formulas is the knowledge base formula
\begin{align*}
    \kbwith
    = \bigwedge_{\selindexin} \formulaofat{\selindex}{\shortcatvariables} \, .
\end{align*}
To show efficient representations, we use the following identities.

\begin{lemma}[Computation Network Symmetries]
    \label{lem:comNetSymmetries}
    For the $\catorder$-ary $\land$-connective (where $\catorder\in\nn$) and the unary $\lnot$-connective it holds that
    \begin{align*}
        \contractionof{\tbasisat{\headvariable},\bencodingofat{\land}{\headvariable,\shortcatvariables}}{\shortcatvariables}
        = \bigotimes_{\catenumeratorin} \tbasisat{\catvariableof{\catenumerator}}
        \andspace
        \contractionof{\tbasisat{\headvariable},\bencodingofat{\lnot}{\headvariable,\catvariable}}{\catvariable}
        = \fbasisat{\catvariable} \, .
    \end{align*}
\end{lemma}
\begin{proof}
    Follows directly from the definitions of the basis encodings and the connectives.
\end{proof}

\begin{example}
    \label{exa:propFormulaHeadSym}
    For the propositional formula from \exaref{exa:propFormulaCoordinatewise}
    \begin{align*}
        \formulaat{\catvariableof{[3]}}={(\catvariableof{0} \lor \catvariableof{1}) \land \lnot \catvariableof{2}} \, ,
    \end{align*}
    we can write the formula in terms of a \ComputationActivationNetwork{} with activation tensor $\tbasis$ and computation network decomposed by the basis encodings as depicted below. First, it is written with one activation vector. Second, we see that it can also be interpreted with multiple features.
    \begin{center}
        \begin{tikzpicture}[scale=0.4, yscale=-1, thick] 

            \draw[] (-2,1)--(-2,-1) node[midway,left] {\colorlabelsize $\catvariableof{0}$};
            \draw[] (0.5,1)--(0.5,-1) node[midway,right] {\colorlabelsize $\catvariableof{1}$};
            \draw[] (3,1)--(3,-1) node[midway,right] {\colorlabelsize $\catvariableof{2}$};
            \draw (-3,-1) rectangle (4, -3);
            \node[anchor=center] (text) at (0.5,-2) {\corelabelsize $(\catvariableof{0} \lor \catvariableof{1}) \land \lnot \catvariableof{2}$};

            \node[anchor=center] (text) at (5,-2) {${=}$};

            \begin{scope}
            [shift={(7,0)}]

                \draw[->-] (0,1)--(0,-1) node[midway,left] {\colorlabelsize $\catvariableof{0}$};
                \draw[->-] (3,1)--(3,-1) node[midway,right] {\colorlabelsize $\catvariableof{1}$};
                \draw[->-] (6,1)--(6,-1) node[midway,right] {\colorlabelsize $\catvariableof{2}$};

                \draw (-1,-1) rectangle (4, -3);
                \node[anchor=center] (text) at (1.5,-2) {\corelabelsize $\bencodingof{\lor}$};

                \draw[->-] (1.5,-3) --(1.5,-5) node[midway,right]{\colorlabelsize $\headvariableof{0 \lor 1}$};

                \draw (5,-1) rectangle (7, -3);
                \node[anchor=center] (text) at (6,-2) {\corelabelsize $\bencodingof{\lnot}$};

                \draw[->-] (6,-3) --(6,-5) node[midway,right]{\colorlabelsize $\headvariableof{\lnot 2}$};

                \draw (0.5,-5) rectangle (6.5,-7);
                \node[anchor=center] (text) at (3.5,-6) {\corelabelsize $\bencodingof{\land}$};

                \draw[->-] (4,-7) -- (4,-8.5) node[right] {\colorlabelsize $\headvariableof{(0 \lor 1) \land \lnot 2}$};
                \drawvariabledot{4}{-8}
                \draw[] (4,-8) -- (4,-9);
                \draw (3,-9) rectangle (5,-11);
                \node[anchor=center] (text) at (4,-10) {$\tbasis$};

            \end{scope}

            \node[anchor=center] (text) at (15,-2) {${=}$};

            \begin{scope}
            [shift={(17,0)}]

                \draw[->-] (0,1)--(0,-1) node[midway,left] {\colorlabelsize $\catvariableof{0}$};
                \draw[->-] (3,1)--(3,-1) node[midway,right] {\colorlabelsize $\catvariableof{1}$};
                \draw[] (7,1)--(7,-1) node[midway,right] {\colorlabelsize $\catvariableof{2}$};

                \draw (-1,-1) rectangle (4, -3);
                \node[anchor=center] (text) at (1.5,-2) {\corelabelsize $\bencodingof{\lor}$};

                \draw (1.5,-4.5) -- (1.5,-5);
                \draw[->-] (1.5,-3) --(1.5,-4.5) node[midway,right]{\colorlabelsize $\headvariableof{0 \lor 1}$};

                \drawvariabledot{1.5}{-4}
                \draw (0.5,-5) rectangle (2.5,-7);
                \node[anchor=center] (text) at (1.5,-6) {\corelabelsize $\tbasis$};

                \node[anchor=center] (text) at (5,-2) {$\otimes$};

                \draw (6,-1) rectangle (8, -3);
                \node[anchor=center] (text) at (7,-2) {\corelabelsize $\fbasis$};

            \end{scope}

        \end{tikzpicture}
    \end{center}
\end{example}

We use this to decompose knowledge bases into their individual formulas as follows.

\begin{theorem}
    \label{the:kbDecomposition}
    For any knowledge base $\kbwith = \bigwedge_{\selindexin} \formulaofat{\selindex}{\shortcatvariables}$ it holds that
    \begin{align*}
        \kbwith
        = \contractionof{\{\formulaofat{\selindex}{\shortcatvariables} \wcols \selindexin\}}{\shortcatvariables} \, .
    \end{align*}
\end{theorem}
\begin{proof}
    With \lemref{lem:comNetSymmetries} we have
    \begin{align*}
        \kbwith
        &= \contractionof{\{\tbasisat{\headvariableof{\land}},\bencodingofat{\land}{\headvariableof{\land},\headvariables}\}
            \cup \{\bencodingofat{\formulaof{\selindex}}{\headvariableof{\selindex},\shortcatvariables}\wcols \selindexin\}}{\shortcatvariables} \\
        &= \contractionof{
            \bigcup_{\selindexin} \{\tbasisat{\headvariableof{\selindex}},\bencodingofat{\formulaof{\selindex}}{\headvariableof{\selindex},\shortcatvariables}\wcols \selindexin\}}{\shortcatvariables} \\
        &= \contractionof{\{\formulaofat{\selindex}{\shortcatvariables} \wcols \selindexin\}}{\shortcatvariables} \, .
    \end{align*}
\end{proof}

\begin{example}[Sparse representation of Sudoku rule knowledge base]
    \label{exa:sudokuDecomposition}
    We now exploit \theref{the:kbDecomposition} to find efficient tensor network representation of the Sudoku knowledge base from \exaref{exa:sudokuEntailment}.
    We directly get, that the knowledge base $\sudokukbof{\sudokunum}$ of Sudoku rules is a tensor network of the $4\cdot \sudokunum^4$ constraint formulas using the $\sudokunum^2$-ary connective $\woneoplus$, and the evidence $\sudokustartevidence$ can be encoded by vectors $\tbasisat{\catvariableof{(r_0,r_1,c_0,c_1,i)}}$.
    To get a representation by matrices instead of tensors of order $\sudokunum^2$, we introduce a hidden variable $\decvariable$ taking values in $[\sudokunum^2]$ for each of the constraints.
    With the usage of matrices
    \begin{align*}
        \hypercoreofat{\catenumerator}{\catvariableof{\catenumerator},\decvariable}
        = \fbasisat{\catvariableof{\catenumerator}} \otimes \onesat{\decvariable} + (\tbasisat{\catvariableof{\catenumerator}}-\fbasisat{\catvariableof{\catenumerator}}) \otimes \onehotmapofat{\catenumerator}{\decvariable}
    \end{align*}
    we have the decomposition
    \begin{align*}
        \woneoplus[\catvariableof{[\sudokunum^2]}]
        = \contractionof{\{\hypercoreofat{\catenumerator}{\catvariableof{\catenumerator},\decvariable} \wcols \catenumerator\in[\sudokunum^2]\}}{\catvariableof{[\sudokunum^2]}} \, ,
    \end{align*}
    which is a $\cpformat$ decomposition (see \exaref{exa:cpFormat}) depicted in \figref{fig:sudokuDecomposition} a).

    Alternatively there is a $\ttformat$ decomposition (see \exaref{exa:ttFormat}) of the constraint $\woneoplus$, which we depict in \figref{fig:sudokuDecomposition} b).
    We introduce for $\catenumerator\in[\catorder-1]$ hidden variables $\decvariable^{\catenumerator}$ of dimension $2$, which are interpreted as the indicator, whether one of the variables $\catvariableof{[\catenumerator]}$ is true.
    Following this interpretation we introduce $\ttformat$ cores
    \begin{align*}
        \sechypercoreofat{0}{\catvariableof{0},\secdecvariable^{0}}
        &= \tbasisat{\catvariableof{0}}\otimes\tbasisat{\secdecvariable^{0}}
        + \fbasisat{\catvariableof{0}}\otimes\fbasisat{\secdecvariable^{0}} \\
        \sechypercoreofat{\catorder-1}{\secdecvariable^{\catorder-2},\catvariableof{\catorder-1}}
        &= \fbasisat{\secdecvariable^{\catorder-2}} \otimes \tbasisat{\catvariableof{\catorder-1}}
        + \tbasisat{\secdecvariable^{\catorder-2}} \otimes \fbasisat{\catvariableof{\catorder-1}}
    \end{align*}
    and for $\catenumerator\in\{1,\ldots,\catorder-2\}$
    \begin{align*}
        \sechypercoreofat{\catenumerator}{\secdecvariable^{\catenumerator-1},\catvariableof{\catenumerator},\secdecvariable^{\catenumerator}}
        = \identityat{\secdecvariable^{\catenumerator-1},\secdecvariable^{\catenumerator}}  \otimes \fbasisat{\catvariableof{\catenumerator}}
        + \fbasisat{\secdecvariable^{\catenumerator-1}} \otimes \tbasisat{\catvariableof{\catenumerator}} \otimes \tbasisat{\secdecvariable^{\catenumerator-1}} \, .
    \end{align*}
    We notice, that the $\ttformat$ decomposition of the constraint $\woneoplus$ introduces $\catorder-1$ many hidden variables of dimension $2$, whereas the $\cpformat$ decomposition introduces a single hidden variable of dimension $\catorder$.
    However, in the following we will further apply the $\cpformat$ decomposition. 

    \begin{figure}[t]

        \begin{center}
            \begin{tikzpicture}[scale=0.35,thick]

                \draw (-6,-1) rectangle (-12,1);
                \node[anchor=center] (A) at (-9,0) {\corelabelsize $\woneoplus$};
                \draw (-11.5,-1)--(-11.5,-2.5) node[midway,left] {\colorlabelsize $\catvariableof{0,0,0,0,0}$};
                \draw (-11,-1)--(-11,-2.5) node[midway,right] {\colorlabelsize $\catvariableof{0,0,0,0,1}$};
                \node[anchor=center] (A) at (-9,-2.5) {$\cdots$};
                \draw (-7,-1)--(-7,-2.5) node[midway,right] {\colorlabelsize $\catvariableof{0,0,0,0,\sudokunum^2\shortminus 1}$};

                \node[anchor=center] (A) at (-3.5,1) {$a)$};
                \node[anchor=center] (A) at (-3.5,0) {$=$};

                \draw (-1,-1) rectangle (1,1);
                \node[anchor=center] (A) at (0,0) {\corelabelsize $\hypercoreof{0}$};
                \draw (0,-1)--(0,-2.5) node[midway,right] {\colorlabelsize $\catvariableof{0,0,0,0,0}$};

                \draw (3,-1) rectangle (5,1);
                \node[anchor=center] (A) at (4,0) {\corelabelsize $\hypercoreof{1}$};
                \draw (4,-1)--(4,-2.5) node[midway,right] {\colorlabelsize $\catvariableof{0,0,0,0,1}$};

                \node[anchor=center] (text) at (8,0) {$\hdots$};

                \draw (10.75,-1) rectangle (13.25,1);
                \node[anchor=center] (A) at (12,0) {\corelabelsize $\hypercoreof{\sudokunum^2\shortminus1}$};
                \draw (12,-1)--(12,-2.5) node[midway,right] {\colorlabelsize $\catvariableof{0,0,0,0,\sudokunum^2\shortminus1}$};

                \drawvariabledot{6}{4}
                \node[anchor=south] (text) at (6,4) {\colorlabelsize $\decvariableof{0,0,0,0,:}$};

                \draw (6,4) to[bend right= 20] (0,1);
                \draw (6,4) to[bend right= 10] (4,1);
                \draw (6,4) to[bend right= -20] (12,1);

                \begin{scope}[shift={(0,-7)}]
                    \node[anchor=center] (A) at (-3.5,1) {$b)$};
                    \node[anchor=center] (A) at (-3.5,0) {$=$};

                    \draw (-2,-1) rectangle (0,1);
                    \node[anchor=center] (A) at (-1,0) {\corelabelsize $\sechypercoreof{0}$};
                    \draw (-1,-1)--(-1,-2.5) node[midway,right] {\colorlabelsize $\catvariableof{0,0,0,0,0}$};

                    \draw (0,0) -- (3,0) node [midway,above] {\colorlabelsize $\secdecvariableof{0,0,0,0,:}^0$};

                    \draw (3,-1) rectangle (5,1);
                    \node[anchor=center] (A) at (4,0) {\corelabelsize $\sechypercoreof{1}$};
                    \draw (4,-1)--(4,-2.5) node[midway,right] {\colorlabelsize $\catvariableof{0,0,0,0,1}$};

                    \node[anchor=center] (text) at (6.5,1) {\colorlabelsize $\secdecvariableof{0,0,0,0,:}^{1}$};
                    \draw (5,0) -- (7,0); 

                    \node[anchor=center] (text) at (8,0) {$\hdots$};

                    \node[anchor=center] (text) at (10.25,1) {\colorlabelsize $\secdecvariableof{0,0,0,0,:}^{\sudokunum^2\shortminus2}$};
                    \draw (9.75,0) -- (11.75,0);

                    \draw (11.75,-1) rectangle (14.25,1);
                    \node[anchor=center] (A) at (13,0) {\corelabelsize $\sechypercoreof{\sudokunum^2\shortminus1}$};
                    \draw (13,-1)--(13,-2.5) node[midway,right] {\colorlabelsize $\catvariableof{0,0,0,0,\sudokunum^2\shortminus1}$};

                \end{scope}

            \end{tikzpicture}
        \end{center}
        \caption{Decomposition of the position constraint $\woneoplus$ at position $(r0,r1,c0,c1)=(0,0,0,0)$ into a) a $\cpformat$ decomposition with hidden variable $\decvariableof{0,0,0,0,:}$ and b) a $\ttformat$ decomposition with $d-1$ hidden variables $\decvariableof{0,0,0,0,:}^k, k\in[d-1]$.}
        \label{fig:sudokuDecomposition}
    \end{figure}

    Given evidence $\sudokustartevidence$ we denote the Sudoku knowledge base $\sudokukbof{\sudokunum,\sudokustartevidence}$.
    We model the Sudoku knowledge base $\sudokukbof{\sudokunum,\sudokustartevidence}$ as a tensor network on a hypergraph $\graphof{\mathrm{Sudoku},n}$ consistent in
    \begin{itemize}
        \item $\sudokunum^6+4\cdot \sudokunum^4$ nodes by $\sudokunum^6$ categorical variables $\catvariableof{(r0,r1,c0,c1,i)}$ and by $4\cdot \sudokunum^4$ decomposition variables to the constraints
        \item $5\cdot \sudokunum^6$ edges
        \begin{align*}
            \edges=
            \bigcup_{r0,r1,c0,c1\in[\sudokunum]}
            \big\{
            &\{\catvariableof{(r0,r1,c0,c1,i)}\},\{\catvariableof{(r0,r1,c0,c1,i)},\decvariableof{r0,r1,c0,c1,:}\},\{\catvariableof{(r0,r1,c0,c1,i)},\decvariableof{r0,r1,:,:,i}\},\\
            &\{\catvariableof{(r0,r1,c0,c1,i)},\decvariableof{:,:,c0,c1,i}\},
            \{\catvariableof{(r0,r1,c0,c1,i)},\decvariableof{r0,:,c0,:,i}\}\big\}
        \end{align*}
        We denote the decomposition variables to the position, row, column and square constraints by $\decvariableof{r0,r1,c0,c1,:}$, $\decvariableof{r0,r1,:,:,i}$, $\decvariableof{:,:,c0,c1,i}$ and $\decvariableof{r0,:,c0,:,i}$.
    \end{itemize}
    Each edge containing a decomposition variable is decorated by a matrix $\hypercoreofat{\catenumerator}{\catvariable,\decvariable}$ corresponding to a core in the $\cpformat$ decomposition of a constraint.
    Here $\catenumerator$ is determined by the tuple $(r0,r1,c0,c1,i)$ and the type of the constraint (for example, for the variable $\catvariableof{(0,1,1,2,1)}$ and the row constraint $\decvariableof{(0,1,:,:,1)}$ we have $\catenumerator=1\cdot n + 2$.
    We further assign to each edge containing a single variable $\{\catvariableof{(r0,r1,c0,c1,i)}\}$ either the vector $\tbasisat{\catvariableof{(r0,r1,c0,c1,i)}}$ if $(r0,r1,c0,c1,i)\in \sudokustartevidence$ or the trivial vector $\onesat{\catvariableof{(r0,r1,c0,c1,i)}}$.
\end{example}

\subsection{Entailment decision by message passing}

Since contracting the whole tensor network is often infeasible, local contractions can be considered to decide entailment in some cases.
Here, a local contraction describes the calculation of contractions along few closely connected tensors in the network.
Before presenting the resulting Constraint Propagation algorithm, we first show two important properties of local entailment motivating the procedure.

\begin{theorem}[Monotonicity of propositional logics]
    \label{the:monotonicityPL}
    If $\seckb\subset\kb$ and $\seckb\models\formula$ then also $\kb\models\formula$.
\end{theorem}
\begin{proof}
    Since $\seckb\models\formula$ it holds that $\contraction{\seckb[\shortcatvariables],\lnot\formula[\shortcatvariables]}=0$ and thus
    \begin{align*}
        \contractionof{\seckb[\shortcatvariables],\lnot\formula[\shortcatvariables]}{\shortcatvariables}=\zerosat{\shortcatvariables} \, .
    \end{align*}
    Denoting by $\kb/\seckb$ the conjunctions of formulas in $\kb$ not in $\seckb$, we have
    \begin{align*}
        \contraction{\kbwith,\lnot\formulawith}
        &= \contraction{(\kb/\seckb)[\shortcatvariables],\seckb[\shortcatvariables],\lnot\formulawith} \\
        &= \contraction{(\kb/\seckb)[\shortcatvariables],\contractionof{\seckb[\shortcatvariables],\lnot\formulawith}{\shortcatvariables}} \\
        &= \contraction{(\kb/\seckb)[\shortcatvariables],\zerosat{\shortcatvariables}} \\
        &= 0 \, .
    \end{align*}
\end{proof}
To decide entailment, we can therefore investigate entailment on smaller parts of the knowledge base.
This is sound by the above theorem but not complete since it can happen that no smaller part of the knowledge base entails the formula while the whole knowledge base does.
We can furthermore add entailed formulas to the knowledge base without changing it as is shown next.

\begin{theorem}[Invariance of adding entailed formulas]
    \label{the:addingEntailed}
    If and only if $\kb\models\formula$ we have that
    \begin{align*}
        \kbwith
        = \contractionof{\kbwith,\formulawith}{\shortcatvariables} \, .
    \end{align*}
\end{theorem}
\begin{proof}
    We use that $\formulawith+\lnot\formulawith=\onesat{\shortcatvariables}$ and thus
    \begin{align*}
        \kbwith
        &= \contractionof{\kbwith,(\formulawith+\lnot\formulawith)}{\shortcatvariables} \\
        &= \contractionof{\kbwith,\formulawith}{\shortcatvariables}  + \contractionof{\kbwith,\lnot\formulawith}{\shortcatvariables} \,. \\
    \end{align*}
    Since $\contractionof{\kbwith,\lnot\formulawith}{\shortcatvariables}$ is boolean, we have that
    \begin{align*}
        \kbwith=\contractionof{\kbwith,\formulawith}{\shortcatvariables}
    \end{align*}
    if and only if $\contraction{\kbwith,\lnot\formulawith}=0$, that is $\kb\models\formula$.
\end{proof}

The mechanism of \theref{the:addingEntailed} provides us with a means to store entailment information in small-order auxiliary tensors.
One way to exploit this accessibility of local entailment information are message passing schemes similar to \algoref{alg:treeBeliefPropagation} propagating the information.
This approach decides local entailment by iteratively adding entailed formulas to the knowledge base and checking further entailment on neighboring tensors of the knowledge base.
Since for entailment decisions the support of the contractions is sufficient, we can apply non-zero indicators before sending contraction messages.
We then schedule new messages in the direction $(\sedge,\redge)$ once the support of a message received at $\sedge$ has been changed.
Note that such a scheduling system is guaranteed to converge since there can only be a finite number of message changes.
We further directly reduce the computation of messages to their support and call the resulting Constraint Propagation (\algoref{alg:constraintPropagation}).

\begin{algorithm}[hbt!]
    \caption{Constraint Propagation}\label{alg:constraintPropagation}
    \begin{algorithmic}
        \Require Tensor network $\extnet$ on a hypergraph $\graph$
        \Ensure Messages $\{\messagewith\wcols(\sedge,\redge)\in\dirovedges\}$ containing entailment statements
        \iosepline
        \State Initialize a queue $\scheduler = \dirovedges$ of message directions
        \State Initialize messages $\messagewith = \onesat{\catvariableof{\sedge\cap\redge}}$ for $(\sedge,\redge)\in\dirovedges$
        \While{$\scheduler$ not empty}
            \State Pop a $(\sedge,\redge)$ pair from $\scheduler$
            \State Update the message
            \begin{align*}
                \messagewith
                = \nonzeroof{\contractionof{\{\hypercoreofat{\sedge}{\catvariableof{\sedge}}\}
                    \cup \{\mesfromtoat{\secsedge}{\sedge}{\catvariableof{\secsedge\cap\sedge}} \wcols (\secsedge,\sedge)\in\dirovedges \ncond \secsedge\neq \redge\}
                }{\catvariableof{\sedge\cap \redge}}}
            \end{align*}
            \If{$\hypercoreat{\catvariableof{\sedge\cap \redge}}\neq\messagewith$}
                \State Update the message: $\messagewith\coloneqq\hypercoreat{\catvariableof{\sedge\cap \redge}}$
                \State Add $\scheduler = \scheduler \cup \{(\redge,\secsedge) \wcols (\redge,\secsedge)\in\dirovedges\}$ 
            \EndIf
        \EndWhile
        \State \Return Messages $\{\messagewith\wcols(\sedge,\redge)\in\dirovedges\}$
    \end{algorithmic}
\end{algorithm}

\begin{theorem}
    \label{the:constraintPropagationSoundness}
    All messages during constraint propagation are sound, meaning that for all $(\sedge,\redge)\in\dirovedges$ it holds that
    \begin{align*}
        \nonzeroof{\contractionof{\extnet}{\catvariableof{\sedge\cap\redge}}} \prec \messagewith \, .
    \end{align*}
\end{theorem}
\begin{proof}
    We show this theorem by induction over the \whileSymbol{} loop of \algoref{alg:constraintPropagation}.
    At the first iteration, we have for all messages $\messagewith=\onesat{\catvariableof{\sedge\cap\redge}}$ and thus
    \begin{align}
        \label{eq:nzMessageAddingEquivalence}
        \extnet = \contractionof{\{\extnet\}\cup\{\messagewith\wcols(\sedge,\redge)\in\dirovedges\}}{\nodevariables} \, .
    \end{align}
    By \theref{the:monotonicityPL} we then have for the first message send along the pair $(\sedge,\redge)$ that
    \begin{align*}
        \nonzeroof{\contractionof{\extnet}{\catvariableof{\sedge\cap\redge}}} \prec
        &\nonzeroof{\contractionof{\{\hypercoreofat{\sedge}{\catvariableof{\sedge}}\}
            \cup \{\mesfromtoat{\secsedge}{\sedge}{\catvariableof{\secsedge\cap\sedge}} \wcols (\secsedge,\sedge)\in\dirovedges \ncond \secsedge\neq \redge\}
        }{\catvariableof{\sedge\cap \redge}}} \\
        &= \messagewith \, .
    \end{align*}

    We now assume that at an arbitrary state of the algorithm the inequality holds for all previously sent messages.
    By \theref{the:addingEntailed} we can contract the messages with the tensor network without changing it and \eqref{eq:nzMessageAddingEquivalence} thus still holds.
    We then conclude with \theref{the:monotonicityPL} that the claimed property also holds for the new message.
\end{proof}

\begin{example}[Message passing for the Sudoku instance of \exaref{exa:sudokuEntailment}]
    \label{exa:sudokuMessagePassing}
    We iteratively solve a Sudoku puzzle by determining a possible value based on neighboring cells, rows and squares (using \theref{the:monotonicityPL}) and adding to our knowledge (using \theref{the:addingEntailed}).
    For example, consider the following $\sudokunum=2$ Sudoku puzzle, where a first entailment step uses only the knowledge of the rules and the \textcolor{\concolor}{blue} cells to determine the value $3$ in the first square:
    \begin{center}
        \begin{sudoku4x4}
            \matrix[sudokumatrix] (M) at (0,0) {
                1 & \ & \textcolor{\concolor}{3} & 2 \\
                \ & \textcolor{\concolor}{2} & \  & \  \\
                \ & \ & 4 & \ \\
                4 & 3 &  \ & \  \\
            };
            \draw[thick]([yshift=9.5pt,xshift=-0.6pt]M-1-2.east) -- ([yshift=-9.5pt,xshift=-0.6pt]M-4-2.east);
            \draw[thick]([xshift=-9.5pt,yshift=0.6pt]M-2-1.south) -- ([xshift=9.5pt,yshift=0.6pt]M-2-4.south);

            \node[anchor=center] (ist) at (1.75,0) {$=$};

            \matrix[sudokumatrix] (M) at (3.5,0) {
                1 & \ & 3 & 2 \\
                \textcolor{\probcolor}{3} & 2 & \  & \  \\
                \ & \ & 4 & \ \\
                4 & 3 &  \ & \  \\
            };
            \draw[thick]([yshift=9.5pt,xshift=-0.6pt]M-1-2.east) -- ([yshift=-9.5pt,xshift=-0.6pt]M-4-2.east);
            \draw[thick]([xshift=-9.5pt,yshift=0.6pt]M-2-1.south) -- ([xshift=9.5pt,yshift=0.6pt]M-2-4.south);

            \node[anchor=center] (ist) at (6.25,0) {$= \quad \ldots \quad =$};

            \matrix[sudokumatrix] (M) at (9,0) {
                1 & \textcolor{\probcolor}{4} & 3 & 2 \\
                \textcolor{\probcolor}{3} & 2 & \textcolor{\probcolor}{1} & \textcolor{\probcolor}{4}  \\
                \textcolor{\probcolor}{2} & \textcolor{\probcolor}{1} & 4 & \textcolor{\probcolor}{3} \\
                4 & 3 & \textcolor{\probcolor}{2} & \textcolor{\probcolor}{1}  \\
            };
            \draw[thick]([yshift=9.5pt,xshift=-0.6pt]M-1-2.east) -- ([yshift=-9.5pt,xshift=-0.6pt]M-4-2.east);
            \draw[thick]([xshift=-9.5pt,yshift=0.6pt]M-2-1.south) -- ([xshift=9.5pt,yshift=0.6pt]M-2-4.south);
        \end{sudoku4x4}
    \end{center}

    To illustrate the first reasoning step of assigning $\textcolor{\probcolor}{\catvariableof{0,1,0,0,2}}$ we make the following entailment steps applying \theref{the:monotonicityPL}.
    We also depict in \figref{fig:contractionPropagationSudoku} the corresponding messages in the Constraint Propagation Algorithm on the hypergraph $\graph^{\mathrm{Sudoku},n}$.
    \begin{itemize}
        \item From $\textcolor{\concolor}{\catvariableof{0,1,0,1,1}}$ (i.e. the $2$ in the cell $(0,1,0,1)$) and the Sudoku rule that at the cell $(0,1,0,1)$ exactly one number is assigned, we get
        \begin{align*}
            \left( \woneoplus_{i\in[\sudokunum^2]} \catvariableof{0,1,0,1,i} \right) \land \textcolor{\concolor}{\catvariableof{0,1,0,1,1}} \models \lnot\catvariableof{0,1,0,1,2} \, ,
        \end{align*}
        That is, that the number $3$ is not in the cell $(0,1,0,1)$.
        This entailment step is performed by three consecutive messages (see $\messagesymbol^{(0,[3])}$ in \figref{fig:contractionPropagationSudoku}) along the directions 
        \begin{align*}
        (\sedge,\redge)
            \in \big[&(\{\catvariableof{0,1,0,1,1}\},\{\catvariableof{0,1,0,1,1},\decvariableof{0,1,0,1,:}\}),
                (\{\catvariableof{0,1,0,1,1},\decvariableof{0,1,0,1,:}\},\{\catvariableof{0,1,0,1,2},\decvariableof{0,1,0,1,:}\}), \\
                &(\{\catvariableof{0,1,0,1,2},\decvariableof{0,1,0,1,:}\},\{\catvariableof{0,1,0,1,2},\decvariableof{0,:,0,:,2}\})\big] \, .
        \end{align*}
        Intuitively, the messages commmunicate to the square constraint $\decvariableof{0,:,0,:,2}$, that by the position constraint $\decvariableof{0,1,0,1,:}$ the variable $3$ cannot be assigned at $(0,1,0,1)$.
        \item From $\textcolor{\concolor}{\catvariableof{0,0,1,0,2}}$ (i.e. the $3$ in the cell $(0,0,1,0)$) and the Sudoku rule that at the row $(0,0)$ exactly one number is assigned, we get
        \begin{align*}
            \left( \woneoplus_{c0,c1\in[\sudokunum]} \catvariableof{0,0,c0,c1,2} \right) \land \textcolor{\concolor}{\catvariableof{0,0,1,0,2}} \models \lnot\catvariableof{0,0,0,0,2}\land \lnot\catvariableof{0,0,0,1,2} \, ,
        \end{align*}
        That is, that the number $3$ is neither in the cell $(0,0,0,0)$ nor in $(0,0,0,1)$.
        This entailment step is performed by five consecutive messages (see $\messagesymbol^{(1,[5])}$ in \figref{fig:contractionPropagationSudoku}) along the directions
        \begin{align*}
        (\sedge,\redge)
            \in \big[
            &(\{\catvariableof{0,0,1,0,2}\},\{\catvariableof{0,0,1,0,2},\decvariableof{0,0,:,:,2}\}),
                (\{\catvariableof{0,0,1,0,2},\decvariableof{0,0,:,:,2}\},\{\catvariableof{0,0,0,0,2},\decvariableof{0,0,:,:,2}\}), \\
                &(\{\catvariableof{0,0,1,0,2},\decvariableof{0,0,:,:,2}\},\{\catvariableof{0,0,0,1,2},\decvariableof{0,0,:,:,2}\}),
                (\{\catvariableof{0,0,0,0,2},\decvariableof{0,0,:,:,2}\},\{\catvariableof{0,0,0,0,2},\decvariableof{0,:,0,:,2}\}) \\
                &(\{\catvariableof{0,0,0,1,2},\decvariableof{0,0,:,:,2}\},\{\catvariableof{0,0,0,1,2},\decvariableof{0,:,0,:,2}\})
                \big] \, .
        \end{align*}
        The messages communicate that based on the decomposition cores of the constraint to the number $i=3$ in the first row $(r_0,r_1)=(0,0)$, that the number $3$ cannot be assigned at $(0,0,0,0)$ and $(0,0,0,1)$.
    \end{itemize}
    We add these formulas to our knowledge base (justified by \theref{the:addingEntailed}) and use the rule, that $3$ appears exactly once in the first square
    \begin{align*}
        &\left( \woneoplus_{r1,c1\in[\sudokunum]} \catvariableof{0,r1,0,c1,2} \right)
        \land (\lnot\catvariableof{0,1,0,1,2})
        \land (\lnot\catvariableof{0,0,0,0,2}\land \lnot\catvariableof{0,0,0,1,2})
        \models \textcolor{\probcolor}{\catvariableof{0,1,0,0,2}} \, .
    \end{align*}
    That is, we conclude that the number $3$ must be in the cell $(0,1,0,0)$, which information is also included in the updated knowledge base for further reasoning steps.
    This last entailment step is performed by four consecutive messages (see $\messagesymbol^{(2,[4])}$ in \figref{fig:contractionPropagationSudoku}) along the directions
    \begin{align*}
    (\sedge,\redge)
        \in \big[
            &(\{\catvariableof{0,1,0,1,2},\decvariableof{0,:,0,:,2}\},\{\catvariableof{0,1,0,0,2},\decvariableof{0,:,0,:,2}\}),
            (\{\catvariableof{0,0,0,1,2},\decvariableof{0,:,0,:,2}\},\{\catvariableof{0,1,0,0,2},\decvariableof{0,:,0,:,2}\}), \\
            &(\{\catvariableof{0,0,1,0,2},\decvariableof{0,:,0,:,2}\},\{\catvariableof{0,1,0,0,2},\decvariableof{0,:,0,:,2}\}),
            (\{\catvariableof{0,1,0,0,2},\decvariableof{0,:,0,:,2}\},\{\catvariableof{0,1,0,0,2}\})\big]
    \end{align*}
    The first three messages communicate, that the $3$ is not possible the positions $(0,1,0,1),(0,0,0,1)$ and $(0,0,1,0)$ and the fourth message concludes that the $3$ then has to be at position $(0,1,0,0)$.

    We now iteratively apply similar reasoning steps and store the entailed variables in \textcolor{\probcolor}{$E^{\mathrm{entailed}}$}, until we arrive at the right side of the above sketch.
    We conclude that
    \begin{align*}
        \sudokukbof{2} \land \left(\bigwedge_{(r_0,r_1,c_0,c_1,i)\in \sudokustartevidence} \catvariableof{r0,r1,c0,c1,i} \right)
        \models \textcolor{\probcolor}{\left(\bigwedge_{(r_0,r_1,c_0,c_1,i)\in E^{\mathrm{entailed}}} \catvariableof{r0,r1,c0,c1,i} \right)} \, .
    \end{align*}
    Since all Sudoku rules are satisfied in the final assignment and to each cell $(r_0,r_1,c_0,c_1)$ we found exactly one $i\in[\sudokunum^2]$ such that $(r_0,r_1,c_0,c_1,i)\in \sudokustartevidence\cup\textcolor{\probcolor}{E^{\mathrm{entailed}}}$, there is a unique solution of the puzzle and we conclude
    \begin{align*}
        &\sudokukbof{2} \land \left(\bigwedge_{(r_0,r_1,c_0,c_1,i)\in \sudokustartevidence} \catvariableof{r0,r1,c0,c1,i} \right) \\
        &\quad= \left(\bigwedge_{(r_0,r_1,c_0,c_1,i)\in \sudokustartevidence} \catvariableof{r0,r1,c0,c1,i} \right)
        \land \textcolor{\probcolor}{\left(\bigwedge_{(r_0,r_1,c_0,c_1,i)\in E^{\mathrm{entailed}}} \catvariableof{r0,r1,c0,c1,i} \right)} \, .
    \end{align*}
\end{example}

\begin{figure}[t]
    \begin{center}
        \begin{tikzpicture}[scale=0.35,thick]

            \draw (-1,-1) rectangle (1,1);
            \node[anchor=center] (A) at (0,0) {\corelabelsize $\hypercoreof{0}$};
            \draw (0,-1)--(0,-3) node[midway,right] {\colorlabelsize $\catvariableof{0,0,0,0,2}$};

            \draw (3,-1) rectangle (5,1);
            \node[anchor=center] (A) at (4,0) {\corelabelsize $\hypercoreof{1}$};
            \draw (4,-1)--(4,-3) node[midway,right] {\colorlabelsize $\catvariableof{0,0,0,1,2}$};

            \draw (7,-1) rectangle (9,1);
            \node[anchor=center] (A) at (8,0) {\corelabelsize $\hypercoreof{2}$};
            \draw (8,-1)--(8,-3); 

            \drawvariabledot{8}{-3}
            \draw (8,-3) -- (7.25,-3);
            \draw[\probcolor] (7,-5) rectangle (9,-7);
            \node[anchor=center, \probcolor] (A) at (8,-6) {\corelabelsize $\tbasis$};
            \draw[\probcolor] (8,-5)--(8,-3) node[midway,right] {\colorlabelsize $\catvariableof{0,1,0,0,2}$};

            \draw[\newmessagecolor,dashed, ->] (6.5,-1) to [bend right = 30] (6.5,-5);
            \node[\newmessagecolor,anchor=center] (A) at (5.25,-4) {\colorlabelsize $\messagesymbol^{(2,3)}$};

            \draw[\newmessagecolor,dashed, ->] (11,1.25) to [bend right = 30] (9,1.25);
            \node[\newmessagecolor,anchor=center] (A) at (9.75,2) {\colorlabelsize $\messagesymbol^{(2,0)}$};

            \draw[\newmessagecolor,dashed, ->] (5,1.5) to [bend right = -40] (6.8,1.1);
            \node[\newmessagecolor,anchor=center] (A) at (6,0.75) {\colorlabelsize $\messagesymbol^{(2,1)}$};

            \draw[\newmessagecolor,dashed, ->] (1,1.5) to [bend right = -40] (7,1.75);
            \node[\newmessagecolor,anchor=center] (A) at (3.5,2) {\colorlabelsize $\messagesymbol^{(2,2)}$};

            \draw (11,-1) rectangle (13,1);
            \node[anchor=center] (A) at (12,0) {\corelabelsize $\hypercoreof{3}$};
            \draw (12,-1)--(12,-2.5) node[midway,right] {\colorlabelsize $\catvariableof{0,1,0,1,2}$};

            \drawvariabledot{6}{4}
            \node[anchor=south] (text) at (6,4) {\colorlabelsize $\decvariableof{0,:,0,:,2}$};

            \draw (6,4) to[bend right= 20] (0,1);
            \draw (6,4) to[bend right= 10] (4,1);
            \draw (6,4) to[bend right= -10] (8,1);
            \draw (6,4) to[bend right= -20] (12,1);

            \draw (3,-7) rectangle (5,-5);
            \node[anchor=center] (A) at (4,-6) {\corelabelsize $\hypercoreof{1}$};
            \drawvariabledot{4}{-3}
            \draw (4,-3) -- (3.25,-3);
            \draw (4,-5)--(4,-1);

            \draw (-1,-7) rectangle (1,-5);
            \node[anchor=center] (A) at (0,-6) {\corelabelsize $\hypercoreof{0}$};
            \drawvariabledot{0}{-3}
            \draw (0,-3) -- (-0.75,-3);
            \draw (0,-5)--(0,-1);

            \draw (-5,-7) rectangle (-3,-5);
            \node[anchor=center] (A) at (-4,-6) {\corelabelsize $\hypercoreof{2}$};
            \draw (-4,-5)--(-4,-3);

            \drawvariabledot{-4}{-3}
            \draw (-4,-3) -- (-4.75,-3);

            \draw[\concolor] (-5,-1) rectangle (-3,1);
            \node[anchor=center, \concolor] (A) at (-4,0) {\corelabelsize $\tbasis$};
            \draw[\concolor] (-4,-1)--(-4,-3) node[midway,right] {\colorlabelsize $\catvariableof{0,0,1,0,2}$};

            \draw[\newmessagecolor,dashed, ->] (-5.5,-1) to [bend right = 30] (-5.5,-5);
            \node[\newmessagecolor,anchor=center] (A) at (-6.75,-2) {\colorlabelsize $\messagesymbol^{(1,0)}$};

            \draw[\newmessagecolor,dashed, ->] (-3,-7.25) to [bend right = 30] (-1,-7.25);
            \node[\newmessagecolor,anchor=center] (A) at (-2,-6.75) {\colorlabelsize $\messagesymbol^{(1,1)}$};

            \draw[\newmessagecolor,dashed, ->] (-3,-7.5) to [bend right = 40] (3,-7.25);
            \node[\newmessagecolor,anchor=center] (A) at (0,-9) {\colorlabelsize $\messagesymbol^{(1,2)}$};

            \draw[\newmessagecolor,dashed, <-] (-1.5,-1) to [bend right = 30] (-1.5,-5);
            \node[\newmessagecolor,anchor=center] (A) at (-2.75,-4) {\colorlabelsize $\messagesymbol^{(1,3)}$};

            \draw[\newmessagecolor,dashed, <-] (2.5,-1) to [bend right = 30] (2.5,-5);
            \node[\newmessagecolor,anchor=center] (A) at (1.25,-4) {\colorlabelsize $\messagesymbol^{(1,4)}$};

            \draw (-9,-7) rectangle (-7,-5);
            \node[anchor=center] (A) at (-8,-6) {\corelabelsize $\hypercoreof{3}$};
            \draw (-8,-5)--(-8,-3) node[midway,left] {\colorlabelsize $\catvariableof{0,0,1,1,2}$};

            \drawvariabledot{-2}{-10}
            \node[anchor=north] (text) at (-2,-10) {\colorlabelsize $\decvariableof{0,0,:,:,2}$};

            \draw (-2,-10) to[bend right= -20] (-8,-7);
            \draw (-2,-10) to[bend right= -10] (-4,-7);
            \draw (-2,-10) to[bend right= 10] (0,-7);
            \draw (-2,-10) to[bend right= 20] (4,-7);

            \draw (11,-7) rectangle (13,-5);
            \node[anchor=center] (A) at (12,-6) {\corelabelsize $\hypercoreof{2}$};
            \draw (12,-5)--(12,-2.5);

            \drawvariabledot{12}{-3}
            \draw (12,-3) -- (12.5,-3);

            \draw (15,-7) rectangle (17,-5);
            \node[anchor=center] (A) at (16,-6) {\corelabelsize $\hypercoreof{0}$};
            \draw (16,-5)--(16,-3) node[midway,right] {\colorlabelsize $\catvariableof{0,1,0,1,0}$};

            \draw (19,-7) rectangle (21,-5);
            \node[anchor=center] (A) at (20,-6) {\corelabelsize $\hypercoreof{1}$};
            \draw (20,-5)--(20,-3) node[midway,right] {\colorlabelsize $\catvariableof{0,1,0,1,1}$};

            \draw[\concolor] (19,-1) rectangle (21,1);
            \node[anchor=center, \concolor] (A) at (20,0) {\corelabelsize $\tbasis$};
            \draw[\concolor] (20,-1)--(20,-3) node[midway,right] {\colorlabelsize $\catvariableof{0,1,0,1,1}$};

            \drawvariabledot{20}{-3}
            \draw (20,-3) -- (20.5,-3);

            \draw[\newmessagecolor,dashed, ->] (18.5,-1) to [bend right = 30] (18.5,-5);
            \node[\newmessagecolor,anchor=center] (A) at (17.25,-2) {\colorlabelsize $\messagesymbol^{(0,0)}$};

            \draw[\newmessagecolor,dashed, ->] (19,-7.5) to [bend right = -40] (13,-7.25);
            \node[\newmessagecolor,anchor=center] (A) at (16,-9) {\colorlabelsize $\messagesymbol^{(0,1)}$};

            \draw[\newmessagecolor,dashed, <-] (13.5,-1) to [bend right = -30] (13.5,-5);
            \node[\newmessagecolor,anchor=center] (A) at (14.9,-4) {\colorlabelsize $\messagesymbol^{(0,2)}$};

            \draw (23,-7) rectangle (25,-5);
            \node[anchor=center] (A) at (24,-6) {\corelabelsize $\hypercoreof{3}$};
            \draw (24,-5)--(24,-3) node[midway,right] {\colorlabelsize $\catvariableof{0,1,0,1,3}$};

            \drawvariabledot{18}{-10}
            \node[anchor=north] (text) at (18,-10) {\colorlabelsize $\decvariableof{0,1,0,1,:}$};

            \draw (18,-10) to[bend right= -20] (12,-7);
            \draw (18,-10) to[bend right= -10] (16,-7);
            \draw (18,-10) to[bend right= 10] (20,-7);
            \draw (18,-10) to[bend right= 20] (24,-7);

        \end{tikzpicture}
    \end{center}
    \caption{
        The tensor network decomposition of $3$ out of $4\cdot2^2=64$ rules in the $2^2\times2^2$ Sudoku knowledge base (see \exaref{exa:sudokuDecomposition}),  namely to the number $3$ appearing once in the $(0,0)$-square (top), the number $3$ appearing once in the $(0,0)$-row (bottom left) and a unique number appearing at the $(0,1,0,1)$-position (bottom right).
        The evidence of the number $3$ already being assigned at the position $(0,0,1,0)$ is sketched by a basis vector $\textcolor{\concolor}{\tbasis}$ on the left side, and the number $2$ assigned at position $(0,1,0,1)$ analogously on the right side.
        During Constraint Propagation \algoref{alg:constraintPropagation} on the hypergraph of Sudoku rules and evidence (see \exaref{exa:sudokuMessagePassing}), this evidence is in three epochs of messages propagated to the constraints by partial entailment steps and imply that $\textcolor{\probcolor}{\catvariableof{0,1,0,0,2}}$ is true, that is that at the position $(0,1,0,0)$ the number $3$ needs to be assigned.
        We depict the messages between the cores by dashed lines labeled by $\messagesymbol^{(0,[3])},\messagesymbol^{(1,[5])}$ and $\messagesymbol^{(2,[4])}$ and provide further interpretation in \exaref{exa:sudokuMessagePassing}.
    }\label{fig:contractionPropagationSudoku}
\end{figure}
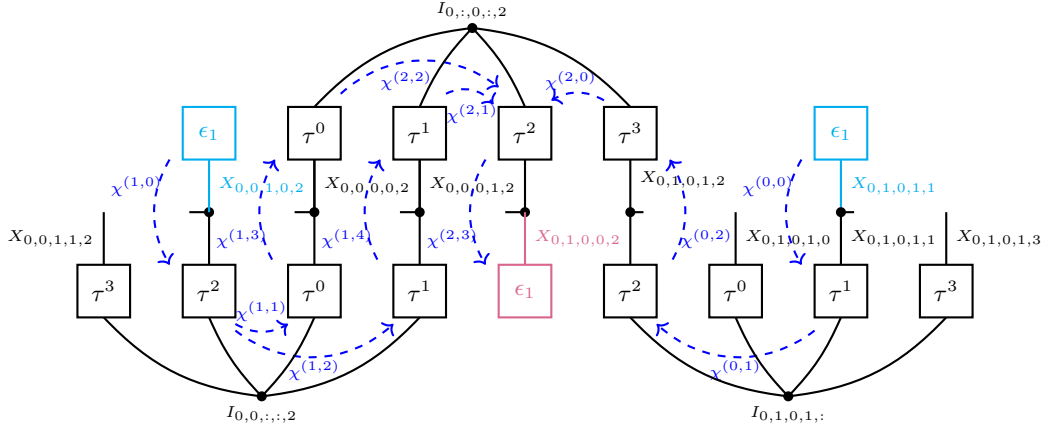

    \section{\HybridLogicNetworks{}}\label{sec:hln}

Let us now exploit the common formulation of logical formulas and probabilistic models in \CompActNets{} to define hybrid models that combine both aspects.
We call \CompActNets{} \HybridLogicNetworks{} in the special case of Boolean statistics $\sstat$ and elementary activations.

\subsection{Parametrization}

We first introduce \HybridLogicNetworks{}, which can be regarded as a unification of logical and probabilistic models.

\begin{definition}[\HybridLogicNetwork{} (HLN)]
    \label{def:hybridLogicNetwork}
    Given a Boolean statistic $\hlnstat$, we call any element of $\elrealizabledistsof{\hlnstat}$ a \HybridLogicNetwork{}.
    The extended canonical parameter set for $\hlnstat$ is the set
    \begin{align*}
        \hybridparamset\coloneqq
        \{\hardparam\wcols \hardlegset\subset[\seldim]\ncond \headindexof{\hardlegset}\in\bigtimes_{\selindex\in\hardlegset}[2]\} \times \parspace \, .
    \end{align*}
    For each \HybridLogicNetwork{} $\hlnwith$, we can associate a tuple $\hybridparam$ consisting of a subset $\hardlegset\subset[\seldim]$, a tuple $\headindexof{\hardlegset}\in\bigtimes_{\selindex\in\hardlegset}[2]$, and $\canparamwithin$ such that
    \begin{align*}
        \hlnwith
        = \normalizationof{\hlnstatccwith,\paracttensorwith}{\shortcatvariables}
    \end{align*}
    where the activation core is
    \begin{align*}
        \paracttensorwith = \contractionof{\softacttensorwith,\hardacttensorwith}{\headvariables} \, .
    \end{align*}
\end{definition}

We notice that the parametrization by $\hybridparamset$ is one-to-one for any non-vanishing elementary activation tensor up to a scalar factor.
Given an arbitrary elementary activation tensor $\bigotimes_{\selindexin}\acttensorlegwith$, we can always find a corresponding tuple in $\hybridparamset$ by choosing\footnote{Here $\nonzeroof{\cdot}$ is the indicator of non-zero entries acting coordinatewise and $\onesat{\headvariableof{\selindex}}$ is the vector $[1,1]^T$.}
\begin{align*}
    \hardlegset = \{\selindex \wcols \nonzeroof{\acttensorlegwith}\neq\onesat{\headvariableof{\selindex}}\} \, ,
\end{align*}
further for all $\selindex\in\hardlegset$
\begin{align*}
    \headindexof{\selindex}
    = \begin{cases}
          0 & \ifspace \nonzeroof{\acttensorlegwith} = \onehotmapofat{0}{\headvariableof{\selindex}} \\
          1 & \ifspace \nonzeroof{\acttensorlegwith} = \onehotmapofat{1}{\headvariableof{\selindex}} \,
    \end{cases}
\end{align*}
and a parameter vector $\canparamwithin$ defined for all $\selindexin$ as
\begin{align*}
    \canparamat{\indexedselvariable} =
    \begin{cases}
        0 & \ifspace \selindex\in\hardlegset \\
        \lnof{\frac{\acttensorlegat{\headvariableof{\selindex}=1}}{\acttensorlegat{\headvariableof{\selindex}=0}}}
        & \ifspace \selindex\notin\hardlegset \, .
    \end{cases}
\end{align*}
Then we have by construction that there is $\lambda>0$ with
\begin{align*}
    \bigotimes_{\selindexin}\acttensorlegwith
    = \lambda \cdot \paracttensorwith \, .
\end{align*}

Let us demonstrate the utility of \HybridLogicNetworks{} with an example from accounting.

\begin{example}[\HybridLogicNetwork{} for a toy accounting model]
    \label{exa:hlnAccountingRep}

    We consider a system of three variables $A1$ Account 1 is booked, $A2$ Account 2 is booked, $F$ a feature on an invoice.
    Assume the following two rules have to be respected:
    \begin{itemize}
        \item \textcolor{\concolor}{Exactly one account must be booked.}
        \item \textcolor{\probcolor}{If feature $\mathrm{F}$ is present on the invoice, the account $\mathrm{A1}$ is typically booked.}
    \end{itemize}
    We formalize this with the statistic
    \begin{align*}
        \hlnstat = (\catvariableof{A1} \oplus \catvariableof{A2}, \catvariableof{F}\Rightarrow \catvariableof{A1})\, .
    \end{align*}
    While the first formula is a hard feature, the second is soft since prone to exceptions.
    We parameterize the first output of the statistic with the hard parameters by setting the set of indices to be initialized with hard logic $A = \{0\}$ and the corresponding initialization $y_0 = 1$, meaning that the first output of the statistic has to be true for the input to have positive probability.
    The ``hard logic activation tensor'' should be indifferent to the second part of the statistic and only impose rules on the first part, leading to
    \begin{align*}
        \textcolor{\concolor}{\kappa^{(A,y_A)}[Y_0,Y_1]} =
        \onehotmapofat{\headindexof{0}}{\headvariableof{0}} \otimes
        \onesat{\headvariableof{1}}
        =\begin{bmatrix}
             0 \\
             1
        \end{bmatrix} \otimes \begin{bmatrix}
                                  1\\1
        \end{bmatrix}.
    \end{align*}
    Since the first feature is hard, the ``soft logic activation tensor'' should be invariant under the first coordinate of the canonical parameter and we set $\canparamat{\selvariable=0}=0$. The soft parameters are chosen as $\theta[L] = [0,\theta[L=1]]^\intercal$ to achieve
    \begin{align*}
        \textcolor{\probcolor}{\alpha^\theta[Y_0,Y_1]} &= \alpha^{0,0}\left[Y_0\right]\otimes \alpha^{1,\theta[L=1]}\left[Y_1\right] = \begin{bmatrix}
                                                                                                                                           1\\1
        \end{bmatrix}\otimes \begin{bmatrix}
                                 1\\ \expof{\canparamat{\selvariable=1}}
        \end{bmatrix}.
    \end{align*}
    The activation tensor of the hybrid network then has the form
    \begin{align*}
        \paracttensor[{\headvariableof{0},\headvariableof{1}}]
        = \begin{bmatrix}
              0 \\ 1
        \end{bmatrix}
        \otimes
        \begin{bmatrix}
            1 \\ \expof{\canparamat{\selvariable=1}}
        \end{bmatrix}.
    \end{align*}
    We get the following tensor network representation of the \HybridLogicNetwork{} representing the toy accounting example before normalization to a distribution
    \begin{center}
        \begin{tikzpicture}[scale=0.35, thick, yscale=-1, xscale=-1] 

            \begin{scope}[shift={(14,-3)}]
                \draw (-4.5,0) rectangle (10.5,2);
                \node[anchor=center] (A) at (3,1) {\corelabelsize $\probof{(\catvariableof{A1}\oplus \catvariableof{A2}, X_F\Rightarrow \catvariableof{A1}),(\{0\},(1),(0,\theta))}$};
                \draw[-<-]  (-2,2)--(-2,4) node[midway,left] {\colorlabelsize $\catvariableof{A1}$};
                \draw[-<-]  (3,2)--(3,4) node[midway,left] {\colorlabelsize $\catvariableof{A2}$};
                \draw[-<-]  (8,2)--(8,4) node[midway,left] {\colorlabelsize $\catvariableof{F}$};

                \node[anchor=center] (A) at (-7.1,1) {$= \,\,\frac{1}{\partitionfunctionof{\canparam}}\,\cdot$};
            \end{scope}

            \draw[->-]  (3,1.75)--(3,-1) node[midway,left] {\colorlabelsize $\catvariableof{A2}$};

            \draw[->-] (-1.5,1) to[bend right=20] (-3,-1);
            \draw[->-] (-1.5,1) to[bend right=-20] (0,-1);
            \drawvariabledot{-1.5}{1}
            \draw (-1.5,1)--(-1.5,1.75)  node[midway,left] {\colorlabelsize $\catvariableof{A1}$};

            \draw[->-] (-6,1.75)--(-6,-1) node[midway,left] {\colorlabelsize $\catvariableof{F}$};

            \draw (-7,-1) rectangle (-2, -3);
            \node[anchor=center] (text) at (-4.5,-2) {$\bencodingof{\Rightarrow}$};
            \draw[->-] (-4.5,-3) --(-4.5,-5.5) node[midway,left]{\colorlabelsize $\headvariableof{F \Rightarrow A1}$};

            \draw[fill, \probcolor] (-4.5,-5.5) circle (\dotsize);
            \draw[\probcolor] (-4.5,-5.5) -- (-5,-5.5);
            \draw[\probcolor]  (-13, -3.5) rectangle (-5, -6.5);
            \node[anchor=center,\probcolor] (text) at (-9,-5) {$\begin{bmatrix}
                                                                    1 \\
                                                                    \expof{\canparamat{\selvariable=1}}
            \end{bmatrix}$};

            \draw (-1,-1) rectangle (4, -3);
            \node[anchor=center] (text) at (1.5,-2) {$\bencodingof{\oplus}$};
            \draw[->-] (1.5,-3) --(1.5,-5.5) node[midway,left]{\colorlabelsize $\headvariableof{A1 \oplus A2}$};

            \draw[fill, \concolor] (1.5,-5.5) circle (\dotsize);
            \draw[\concolor] (1.5,-5.5) -- (1,-5.5);
            \draw[\concolor]  (-1, -3.5) rectangle (1, -6.5);
            \node[anchor=center,\concolor] (text) at (0,-5) {$\begin{bmatrix}
                                                                  0 \\
                                                                  1
            \end{bmatrix}$};

        \end{tikzpicture}
    \end{center}
    The resulting \HybridLogicNetwork{} is a tensor $\probofat{\hlnstat,\hybridparam}{\catvariableof{A1},\catvariableof{A2},\catvariableof{F}}$ of order $3$.
    With $Y_{F\Rightarrow A_1}=1$ for $F=0$ and any $A_1$ it has the coordinates
    \begin{center} 
        \begin{tikzpicture}[scale=1.75]
            \node (A) at (-4,0) {\corelabelsize $\probofat{(\catvariableof{A1}\oplus\catvariableof{A2},\catvariableof{F}\Rightarrow\catvariableof{A1}),\big(\{0\},(1),(0,\theta)\big)}{\catvariableof{A1},\catvariableof{A2},\catvariableof{F}} \,\, = \,\, \frac{1}{1+3 \cdot \expof{\canparam}} \, \cdot $};
            \node (A) at (0,0) {
                $\begin{bmatrix}
                     0 & \exp[\canparam] \\
                     \exp[\canparam] & 0
                \end{bmatrix}$
            };
            \node (A) at (1.25,0.3) {
                $\begin{bmatrix}
                     0 & 1 \\
                     \exp[\canparam] & 0
                \end{bmatrix}$
            };
            \draw[->,dashed] (-0.325,0.5) node[below] {\tiny $0$} -- (0.325,0.5) node [midway, above] {\tiny $\catvariableof{A2}$} node[below] {\tiny $1$};
            \draw[<-,dashed] (-0.9,-0.2) node[right] {\tiny $1$} -- (-0.9,0.2) node [midway, left] {\tiny $\catvariableof{A1}$} node[right] {\tiny $0$};
            \draw[->,dashed] (0,-0.5) node[above] {\tiny $0$} -- (1.25,-0.2) node [midway, below] {\tiny $X_F$} node[above] {\tiny $1$};
        \end{tikzpicture}
    \end{center}
\end{example}

\subsection{Parameter estimation in \HybridLogicNetworks{}}\label{sec:paramEst}

Let us now briefly discuss how \HybridLogicNetworks{} can be trained on data based on likelihood maximization.
Given a dataset $\dataset$ consisting of $\datanum$ independent and identically distributed samples from an unknown distribution, we want to find a \HybridLogicNetwork{} $\hlnwith$ with a statistic $\sstat=(\formulaof{0},\dots,\formulaof{\seldim-1})$ that minimizes the negative log likelihood
\begin{align*}
    \lossof{\hybridparam} \coloneqq -\frac{1}{\datanum} \sum_{\datindexin} \lnof{\probofat{\hlnparameters}{\shortcatvariables=\shortcatindices^{\datindex}}}
    \, .
\end{align*}
We can rewrite the loss using the empirical mean vector $\datameanwith\in\parspace$, which is defined for $\selindexin$ as
\begin{align*}
    \datameanat{\indexedselvariable}
    = \frac{1}{\datanum} \sum_{\datindexin} \formulaofat{\selindex}{\shortcatvariables=\shortcatindices^{\datindex}} \, ,
\end{align*}
by
\begin{align*}
    \lossof{\hybridparam} =
    \contraction{\datameanwith,\canparamwith} - \lnof{\contraction{\paracttensorwith,\bencsstatwith}} \, .
\end{align*}
Since $\hardparam$ influences only the second term, the best hard parameters can be found by
\begin{align*}
    \hardlegset = \{\selindex \wcols \datameanat{\indexedselvariable}\in\ozset\} \andspace
    \headindexof{\selindex} = \datameanat{\indexedselvariable} \quad \text{for} \quad \selindex\in\variableset \, .
\end{align*}
We further optimize the coordinates $\selindex\in[\seldim]\setexcept{\hardlegset}$ of $\canparamwithin$ alternately by the coordinate descent steps
\begin{align*}
    \difofwrt{\lossof{\hybridparam}}{\canparamat{\indexedselvariable}} = 0
    \Leftrightarrow
    \canparamat{\indexedselvariable}
    = \lnof{
        \frac{\meanparamat{\indexedselvariable}}{(1-\meanparamat{\indexedselvariable})}
        \cdot \frac{\hypercoreat{\headvariableof{\selindex}=0}}{\hypercoreat{\headvariableof{\selindex}=1}}
    } \, .
\end{align*}
where
\begin{align*}
    \hypercoreat{\headvariableof{\selindex}}
    = \contractionof{\{\bencodingof{\formulaof{\secselindex}} \wcols \secselindex\in[\seldim]\}
        \cup\{\softactsymbolof{\secselindex,\canparam} \wcols \secselindex\in[\seldim]\ncond\secselindex\neq\selindex\}
        \cup\{\basemeasure\}}{\headvariableof{\selindex}} \, .
\end{align*}
Based on an interpretation of the coordinate descent steps as matching steps for the mean parameters or moments to $\formulaof{\selindex}$, we call this method \emph{alternating moment matching} for \HybridLogicNetworks{} and provide pseudocode for it it in \algoref{alg:AMM_HLN}.
We notice that, during the coordinate descent steps, computing the marginal probability of the variable $\headvariableof{\selindex}$ with respect to the current network parameters is required.
This is the computational bottleneck of the algorithm and can be approached by various approximate inference methods, e.g., variational inference (see for example the CAMEL method \cite{ganapathi_constrained_2008}).

\begin{algorithm}[hbt!]
    \caption{Alternating Moment Matching for \HybridLogicNetworks{}}\label{alg:AMM_HLN}
    \begin{algorithmic}
        \Require Mean parameter $\datameanwith$
        \Ensure Parameters $\hybridparam$ for the approximating HLN $\expdist$ 
        \iosepline
        \State Set
        \begin{align*}
            \hardlegset = \Big\{ \selindex \wcols \selindexin \ncond \meanparamat{\indexedselvariable}\in\{0,1\} \Big\}
        \end{align*}
        and a tuple $\headindexof{\hardlegset}$ with $\headindexof{\selindex}=\meanparamat{\indexedselvariable}$ for $\selindex\in\hardlegset$.
        \State Set $\canparamwith=\zerosat{\selvariable}$
        \While{Convergence criterion is not met}
            \ForAll{$\selindex\in[\seldim]\setexcept{\hardlegset}$}
                \State Compute
                \begin{align*}
                    \hypercoreat{\headvariableof{\selindex}}
                    = \contractionof{\{\bencodingof{\formulaof{\secselindex}} \wcols \secselindex\in[\seldim]\}
                        \cup\{\softactsymbolof{\secselindex,\canparam} \wcols \secselindex\in[\seldim]\ncond\secselindex\neq\selindex\}
                        \cup\{\basemeasure\}}{\headvariableof{\selindex}}
                \end{align*}
                \State Set
                \begin{align*}
                    \canparamat{\indexedselvariable} = \lnof{
                        \frac{\meanparamat{\indexedselvariable}}{(1-\meanparamat{\indexedselvariable})}
                        \cdot \frac{\hypercoreat{\headvariableof{\selindex}=0}}{\hypercoreat{\headvariableof{\selindex}=1}}
                    }
                \end{align*}
            \EndFor
        \EndWhile
        \State \Return $(\hardlegset,\headindexof{\hardlegset},\canparamwith)$
    \end{algorithmic}
\end{algorithm}

It can be shown that the algorithm converges if and only if there is a \HybridLogicNetwork{} matching the empirical moments of the data.
For more details we refer to~\cite[Chapter~9]{goessmann_tensor-network_2025}.

\begin{example}[Continuation of \exaref{exa:hlnAccountingRep}]\label{exa:hlnAccountingAMM}
    Recall the statistic of \exaref{exa:hlnAccountingRep} and consider a dataset of $\datanum=20$ states summarized in the frequency table:
    \begin{center}
        \begin{tabular}{|c|ccc|}
            \hline
            \textbf{Frequency in Dataset} & $\catindexof{A1}$ & $\catindexof{A2}$ & $\catindexof{F}$ \\
            \hline
            0                             & 0                    & 0                    & 0                   \\
            0                             & 0                    & 0                    & 1                   \\
            7                             & 0                    & 1                    & 0                   \\
            2                             & 0                    & 1                    & 1                   \\
            1                             & 1                    & 0                    & 0                   \\
            10                            & 1                    & 0                    & 1                   \\
            0                             & 1                    & 1                    & 0                   \\
            0                             & 1                    & 1                    & 1                   \\
            \hline
        \end{tabular}
    \end{center}
    We then have for the satisfaction rates of $\formulaof{0}=\catvariableof{A1}\oplus\catvariableof{A2}$ and $\formulaof{1}=\catvariableof{F}\Rightarrow\catvariableof{A1}$ that
    \begin{align*}
        \datameanat{\selvariable=0} = \frac{20}{20} = 1 \andspace
        \datameanat{\selvariable=1} = \frac{7+1+10}{20} = 0.9 \, .
    \end{align*}
    \algoref{alg:AMM_HLN} yields a reasonable convergence criterion choice (such as finite iterations or convergence of $\canparamat{\selvariable}$)
    \begin{align*}
        \hardlegset = \{0\} \quad, \quad \headindexof{\hardlegset} = 1 \andspace
        \canparamat{\selvariable} =
        \begin{bmatrix}
            0 \\
            \lnof{(\frac{0.9}{0.1})\cdot(\frac{1}{3})}
        \end{bmatrix}
        =
        \begin{bmatrix}
            0 \\
            \lnof{3}
        \end{bmatrix}
        \approx
        \begin{bmatrix}
            0 \\
            1.098612
        \end{bmatrix} \, .
    \end{align*}
    To derive this, we notice that \algoref{alg:AMM_HLN} treats formula $\formulaof{0}$ as a hard constraint and assigns $\hardlegset = \{0\}$ and $\headindexof{\hardlegset} = 1$.
    In the \whileSymbol{} loop we then have for the formula $\formulaof{1}$
    \begin{align*}
        \hypercoreat{\headvariableof{1}}
        = \contractionof{\tbasisat{\headvariableof{0}},\bencodingofat{\formulaof{0}}{\headvariableof{0},\catvariableof{F},\catvariableof{A1},\catvariableof{A2}},
            \bencodingofat{\formulaof{1}}{\headvariableof{1},\catvariableof{F},\catvariableof{A1},\catvariableof{A2}}}{\headvariableof{1}}
        = \begin{bmatrix}
              1 \\
              3
        \end{bmatrix}
    \end{align*}
    since $\formulaof{0}$ has $4$ models, of which $3$ are also models of $\formulaof{1}$ and $1$ is instead a model of $\lnot\formulaof{1}$.
    Notice, that the tensor $\hypercoreat{\headvariableof{1}}$ will not change in any further iteration of the \whileSymbol{} and the parameter $\canparamat{\selvariable=1}$ will therefore stay constant until the termination of the algorithm.
\end{example}

\subsection{Entailment by \HybridLogicNetworks{}} 

Let us now demonstrate a further use of our unified treatment of probabilistic and logical models by investigating a generalized concept of entailment.
Entailment can be generalized to probabilistic models by deciding whether a propositional formula is always satisfied given a probabilistic model.

\begin{theorem}\label{the:hybridEntailment}
    Let $\probofat{\hlnparameters}{\shortcatvariables}$ be a \HybridLogicNetwork{} and $\secexformulaat{\shortcatvariables}$ a propositional formula.
    Then $\probof{\hlnparameters}$ probabilistically entails $\secexformula$, that is,
    \begin{align*}
        \contraction{\probofat{\hlnparameters}{\shortcatvariables},\secexformulaat{\shortcatvariables}} = 1 \, ,
    \end{align*}
    if and only if
    \begin{align*}
        \hlnformula \models \secexformula \, ,
    \end{align*}
    where
    \begin{align*}
        \hlnformulawith =
        \left(\bigwedge_{\selindex\in\hardlegset\wcols\headindexof{\selindex}=1} \enumformulaat{\shortcatvariables}\right)
        \land
        \left(\bigwedge_{\selindex\in\hardlegset\wcols\headindexof{\selindex}=0} \lnot\enumformulaat{\shortcatvariables}\right).
    \end{align*} 
\end{theorem}
\begin{proof}
    We have
    \begin{align*}
        \contraction{\probofat{\hlnparameters}{\shortcatvariables},\secexformulaat{\shortcatvariables}} = 1
    \end{align*}
    if and only if 
    \begin{align*}
        \contraction{\probofat{\hlnparameters}{\shortcatvariables},\secexformulaat{\shortcatvariables}-\onesat{\shortcatvariables}} = 0
    \end{align*}
    which is equal to 
    \begin{align*}
        \contraction{\probofat{\hlnparameters}{\shortcatvariables},\lnot\secexformulaat{\shortcatvariables}} = 0 \, .
    \end{align*}
    Since $\probofat{\hlnparameters}{\shortcatvariables}$ is non-negative this is equivalent to
    \begin{align*}
        \contraction{\nonzeroof{\probofat{\hlnparameters}{\shortcatvariables}},\lnot\secexformulaat{\shortcatvariables}} = 0 \, .
    \end{align*}
    We use that
    \begin{align*}
        \nonzeroof{\probofat{\hlnparameters}{\shortcatvariables}} 
        = \hlnformulawith 
    \end{align*}
    and get that this is further equivalent to
    \begin{align*}
        \contraction{\hlnformulawith ,\lnot\secexformulaat{\shortcatvariables}} = 0 \, ,
    \end{align*}
    which is by \defref{def:logicalEntailment} $\hlnformula \models \secexformula$.
\end{proof}

\begin{example}[Continuation of \exaref{exa:hlnAccountingAMM}]
    Consider again the \HybridLogicNetwork{}
    \begin{align*}
        \probofat{(\catvariableof{A1}\oplus\catvariableof{A2},\catvariableof{F}\Rightarrow\catvariableof{A1}),\big(\{0\},(1),(0,\lnof{3})\big)}{\catvariableof{A1},\catvariableof{A2},\catvariableof{F}}
    \end{align*}
    from \exaref{exa:hlnAccountingAMM} and assume we want to decide the probabilistic entailment of the formula
    \begin{align*}
        \secexformulaat{\catvariableof{A1},\catvariableof{A2},\catvariableof{F}}
        = \lnot\catvariableof{A1} \lor \lnot\catvariableof{A2} \lor \lnot\catvariableof{F} \, ,
    \end{align*}
    which has all states but $(1,1,1)$ as a model (and is therefore refered to as a maxterm).
    Using \theref{the:hybridEntailment} we have that
    \begin{align*}
        \contraction{\probofat{(\catvariableof{A1}\oplus\catvariableof{A2},\catvariableof{F}\Rightarrow\catvariableof{A1}),\big(\{0\},(1),(0,\lnof{3})\big)}{\catvariableof{A1},\catvariableof{A2},\catvariableof{F}},\secexformulaat{\catvariableof{A1},\catvariableof{A2},\catvariableof{F}}} =1
    \end{align*}
    if and only if $\catvariableof{A1}\oplus \catvariableof{A2} \models \lnot\catvariableof{A1} \lor \lnot\catvariableof{A2} \lor \lnot\catvariableof{F}$.
    By \defref{def:logicalEntailment}, this entailment holds since by the De-Morgan rule
    \begin{align*}
        \contraction{\catvariableof{A1}\oplus \catvariableof{A2},\lnot \left(\lnot\catvariableof{A1} \lor \lnot\catvariableof{A2} \lor \lnot\catvariableof{F}\right)}
        &= \contraction{\catvariableof{A1}\oplus \catvariableof{A2},\catvariableof{A1},\catvariableof{A2},\catvariableof{F}} \\
        &= \contraction{\catvariableof{F}} \cdot \contraction{\catvariableof{A1}\oplus \catvariableof{A2},\catvariableof{A1},\catvariableof{A2}} \\
        &= 0 \, .
    \end{align*}
    We thus conclude, that $\secexformula$ is probabilistically entailed by $\probof{(\catvariableof{A1}\oplus\catvariableof{A2},\catvariableof{F}\Rightarrow \catvariableof{A1}),\big(\{0\},(1),(0,\lnof{3})\big)}$.
\end{example}
    \section{Implementation in the \python{} library \tnreason{}}\label{sec:implementation}


The concepts presented in this paper have been implemented in the open source \python{} library \tnreason{}\footnote{\tnreason{} is available in version 2.0.0 at \url{pypi.org/tnreason} and maintained at \url{github.com/tnreason/tnreason-py}.}.
In this section, we explain the basic design and functionality of this library and draw close connections to the theoretical exposition in the previous sections. In particular, Appendix~\ref{sec:algExaImplementation} provides detailed implementations of the algorithms and examples in this work.

\subsection{Architecture}

The package consists of four subpackages and three layers of abstraction:
\begin{center}

\begin{tikzpicture}[scale=0.35]
    \draw[dashed] (-30,15) -- (12,15) -- (12,-3) -- (-30,-3) -- (-30,15);

    \draw (-10,10) rectangle (10,14);
    \node [anchor=center] at (0,12) {\spapplication{}};

    \node [anchor=center] at (-20,12) {\layerthreespec};
    \draw[dashed] (-30,9) -- (12,9);
    \node [anchor=center] at (-20,6) {\layertwospec};

    \draw[->-] (6,10) -- (6,8);
    \draw (2,4) rectangle (10,8);
    \node [anchor=center] at (6,6) {\spreasoning{}};
    \draw[->-] (6,4) -- (6,2);

    \draw[->-] (-6,10) -- (-6,8);
    \draw (-10,4) rectangle (-2,8);
    \node [anchor=center] at (-6,6) {\sprepresentation{}};
    \draw[->-] (-6,4) -- (-6,2);

    \draw[dashed] (-30,3) -- (12,3);
    \node [anchor=center] at (-20,0) {\layeronespec};

    \draw (-10,-2) rectangle (10,2);
    \node [anchor=center] at (0,0) {\spengine{}};
\end{tikzpicture}
\end{center}
In the subpackage \tnreason{}.\spengine{} we implement tensors, tensor networks, contractions, and normalizations.
In the subpackage \tnreason{}.\sprepresentation{} the basic tensor encoding schemes such as basis encodings are available.
In the subpackage \tnreason{}.\spreasoning{} we implement reasoning algorithms, such as generalizations of the message passing algorithms presented in \algoref{alg:treeBeliefPropagation}, \algoref{alg:constraintPropagation}, and \algoref{alg:AMM_HLN}.
In the subpackage \tnreason{}.\spapplication{} one can construct tensor network encodings of propositional formulas and datasets.

\subsection{Basic usage}

We demonstrate the basic usage of the \tnreason{} package with the implementation of Example~\ref{exa:propFormulaCoordinatewise}.
We first install the package and import it by
\begin{lstlisting}[language=Python]
from tnreason import engine, application
\end{lstlisting}
Keeping \defref{def:tensor} in mind, the tensor instances in \lstinline[language=Python]!shape! and \lstinline[language=Python]!colors! arguments are \lstinline[language=Python]!list! instances specifying the \lstinline[language=Python]!int! dimension $\catdimof{\catenumerator}$ and a \lstinline[language=Python]!str! identifier for $\catvariableof{\catenumerator}$.
The formula in \exaref{exa:propFormulaCoordinatewise} is a sum of the one-hot encodings of its three models (see \exaref{exa:propFormulaBasCP}) and is created by
\begin{lstlisting}[language=Python]
formula = engine.create_from_slice_iterator(shape=[2,2,2], colors=["X_0","X_1","X_2"], sliceIterator=[(1,{"X_0":0,"X_1":1,"X2":0}), (1,{"X_0":1,"X_1":0,"X2":0}), (1,{"X_0":1,"X_1":1,"X2":0})])
\end{lstlisting}
The slice iterator is an iterator over tuples \lstinline[language=Python]!(val,posDict)!, which specifies elementary tensors to be summed.
The \lstinline[language=Python]!posDict! are \lstinline[language=Python]!dict! instances, where the keys are the \lstinline[language=Python]!str! tensor colors and the values are \lstinline[language=Python]!int!.
Each \lstinline[language=Python]!posDict! collects leg vectors of the corresponding elementary tensor that are not trivial. These leg vectors are the basis vectors enumerated by the corresponding \lstinline[language=Python]!int! value.

Single tensor coordinates can be retrieved by indexing with a \lstinline[language=Python]!posDict!.
We can, for example, check whether \lstinline[language=Python]!{"X_0":0,"X_1":1,"X_2":0}! is a model:
\begin{lstlisting}[language=Python]
assert formula[{"X_0":0,"X_1":1,"X_2"}] == 1
\end{lstlisting}
By default the tensor is created as a \pythoninline{engine.NumpyCore} instance, where coordinates are stored as instances of \pythoninline{numpy.array}.
Further core types exploiting different sparsity principles can be chosen by the argument \pythoninline{coreType}, see \cite[Appendix~A]{goessmann_tensor-network_2025}.

Following \defref{def:tensorNetwork}, tensor networks are implemented as tensor valued \pythoninline{dict} instances with \pythoninline{str} keys.
For example a tensor network is created from the propositional syntax of the above formula (see \exaref{exa:propFormulaHeadSym}): 
\begin{lstlisting}[language=Python]
fDecomp = application.create_cores_to_expressionsDict({"f0": ["and",["or","X_0","X_1"],["not","X_2"]]})
\end{lstlisting}
Here we apply a nested-list description of decomposition hypergraphs (see \defref{def:decompositionHypergraph}) with a specification of the logical connectives in the first position of the list (by \pythoninline{"and","or","not"} we refer to the connectives $\land,\lor,\lnot$).
Equivalently, we can exploit the $\land$ symmetry and create it by multiple formulas:
\begin{lstlisting}[language=Python]
fDecomp = application.create_cores_to_expressionsDict({"f0": ["or","X_0","X_1"], "f1": ["not","X_2"]})
\end{lstlisting}
A depiction of the underlying hypergraph as a factor graph, which highlights edges as blue blocks and nodes as red blocks, can be created with \pythoninline{engine.draw_factor_graph(fDecomp)} (see \figref{fig:factorGraphPlt}).
Single tensors can be obtained by contracting a tensor network while specifying the open variables (for an explanation of the suffixes, see \figref{fig:factorGraphPlt}), for example:
\begin{lstlisting}[language=Python]
contracted = engine.contract(fDecomp, openColors=["X_0_dV","X_1_dV","X_2_dV"])
\end{lstlisting}
By default the contractions are performed using \pythoninline{numpy.einsum} and further execution schemes can be selected with the argument \pythoninline{contractionMethod}, see \cite[Appendix~A]{goessmann_tensor-network_2025}.

\begin{figure}
    \begin{center}
        \includegraphics[width=15cm]{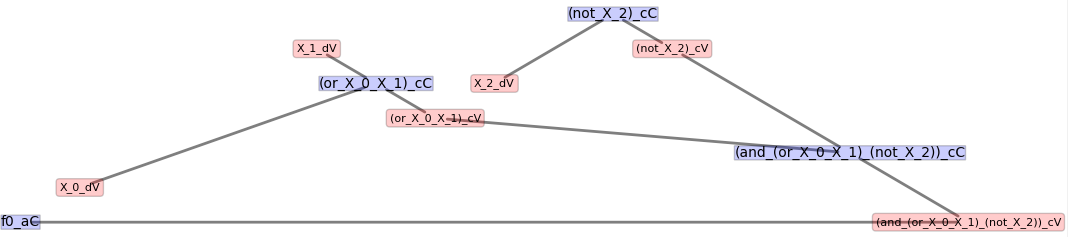}
    \end{center}
    \caption{
        Factor graph highlighting a tensor network decomposition of the syntactic decomposition of the propositional formula of \exaref{exa:propFormulaHeadSym}.
        Blue blocks highlight hyperedges carrying tensors and red blocks highlight variables.
        The tensor label suffices \pythoninline{"_cC"} and \pythoninline{"_aC"} indicate whether the tensor is part of the computation network or the activation network.
        The variable label suffices \pythoninline{"_dV"} and \pythoninline{"_cV"} indicate whether the variable is distributed or computed and therefore auxiliary.
        This graph has been generated with the method \pythoninline{tnreason.engine.draw_factor_graph} of \tnreason{}.
    }\label{fig:factorGraphPlt}
\end{figure}

    \section{Conclusion \& outlook}\label{sec:conclusion}

This work developed a tensor network formalism to capture in a unifying way the main concepts of AI, which build the core of the probabilistic, neural and logical approaches.
We introduced \ComputationActivationNetworks{} (\CompActNets{}) as a generic architecture to represent classes of propositional knowledge bases, graphical models and more generic exponential families.
Moreover, we demonstrate the representation and training of hybrid models combining logical and probabilistic aspects, illustrating that \CompActNets{} represent a powerful, versatile and mathematically grounded framework for Neuro-Symbolic AI.

We have shown that model inference such as the calculating marginal distributions and deciding entailment correspond with tensor network contractions.
To efficiently perform these inferences, we presented message passing schemes, which have been shown to be exact in specific cases.
In general, however, the efficient computation of contractions is not possible, since they are related to the $\mathrm{NP}$-hardness of probabilistic inferences in graphical models (see \cite{koller_probabilistic_2009}) and of logical reasoning (see \cite{russell_artificial_2021}).
In cases where exact inference is not feasible, the derivation of error bounds for approximate inference schemes on \CompActNets{} is an interesting direction for future research.

Further approximation schemes to overcome this bottleneck are summarized under the umbrella of variational inference (see \cite{wainwright_graphical_2008}), such as generic expectation-propagation methods or mean field methods.
While these schemes are developed either for graphical models or more general exponential families, we plan to derive similar methods for more general \CompActNets{}, such as \HybridLogicNetworks{}.
Further frequently applied schemes are particle-based inference schemes such as Gibbs sampling.

The integration of symbolic and neural methods is an active research area (see \cite{colelough_neuro-symbolic_2024} for a systematic review).
The \CompActNets{} framework enables both, symbolic logical as well as probabilistic models, but enables also the representation of generic functions.
\CompActNets{} based on architectures combining symbolically verbalizable and more generic neural parts are thus a promising direction for Neuro-Symbolic AI.

The \CompActNets{} framework offers an immediate practical application as a verifiable reasoning engine for AI agents in high-stakes domains such as regulatory compliance, clinical decision support, accounting, process planning and security.
By leveraging the framework's inherent flexibility, Large Language Models (see \cite{vaswani_attention_2017}) can be adapted to function as semantic translators that dynamically construct problem-specific tensor networks in the form of \CompActNets{} from natural language descriptions, effectively treating the reasoning engine as an external tool.
This approach mitigates the hallucination risks of probabilistic models by delegating complex logical execution to the exact linear algebra of the tensor network, ensuring that the inference process is both rigorous and reproducible.
Consequently, this synergy enables the deployment of reliable AI systems where the intuitive power of the Large Language Model is grounded by the explainable, instance-adaptive topology of the \CompActNets{}.

    \section*{Acknowledgements}
    AG and ME acknowledge funding from the German Federal Ministry of Education and Research (BMBF), grant number FKZ 13N17160, "Verbundprojekt: Quantum Read-Once-Memory - Verwandlung von klassischen Daten zu Quantenzuständen - Teilvorhaben: Quanten{\allowbreak}schaltkreis-Optimierung von Quantenzuständen durch Tensor-Netzwerke (QOQ-tn)".
JS and ME acknowledge funding from the Deutsche Forschungsgemeinschaft (DFG, German Research Foundation) under Germany´s Excellence Strategy – The Berlin Mathematics Research Center MATH+ (EXC-2046/2, project ID: 390685689, Project PaA-7).
MF acknowledges funding from Einstein Research Unit on Quantum Devices.

    \vskip 0.2in
    \bibliographystyle{plainnat}
    \bibliography{references.bib}

    \appendix

    \section{Proof of the Factorization Theorems}\label{sec:proofFactorizationTheorems}

We now provide proofs for the factorization theorems stated in \secref{sec:probPar}.
These proofs are classically known (see e.g. \cite{koller_probabilistic_2009} for Hammersley-Clifford and \cite{casella_statistical_2001} for Fisher-Neyman).
They are here provided in our tensor networks notation and for hypergraphs for completeness. 

\subsection{Hammersley-Clifford}

Different to the original statement (see \cite{clifford_markov_1971}), we here proof the analogous statement for hypergraphs, where we have to demand the property of clique-capturing defined in \defref{def:ccHypergraph}.
We start with showing the following Lemmata to be exploited in the proof.

\begin{lemma}
    \label{the:contractionFactorization}
    Let $\hypercoreat{\catvariableof{\nodes}}$ be a positive tensor. 
    Then we have for any index $\seccatindexof{\nodes}$
    \begin{align*}
        \hypercoreat{\catvariableof{\nodes}}
        = \contractionof{
            \big(\contractionof{\hypercore}{\catvariableof{\nodes\setexcept{\thirdnodes}},\catvariableof{\thirdnodes}=\seccatindexof{\thirdnodes}}\big)^{(-1)^{\cardof{\secnodes}-\cardof{\thirdnodes}}} \wcols \thirdnodes \subset \secnodes \subset \nodes
        }{\catvariableof{\nodes}} \, ,
    \end{align*}
    where the exponentiation is performed coordinatewise and positivity of $\hypercore$ ensures the well-definedness.
\end{lemma}
\begin{proof}
    It suffices to show, that for an arbitrary index $\catindexof{\nodes}$ we have
    \begin{align*}
        \hypercoreat{\indexedcatvariableof{\nodes}}
        = \prod_{\secnodes\subset\nodes} \prod_{\thirdnodes\subset\secnodes}
        \big(\contractionof{\hypercore}{\indexedcatvariableof{\nodes\setexcept{\thirdnodes}}, \catvariableof{\thirdnodes} = \seccatindexof{\thirdnodes}}\big)^{(-1)^{\cardof{\secnodes}-\cardof{\thirdnodes}}} \, .
    \end{align*}
    We do this by applying a logarithm on the right hand side and grouping the terms by $\thirdnodes$ as
    \begin{align*}
        & \lnof{\prod_{\secnodes\subset\nodes} \prod_{\thirdnodes\subset\secnodes}
            \contractionof{\hypercore}{\indexedcatvariableof{\nodes\setexcept{\thirdnodes}}, \catvariableof{\thirdnodes} = \seccatindexof{\thirdnodes}}\big)^{(-1)^{\cardof{\secnodes}-\cardof{\thirdnodes}}}} \\
        & = \sum_{\thirdnodes\subset\nodes} \lnof{\contractionof{\hypercore}{\indexedcatvariableof{\nodes\setexcept{\thirdnodes}}, \catvariableof{\thirdnodes} = \seccatindexof{\thirdnodes}}}
        \left( \sum_{\secnodes\subset\nodes \wcols \thirdnodes\subset \secnodes} (-1)^{\cardof{\secnodes}-\cardof{\thirdnodes}} \right) \\
        & =  \sum_{\thirdnodes\subset\nodes} \lnof{\contractionof{\hypercore}{\indexedcatvariableof{\nodes\setexcept{\thirdnodes}}, \catvariableof{\thirdnodes} = \seccatindexof{\thirdnodes}}}
        \left( \sum_{i \in [\cardof{\nodes}-\cardof{\thirdnodes}]}  (-1)^{i}  \binom{\cardof{\nodes}-\cardof{\thirdnodes}}{i}  \right)
    \end{align*}
    Now, by the generic binomial theorem we have that for $n\in\nn, n \neq 0$
    \[ 0 = (1-1)^n = \sum_{i \in [n]}  (-1)^{i}  \binom{n}{i}   \, . \]
    Therefore, the summands for $\thirdnodes\neq\nodes$ vanish and we have
    \begin{align*}
        & \lnof{ \prod_{\secnodes\subset\nodes} \prod_{\thirdnodes\subset\secnodes}
            \big(\contractionof{\hypercore}{\indexedcatvariableof{\nodes\setexcept{\thirdnodes}}, \catvariableof{\thirdnodes} = \seccatindexof{\thirdnodes}}\big)^{(-1)^{\cardof{\secnodes}-\cardof{\thirdnodes}}} } \\
        & = \lnof{\hypercoreat{\indexedcatvariableof{\nodes}}}
        \left( \sum_{i \in [0]}  (-1)^{i}  \binom{0}{i}  \right) \\
        & = \lnof{\hypercoreat{\indexedcatvariableof{\nodes}}} \, .
    \end{align*}
    Applying the exponential function on both sides establishes the claim.
\end{proof}

\begin{lemma}
    \label{lem:independentContractionFactorization}
    Let $\hypercore$ be a positive tensor and $\secnodes\subset\nodes$ an arbitrary subset.
    When there are $\nodea,\nodeb \in\secnodes$ such that
    \begin{align*}
        \normalizationofwrt{\hypercore}{\catvariableof{\nodea},\catvariableof{\nodeb}}{\catvariableof{\nodes\setexcept{\{\nodea,\nodeb\}}}}
        = \contractionof{
            \normalizationofwrt{\hypercore}{\catvariableof{\nodea}}{\catvariableof{\nodes\setexcept{\{\nodea,\nodeb\}}}},
            \normalizationofwrt{\hypercore}{\catvariableof{\nodeb}}{\catvariableof{\nodes\setexcept{\{\nodea,\nodeb\}}}}
        }{\catvariableof{\secnodes}}\,,
    \end{align*}
    then for any indices $\seccatindexof{\secnodes}$ and $\catindexof{\secnodes}$
    \begin{align*}
        \prod_{\thirdnodes\subset\secnodes}
        \left(\contractionof{\hypercore}{\indexedcatvariableof{\nodes\setexcept{\thirdnodes}},\catvariableof{\thirdnodes}=\seccatindexof{\thirdnodes}}\right)^{(-1)^{\cardof{\secnodes}-\cardof{\thirdnodes}}} = 1 \, .
    \end{align*}
\end{lemma}
\begin{proof}
    We abbreviate
    \begin{align*}
        Z_{\thirdnodes}
        = \contractionof{\hypercore}{\indexedcatvariableof{\nodes\setexcept{\thirdnodes}},\catvariableof{\thirdnodes}=\seccatindexof{\thirdnodes}} \, .
    \end{align*}
    By reorganizing the sum over $\thirdnodes\subset\secnodes$ into  $\thirdnodes\subset\secnodes\setexcept{\{\nodea\cup\nodeb\}}$ we have
    \begin{align}
        \label{eq:indContFacProof}
        \prod_{\thirdnodes\subset\secnodes}
        \left(
            Z_{\thirdnodes}
        \right)^{(-1)^{\cardof{\secnodes}-\cardof{\thirdnodes}}} =
        \prod_{\thirdnodes\subset\secnodes\setexcept{\{\nodea,\nodeb\}}}
        \left(
            \frac{
                Z_{\thirdnodes} \cdot Z_{\thirdnodes\cup\{\nodea,\nodeb\}}
            }{
                Z_{\thirdnodes\cup\{\nodea\}} \cdot Z_{\thirdnodes\cup\{\nodeb\}}
            }
        \right)^{(-1)^{\cardof{\secnodes}-\cardof{\thirdnodes}}} \, .
    \end{align}
    From the independence assumption it follows that for any index $\catindex$
    \begin{align*}
        & \normalizationofwrt{\hypercore}{
            \indexedcatvariableof{\nodea}
        }{\indexedcatvariableof{\nodes\setexcept{\{\thirdnodes\cup\{\nodea,\nodeb\}\}}},\catvariableof{\thirdnodes}=\seccatindexof{\thirdnodes},  \indexedcatvariableof{\nodeb} }
        \\
        & \quad =
        \normalizationofwrt{\hypercore}{
            \indexedcatvariableof{\nodea}
        }{\indexedcatvariableof{\nodes\setexcept{\{\thirdnodes\cup\{\nodea,\nodeb\}\}}}, \catvariableof{\thirdnodes}=\seccatindexof{\thirdnodes}} \\
        & \quad  =
        \normalizationofwrt{\hypercore}{
            \indexedcatvariableof{\nodea}
        }{\indexedcatvariableof{\nodes\setexcept{\{\thirdnodes\cup\{\nodea,\nodeb\}\}}},\catvariableof{\thirdnodes}=\seccatindexof{\thirdnodes},  \catvariableof{\nodeb} = \seccatindexof{\nodeb}}\,.
    \end{align*}
    Applying this in each bracket term of \eqref{eq:indContFacProof} we get
    \begin{align*}
        \frac{
            Z_{\thirdnodes}
        }{
            Z_{\thirdnodes\cup\{\nodea\}}
        }
        & =
        \frac{
            \normalizationofwrt{\hypercore}{
                \indexedcatvariableof{\nodea}
            }{\indexedcatvariableof{\nodes\setexcept{\{\thirdnodes\cup\{\nodea,\nodeb\}\}}}, \catvariableof{\thirdnodes}=\seccatindexof{\thirdnodes}, \indexedcatvariableof{\nodeb} }
        }{
            \normalizationofwrt{\hypercore}{
                \catvariableof{\nodea} =\seccatindexof{\nodea}
            }{\indexedcatvariableof{\nodes\setexcept{\{\thirdnodes\cup\{\nodea,\nodeb\}\}}} , \catvariableof{\thirdnodes}=\seccatindexof{\thirdnodes}, \indexedcatvariableof{\nodeb}}
        } \\
        & =
        \frac{
            \normalizationofwrt{\hypercore}{
                \indexedcatvariableof{\nodea}
            }{\indexedcatvariableof{\nodes\setexcept{\{\thirdnodes\cup\{\nodea,\nodeb\}\}}}, \catvariableof{\thirdnodes}=\seccatindexof{\thirdnodes}, \catvariableof{\nodeb} = \seccatindexof{\nodeb}}
        }{
            \normalizationofwrt{\hypercore}{
                \catvariableof{\nodea} =\seccatindexof{\nodea}
            }{\indexedcatvariableof{\nodes\setexcept{\{\thirdnodes\cup\{\nodea,\nodeb\}\}}}, \catvariableof{\thirdnodes}=\seccatindexof{\thirdnodes},\catvariableof{\nodeb} = \seccatindexof{\nodeb}}
        } \\
        & =
        \frac{
            Z_{\thirdnodes\cup\{\nodeb\}}
        }{
            Z_{\thirdnodes\cup\{\nodea,\nodeb\}}
        } \, .
    \end{align*}
    Thus, each factor in \eqref{eq:indContFacProof} is trivial, which establishes the claim.
\end{proof}

We are finally ready to prove the Hammersley-Clifford \theref{the:factorizationHammersleyClifford} based on the Lemmata above.

\begin{proof}[Proof of \theref{the:factorizationHammersleyClifford}]
    $ii)\Rightarrow i)$
    By \lemref{the:contractionFactorization} we have for any indices $\catindexof{\nodes}$ and $\seccatindexof{\nodes}$
    \begin{align*}
        \probat{\indexedcatvariableof{\nodes}} =
        \prod_{\secnodes\subset\nodes} \prod_{\thirdnodes\subset\secnodes}
        \left(
            \probat{\indexedcatvariableof{\thirdnodes},\catvariableof{\nodes\setexcept{\thirdnodes}}=\seccatindexof{\nodes\setexcept{\thirdnodes}}}
        \right)^{(-1)^{\cardof{\secnodes}-\cardof{\thirdnodes}}} \, .
    \end{align*}
    Using the clique-capturing assumption of \theref{the:factorizationHammersleyClifford}, we find for any subset $\secnodes\subset\nodes$, which is not contained in a hyperedge $\nodea,\nodeb \in\secnodes$ such that $\catvariableof{\nodea}$ is independent of $\catvariableof{\nodeb}$ conditioned on $\catvariableof{\secnodes\setexcept{\{\nodea,\nodeb\}}}$.
    If no such nodes $\nodea,\nodeb \in\secnodes$ exists, $\secnodes$ would be contained in a hyperedge since the hypergraph is assumed to be clique-capturing.
    By \lemref{lem:independentContractionFactorization} we then have
    \begin{align*}
        \prod_{\thirdnodes\subset\secnodes}
        \left(
            \probat{\indexedcatvariableof{\thirdnodes},\catvariableof{\nodes\setexcept{\thirdnodes}}=\seccatindexof{\nodes\setexcept{\thirdnodes}}}
        \right)^{(-1)^{\cardof{\secnodes}-\cardof{\thirdnodes}}} = 1 \, .
    \end{align*}
    Using the function
    \begin{align*}
        \alpha: \{\secnodes : \exists\edge\in\edges: \secnodes \subset \edge \} \rightarrow \edges\,,
    \end{align*}
    we label the remaining node subsets by a hyperedge containing the subset.
    For each $\edge\in\edges$, we build the tensor
    \begin{align*}
        \hypercoreofat{\edge}{\catvariableof{\edge}} = \prod_{\secnodes \wcols \alpha(\secnodes) = \edge} \prod_{\thirdnodes\subset\secnodes}
        \left(
            \probat{\indexedcatvariableof{\thirdnodes},\catvariableof{\nodes\setexcept{\thirdnodes}}=\seccatindexof{\nodes\setexcept{\thirdnodes}}}
        \right)^{(-1)^{\cardof{\secnodes}-\cardof{\thirdnodes}}} \, .
    \end{align*}
    and get that
    \begin{align*}
        \probat{\catvariableof{\nodes}} & = \contractionof{\extnetasset}{\catvariableof{\nodes}} \\
        & = \normalizationof{\extnetasset}{\catvariableof{\nodes}} \, .
    \end{align*}
    We have thus constructed a Markov network with trivial partition function, whose contraction coincides with the probability distribution.

    $i)\Rightarrow ii)$:
    To show the converse statement, assume that there is a Markov network representing the distribution $\probwithnodes$ and choose subsets $\nodesa,\nodesb,\nodesc\subset\nodes$ such that $\nodesc$ separates $\nodesa$ from $\nodesb$.
    Denote by $\nodes_0$ the nodes with paths to $\nodesa$, which do not contain a node in $\nodesc$ and by $\nodes_1$ the nodes with paths to $\nodesb$, which do not contain a node in $\nodesc$.
    Furthermore, we denote by $\edges_0$ the hyperedges which contain a node in $\nodes_0$ and by $\edges_1$ the hyperedges which contain a node in $\nodes_1$.
    By assumption of separability, both sets $\edges_0$ and $\edges_1$ are disjoint and no node in $\nodesa$ is in a hyperedge in $\edges_1$ and respectively no node in $\nodesb$ is in a hyperedge in $\edges_0$.
    We then have
    \begin{align*}
        \normalizationofwrt{\extnetasset}{\catvariableof{\nodesa},\catvariableof{\nodesb}}{\indexedcatvariableof{\nodesc}}
        = & \normalizationof{\extnetasset\cup\{\onehotmapof{\catindexof{\nodesc}}\}}{\catvariableof{\nodesa},\catvariableof{\nodesb}} \\
        = &  \normalizationof{\{\hypercoreof{\edge}\wcols\edgein_0\}\cup\{\onehotmapof{\catindexof{\nodesc}}\}}{\catvariableof{\nodesa}} \\
        & \quad \otimes \normalizationof{\{\hypercoreof{\edge}\wcols\edgein_1\}\cup\{\onehotmapof{\catindexof{\nodesc}}\}}{\catvariableof{\nodesb}} \, .
    \end{align*}
    By \defref{def:condIndependence}, this is the independence of $\catvariableof{\nodesa}$ and $\catvariableof{\nodesb}$ conditioned on $\catvariableof{\nodesc}$.
\end{proof}

\subsection{Fisher-Neyman}

Since sufficient statistics are sometimes introduced based on the data processing inequality (see e.g. \cite{cover_elements_2006}), we also show that also that definition is equivalent to the factorization of the family. Here, $\mutinfof{X}{Y}$ denotes the mutual information of two random variables $X,Y$.

\begin{theorem}[Fisher-Neyman factorization theorem]
    \label{the:generalFactorizationFisherNeyman}
    Let $\probat{\catvariable,\thirdcatvariable}$ be a joint distribution of variables $\catvariable,\,\thirdcatvariable$ with values $\valof{\catvariable},\,\valof{\thirdcatvariable}$ and let $\sstat$ be a statistic, which maps $\valof{\catvariable}$ to $\valof{\headvariableof{\sstat}}$ .
    We introduce a variable $\headvariableof{\sstat}$ and define a joint distribution by
    \begin{align*}
        \probat{\catvariable,\headvariableof{\sstat},\thirdcatvariable}
        = \contractionof{
            \probat{\catvariable,\thirdcatvariable},\bencodingofat{\sstat}{\headvariableof{\sstat},\catvariable}
        }{\catvariable,\headvariableof{\sstat},\thirdcatvariable} \, .
    \end{align*}
    The following are equivalent:
    \begin{itemize}
        \item[i)] The Data Processing Inequality holds straight, that is
        \begin{align*}
            \mutinfof{\thirdcatvariable}{\catvariable}
            = \mutinfof{\thirdcatvariable}{\headvariableof{\sstat}}
        \end{align*}
        \item[ii)] $\thirdcatvariable\rightarrow\headvariableof{\sstat}\rightarrow\catvariable$ is a Markov Chain, that is
        \begin{align*}
            \condindependent{\thirdcatvariable}{\catvariable}{\headvariableof{\sstat}}
        \end{align*}
        \item[iii)]
        There are tensors $\acttensorat{\headvariableof{\sstat},\thirdcatvariable}$ and $\basemeasureat{\catvariable}$ such that for any $\catindex\in\valof{\catvariable}$ and $\thirdcatindex\in\valof{\thirdcatvariable}$
        \begin{align*}
            \probat{\indexedthirdcatvariable,\indexedcatvariable}
            = \acttensorat{\headvariableof{\sstat}=\sstatat{\catindex},\indexedthirdcatvariable} \cdot \basemeasureat{\indexedcatvariable} \, .
        \end{align*}
    \end{itemize}
\end{theorem}
\begin{proof}
    $i) \Leftrightarrow ii)$:
    We always have
    \begin{align*} 
        \mutinfof{\thirdcatvariable}{\catvariable}
        = \mutinfof{\thirdcatvariable}{(\catvariable,\headvariableof{\sstat})}
        = \mutinfof{\thirdcatvariable}{\headvariableof{\sstat}} + \condmutinfof{\thirdcatvariable}{\catvariable}{\headvariableof{\sstat}}  
    \end{align*}
    and thus $i)$ is equivalent to
    \begin{align*}
        \condmutinfof{\thirdcatvariable}{\catvariable}{\headvariableof{\sstat}} = 0 \, .
    \end{align*}
    Using the KL-divergence characterization of the mutual information, this is equivalent to 
    \begin{align*}
        \condprobat{\thirdcatvariable,\catvariable}{\headvariableof{\sstat}}
        = \contractionof{\condprobat{\thirdcatvariable}{\headvariableof{\sstat}},\condprobat{\catvariable}{\headvariableof{\sstat}} }{\thirdcatvariable,\catvariable,\headvariableof{\sstat}} \, .
    \end{align*}
    This is equivalent to the conditional independence statement $ii)$. \\

    $ii) \Rightarrow iii)$:
    Let us assume $ii)$.
    For all $\thirdcatindex\in\valof{\thirdcatvariable}$ and $\catindex\in\valof{\catvariable}$ we then have
    \begin{align*}
        \condprobat{\indexedthirdcatvariable}{\indexedcatvariable}
        &= \condprobat{\indexedthirdcatvariable}{\indexedcatvariable,\headvariableof{\sstat}=\sstatat{\catindex}} \\
        &= \condprobat{\indexedthirdcatvariable}{\headvariableof{\sstat}=\sstatat{\catindex}}
    \end{align*}
    Here we used that $\headvariableof{\sstat}$ has a deterministic dependence on $\catvariable$.
    Therefore, there is a tensor $\acttensor$ such that for all $\thirdcatindex\in\valof{\thirdcatvariable}$ and $\catindex\in\valof{\catvariable}$
    \begin{align*}
        \acttensorat{\headvariableof{\sstat}=\sstatat{\catindex},\indexedthirdcatvariable} = \condprobat{\indexedthirdcatvariable}{\indexedcatvariable} \, .
    \end{align*}
    We further define a tensor $\basemeasureat{\catvariable}=\probat{\catvariable}$ and get
    \begin{align*}
        \probat{\indexedthirdcatvariable,\indexedcatvariable}
        &= \probat{\indexedcatvariable} \cdot \condprobat{\indexedthirdcatvariable}{\indexedcatvariable} \\
        &= \acttensorat{\headvariableof{\sstat}=\sstatat{\catindex},\indexedthirdcatvariable} \cdot \basemeasureat{\indexedcatvariable} \, .
    \end{align*}

    $iii) \Rightarrow ii)$:
    When assuming $iii)$ we have for all $(\catindex,\thirdcatindex)\in\valof{\thirdcatvariable}\times \valof{\catvariable}$
    \begin{align*}
        \condprobat{\indexedthirdcatvariable}{\indexedcatvariable}
        &= \normalizationofwrt{\acttensorat{\headvariableof{\sstat},\thirdcatvariable},\bencodingofat{\sstat}{\headvariableof{\sstat},\catvariable},\basemeasureat{\catvariable}}{\indexedthirdcatvariable}{\indexedcatvariable} \\
        &= \normalizationof{\acttensorat{\headvariableof{\sstat},\thirdcatvariable},\bencodingofat{\sstat}{\headvariableof{\sstat},\indexedcatvariable},\basemeasureat{\indexedcatvariable}}{\indexedthirdcatvariable} \\
        &= \normalizationof{\acttensorat{\headvariableof{\sstat},\thirdcatvariable},\onehotmapofat{\sstatat{\catindex}}{\headvariableof{\sstat}}}{\indexedthirdcatvariable} \\
        &= \condprobat{\indexedthirdcatvariable}{\headvariableof{\sstat}=\sstatat{\catindex}} \, .
    \end{align*}
    We further have for almost all $\headindexof{\sstat}\in\valof{\headvariableof{\sstat}}$, $\thirdcatindex\in\valof{\thirdcatvariable}$ and $\catindex\in\valof{\catvariable}$ that $\headindexof{\sstat}=\sstatat{\catindex}$ and
    \begin{align*}
        \condprobat{\indexedthirdcatvariable}{\indexedcatvariable,\indexedheadvariableof{\sstat}}
        = \condprobat{\indexedthirdcatvariable}{\indexedcatvariable}
    \end{align*}
    and with the above at thus at almost all such pairs
    \begin{align*}
        \condprobat{\indexedthirdcatvariable}{\indexedcatvariable,\indexedheadvariableof{\sstat}}
        = \condprobat{\indexedthirdcatvariable}{\indexedheadvariableof{\sstat}} \, .
    \end{align*}
    This is equivalent to $ii)$.
\end{proof}

\theref{the:factorizationFisherNeyman} follows from \theref{the:generalFactorizationFisherNeyman} by the equivalence of $ii)$ and $iii)$.
    \section{Implementation of the algorithms and examples}\label{sec:algExaImplementation}


The implementations of the algorithms and concepts are available at
\begin{center}
    \href{https://github.com/tnreason/nesy-demonstrations/}{\url{https://github.com/tnreason/nesy-demonstrations/}}
\end{center}
and implemented with \tnreason{} in the version \curvertnreason{}.

\subsection{Algorithm~\ref{alg:treeBeliefPropagation}, \ref{alg:directedBeliefPropagation} and \ref{alg:constraintPropagation} (Tree, Directed Belief and Constraint Propagation)}\label{sec:propAlgsImplementation}

The three message passing algorithms are implemented as functions in one class \pythoninline{ContractionPropagation}, since they share common structure.
\begin{lstlisting}[language=Python]
from tnreason.engine import contract
from tnreason.engine import create_from_slice_iterator as create


class ContractionPropagation:
    """
    Summary Class for the Tree Belief, Directed Belief and Constraint Propagation Algorithms
    """
    def __init__(self, cores):
        self.cores = cores
        self.directions = {send: [receive for receive in cores if
                                  set(cores[send].colors) & set(
                                      cores[receive].colors) and receive != send]
                           for send in cores}
        self.messages = {receive: {} for receive in self.cores}

    def trivial_message(self, send, receive):
        """
        Prepares trivial message from the send to the receive hyperedge
        """
        commonColors = list(set(self.cores[send].colors) & set(self.cores[receive].colors))
        shape = [self.cores[send].shape[i]
                 for i, c in enumerate(self.cores[send].colors) if c in commonColors]
        return create(shape=shape, colors=commonColors, sliceIterator=[(1, {})])

    def calculate_message(self, send, receive):
        """
        Contract received messages with hypercore to send new
        """
        return contract({send: self.cores[send],
                         **{preSend: self.messages[send][preSend] for preSend in self.messages[send]
                            if preSend != receive}},
                        openColors=list(set(self.cores[send].colors) & set(self.cores[receive].colors)))

    def tree_propagation(self):
        """
        Implementation of the Directed Belief Propagation Algorithm:
        Messages are sent starting at the leafs and scheduled if all others received at a core
        """
        schedule = [(send, receive) for send in self.cores for receive in
                    self.directions[send] if len(self.directions[send]) == 1]
        while len(schedule) > 0:
            send, receive = schedule.pop()
            self.messages[receive][send] = self.calculate_message(send, receive)
            for next in self.directions[receive]:
                if (not receive in self.messages[next] and
                        all([(otherSendKey in self.messages[receive] or otherSendKey == next or
                              receive not in self.directions[otherSendKey]) for
                             otherSendKey in self.directions])):
                    schedule.append((receive, next))

    def directed_propagation(self, edgeDirections):
        """
        Implementation of the Directed Belief Propagation Algorithm:
        Messages are sent in direction of the hypergraph
        """
        filteredDirections = {
            send: [
                receive for receive in self.directions[send]
                if (common := set(self.cores[send].colors) & set(self.cores[receive].colors))
                   and common.issubset(set(edgeDirections[send][1]))
                   and common.issubset(set(edgeDirections[receive][0]))
            ]
            for send in self.directions
        }

        schedule = [(send, receive) for send in filteredDirections
                    for receive in filteredDirections[send] if len(edgeDirections[send][0]) == 0]

        while len(schedule) > 0:
            send, receive = schedule.pop()
            self.messages[receive][send] = self.calculate_message(send, receive)
            for x in set(edgeDirections[send][1]) & set(edgeDirections[receive][0]):
                edgeDirections[receive][0].remove(x)
            if len(edgeDirections[receive][0]) == 0:
                schedule = schedule + [(receive, next) for next in filteredDirections[receive]
                                       if (receive, next) not in schedule]

    def constraint_propagation(self, startSendKeys):
        """
        Implementation of the Constraint Propagation Algorithm:
        Messages are resent, when the support of a received message has changed
        """
        schedule = [(send, receive) for send in startSendKeys for receive in
                    self.directions[send]]
        while len(schedule) > 0:
            send, receive = schedule.pop()
            message = (self.messages[receive][send].clone() if send in self.messages[receive]
                       else self.trivial_message(send, receive))
            cont = self.calculate_message(send, receive)

            messageChanged = False
            for val, pos in message:
                if message[pos] != 0 and cont[pos] == 0:
                    message[pos] = - message[pos]
                    messageChanged = True
            self.messages[receive][send] = message

            for next in self.directions[receive]:
                if messageChanged and next != receive and (receive, next) not in schedule:
                    schedule.append((receive, next))
\end{lstlisting}

\subsubsection{Example~\ref{exa:madicRepresentation} and \ref{exa:madicPropagation} (Integer Summation in $\catdim$-adic Representation)}

Following the decomposition of $\catdim$-adic summations into local summations, the function \pythoninline{get_sum_tn} produces a corresponding tensor network of basis encodings.
We test by coordinate retrieval operations, whether the summation is performed correctly.
\begin{lstlisting}[language=Python]
from tnreason import engine
import math

from copy import deepcopy


def get_sum_tn(m, d):
    return {"b_0": engine.create_from_slice_iterator(
        shape=[m, 2, m, m],
        colors=[f"Y_{0}", f"Z_{0}", f"X_{0}", f"TX_{0}"],
        sliceIterator=[(1, {f"Y_{0}": (x + tx) % m, f"Z_{0}": math.floor((x + tx) / m),
                            f"X_{0}": x, f"TX_{0}": tx}) for x in range(m) for tx in range(m)]),
        **{f"middleBlock{k}": engine.create_from_slice_iterator(
            shape=[m, 2, m, m, 2],
            colors=[f"Y_{k}", f"Z_{k}", f"X_{k}", f"TX_{k}", f"Z_{k - 1}"],
            sliceIterator=[
                (1, {f"Y_{k}": (x + tx + z0) % m,
                     f"Z_{k}": math.floor((x + tx + z0) / m),
                     f"X_{k}": x, f"TX_{k}": tx, f"fZ_{k - 1}": z0}) for x
                in range(m) for tx in range(m) for z0 in range(2)]
        ) for k in range(1, d - 1)},
        **{f"b_{d - 1}": engine.create_from_slice_iterator(
            shape=[m, 2, m, m, 2],
            colors=[f"Y_{d - 1}", f"Y_{d}", f"X_{d - 1}", f"TX_{d - 1}", f"Z_{d - 2}"],
            sliceIterator=[
                (1, {f"Y_{d - 1}": (x + tx + z0) % m, f"Y_{d}": math.floor((x + tx + z0) / m),
                     f"X_{d - 1}": x, f"TX_{d - 1}": tx, f"Z_{d - 2}": z0}) for x
                in range(m) for tx in range(m) for z0 in range(2)]
        )}}


def encode_digits(num0, num1, m):
    return {**{f"X_{len(num0) - 1 - i}_eC": engine.create_from_slice_iterator(shape=[m], colors=[
        f"X_{len(num0) - 1 - i}"], sliceIterator=[(1, {f"X_{len(num0) - 1 - i}": int(digit)})]) for
               i, digit in enumerate(num0)},
            **{f"TX_{len(num1) - 1 - i}_eC": engine.create_from_slice_iterator(shape=[m], colors=[
                f"TX_{len(num1) - 1 - i}"], sliceIterator=[
                (1, {f"TX_{len(num0) - 1 - i}": int(digit)})]) for i, digit in
               enumerate(num1)}}


assert 1 == encode_digits("0001", "0000", 10)["X_0_eC"][{"X_0": 1}]
assert 0 == encode_digits("0001", "0000", 10)["X_0_eC"][{"X_0": 0}]

## Example: 08+12=020 in basis 10
m = 10
catorder = 2
assert 1 == int(engine.contract(coreDict={**get_sum_tn(m, catorder), **encode_digits("08", "12", m)},
                                openColors=[f"Y_{k}" for k in range(catorder + 1)])[
                    {"Y_2": 0, "Y_1": 2, "Y_0": 0}])
assert 1 == int(engine.contract(coreDict={**get_sum_tn(m, catorder), **encode_digits("00", "00", m)},
                                openColors=[])[:])
## Example: 10+11=101 in basis 2
m = 2
catorder = 2
assert 1 == int(engine.contract(coreDict={**get_sum_tn(m, catorder), **encode_digits("10", "11", m)},
                                openColors=[f"Y_{k}" for k in range(catorder + 1)])[
                    {"Y_2": 1, "Y_1": 0, "Y_0": 1}])
assert 1 == int(engine.contract(coreDict={**get_sum_tn(m, catorder), **encode_digits("10", "11", m)},
                                openColors=[])[:])

from demonstrations.comp_act_nets.algorithms import propagation as cp

edgeDirections = {
    **{f"X_{i}_eC": [[], [f"X_{i}"]] for i in range(catorder)},
    **{f"TX_{i}_eC": [[], [f"TX_{i}"]] for i in range(catorder)},
    "b_0": [["X_0", "TX_0"], ["Y_0", "Z_0"]],
    **{f"b__{i}": [[f"X_{i}", f"TX_{i}", f"Z_{i - 1}"], [f"Y_{i}", f"Z_{i}"]]
       for i in range(1, catorder - 1)},
    f"b_{catorder - 1}": [[f"X_{catorder - 1}", f"TX_{catorder - 1}", f"Z_{catorder - 2}"],
                          [f"Y_{catorder - 1}", f"Y_{catorder}"]],
}

propagator = cp.ContractionPropagation({**get_sum_tn(m, catorder), **encode_digits("01", "01", m)})
propagator.directed_propagation(edgeDirections=deepcopy(edgeDirections))

## Check whether the message arrived at b_1 states that the carry bit is 1
assert propagator.messages["b_1"]["b_0"][{"Z_0": 0}] == 0
assert propagator.messages["b_1"]["b_0"][{"Z_0": 1}] == 1

propagator = cp.ContractionPropagation({**get_sum_tn(m, catorder),
                                        **encode_digits("10", "10", m)})
propagator.directed_propagation(edgeDirections=deepcopy(edgeDirections))

## Check whether the message arrived at b_1 states that the carry bit is 1
assert propagator.messages["b_1"]["b_0"][{"Z_0": 0}] == 1
assert propagator.messages["b_1"]["b_0"][{"Z_0": 1}] == 0
\end{lstlisting}

\subsubsection{Example~\ref{exa:studentHC} and \ref{exa:studentBP} (Student Markov network)}

We here implement the Markov network on the hypergraph of \exaref{exa:studentHC}, with tensors having independent random coordinates drawn from the uniform distribution on $[0,1]$.
We test in a final \pythoninline{assert} statement, whether the messages resulting from \algoref{alg:treeBeliefPropagation} in a tree implementation contract to the marginal distribution, which we directly compute for comparison.
\begin{lstlisting}[language=Python]
from tnreason.engine import create_random_core, contract

studentTensorNetwork = {
    "t0": create_random_core(name="t0", colors=["G", "D", "I"], shape=[6, 3, 2]),
    "t1": create_random_core(name="t1", colors=["L", "G"], shape=[2, 6]),
    "t2": create_random_core(name="t2", colors=["I", "S"], shape=[2, 10]),
}

## Execute the contraction propagation algorithm in the tree-based implementation

from demonstrations.comp_act_nets.algorithms import propagation as cp

propagator = cp.ContractionPropagation(studentTensorNetwork)
propagator.tree_propagation()

## Test on the marginals of the variables "L","G" (core "t1")

testContraction = contract(studentTensorNetwork, openColors=["L", "G"])
propContraction = contract({"mes_t0_t1": propagator.messages["t1"]["t0"],
                                   "t1": studentTensorNetwork["t1"]}, openColors=["L", "G"])

tolerance = 1e-6
for posDict in [{"L": 0, "G": 1}, {"L": 1, "G": 5}]:
    assert abs(testContraction[posDict] - propContraction[posDict]) < tolerance
\end{lstlisting}

\subsubsection{Example~\ref{exa:sudokuEntailment}, \ref{exa:sudokuDecomposition} and \ref{exa:sudokuMessagePassing} (Sudoku Game)}

We implement the $\sudokunum^2\times\sudokunum^2$ Sudoku with the start assignment given in \exaref{exa:sudokuEntailment} and apply the Constraint Propagation \algoref{alg:constraintPropagation} to deduce the full assignment.
We then test whether the correct board assignment (given in \exaref{exa:sudokuMessagePassing}) has been found.
\begin{lstlisting}[language=Python]
from tnreason.engine import contract
from tnreason.engine import create_from_slice_iterator as create
import numpy as np


def create_sudoku_rule_tensor_network(n):
    """
    Creates a tensor network of n^2 \tau^k matrices to each Sudoku constraint
    """
    rulesSpecDict = {
        ## Column Constraints
        **{f"I_:_:_{c0}_{c1}_{i}": [f"X_{r0}_{r1}_{c0}_{c1}_{i}" for r0 in range(n) for r1 in
                                    range(n)] for c0 in range(n) for c1 in range(n)
           for i in range(n ** 2)},
        ## Row Constraints
        **{f"I_{r0}_{r1}_:_:_{i}": [f"X_{r0}_{r1}_{c0}_{c1}_{i}" for c0 in range(n) for c1 in
                                    range(n)] for r0 in range(n) for r1 in range(n)
           for i in range(n ** 2)},
        ## Squares Constraints
        **{f"I_{r0}_:_{c0}_:_{i}": [f"X_{r0}_{r1}_{c0}_{c1}_{i}" for r1 in range(n) for c1 in
                                    range(n)] for r0 in range(n) for c0 in range(n)
           for i in range(n ** 2)},
        ## Position Constraints
        **{f"I_{r0}_{r1}_{c0}_{c1}_:": [f"X_{r0}_{r1}_{c0}_{c1}_{i}" for i in range(n ** 2)]
           for r0 in range(n) for r1 in range(n) for c0 in range(n) for c1 in range(n)}
    }
    cores = {}
    for decomKey in rulesSpecDict:
        cores.update({
            decomKey + "_" + atomVar: create(
                shape=[2, len(rulesSpecDict[decomKey])],
                colors=[atomVar, decomKey],
                sliceIterator=[(1, {atomVar: 0}),
                               (-1, {atomVar: 0, decomKey: i}),
                               (1, {atomVar: 1, decomKey: i})])
            for i, atomVar in enumerate(rulesSpecDict[decomKey])
        })
    return cores


def encode_trivial_extended_evidence(E, n):
    """
    Prepares e_1 basis vectors for known variables and trivial vectors for others
    """
    return {**{f"{r0}_{r1}_{c0}_{c1}_{i}_eC":
                   create(shape=[2], colors=[f"X_{r0}_{r1}_{c0}_{c1}_{i}"],
                          sliceIterator=[(1, {f"X_{r0}_{r1}_{c0}_{c1}_{i}": 1})])
               for r0, r1, c0, c1, i in E},
            **{f"{r0}_{r1}_{c0}_{c1}_{i}_eC":
                   create(shape=[2], colors=[f"X_{r0}_{r1}_{c0}_{c1}_{i}"],
                          sliceIterator=[(1, {})])
               for r0 in range(n) for r1 in range(n) for c0 in range(n)
               for c1 in range(n) for i in range(n ** 2) if (r0, r1, c0, c1, i) not in E}}


def extract_resulting_evidence(propagator, n):
    """
    Returns the evidence given a ContractionPropagation instance
    """
    return [(r0, r1, c0, c1, i) for r0 in range(n) for r1 in range(n)
            for c0 in range(n) for c1 in range(n) for i in range(n ** 2)
            if contract({
            "eC": propagator.cores[f"{r0}_{r1}_{c0}_{c1}_{i}_eC"],
            **propagator.messages[f"{r0}_{r1}_{c0}_{c1}_{i}_eC"]},
            openColors=[f"X_{r0}_{r1}_{c0}_{c1}_{i}"])[{f"X_{r0}_{r1}_{c0}_{c1}_{i}": 0}] == 0]


def tuples_to_array(evidence, n=2):
    """
    Arranges the variables in an array
    """
    array = np.zeros(shape=(n ** 2, n ** 2))
    for (r0, r1, c0, c1, i) in evidence:
        array[r0 * n + r1, c0 * n + c1] = i + 1
    return array


from demonstrations.comp_act_nets.algorithms import propagation as cp

n = 2
evidence = [(0, 0, 0, 0, 0), (0, 0, 1, 0, 2), (0, 0, 1, 1, 1),
            (0, 1, 0, 1, 1), (1, 0, 1, 0, 3), (1, 1, 0, 0, 3),
            (1, 1, 0, 1, 2)]
propagator = cp.ContractionPropagation(
    cores={**create_sudoku_rule_tensor_network(n=n),
           **encode_trivial_extended_evidence(evidence, n=n)})
propagator.constraint_propagation([f"{r0}_{r1}_{c0}_{c1}_{i}_eC" for (r0, r1, c0, c1, i) in evidence])
solutionArray = tuples_to_array(extract_resulting_evidence(propagator, n=2))
assert np.all(solutionArray == np.array([[1, 4, 3, 2], [3, 2, 1, 4], [2, 1, 4, 3], [4, 3, 2, 1]]))
\end{lstlisting}

\subsection{Algorithm~\ref{alg:AMM_HLN} (Alternating Moment Matching)}

We implement the Alternating Moment Matching algorithm, which estimates the parameters of \HybridLogicNetworks{}, as a class \pythoninline{MomentMatcher}.
\begin{lstlisting}[language=Python]
from tnreason.engine import contract
from tnreason.engine import create_from_slice_iterator as create
import math


class MomentMatcher:
    def __init__(self, cores, hCols, satRates):
        self.cores = cores
        self.hCols = hCols
        self.satRates = satRates

        self.hardParams = {hCol: int(satRates[hCol]) for hCol in self.hCols if
                           satRates[hCol] in [0, 1]}
        self.softParams = {hCol: 0 for hCol in self.hCols if hCol not in self.hardParams}

    def update_canonical_parameter(self, uCol):
        con = contract({**self.cores,
                        **{hCol: create(shape=[2], colors=[hCol],
                                        sliceIterator=[(1, {hCol: self.hardParams[hCol]})])
                           for hCol in self.hardParams},
                        **{hCol: create(shape=[2], colors=[hCol],
                                        sliceIterator=[(1, {hCol: 0}),
                                                       (math.exp(self.softParams[hCol]), {hCol: 1})])
                           for hCol in self.softParams if hCol != uCol}
                        }, openColors=[uCol])
        self.softParams[uCol] = math.log(self.satRates[uCol] * con[{uCol: 0}] / (
                (1 - self.satRates[uCol]) * con[{uCol: 1}]))

    def alternate(self, iterations=1):
        for _ in range(iterations):
            for hCol in self.softParams:
                self.update_canonical_parameter(hCol)
\end{lstlisting}
Let us now show the usage of the algorithm on the toy accounting model presented in Example~\ref{exa:hlnAccountingRep}.
To this end we train the parameters based on a the dataset described in Example~\ref{exa:hlnAccountingAMM}, and assert that the learned parameters are close to the true parameters.
Note that a single iterations suffices for convergence in this simple example.
\begin{lstlisting}[language=Python]
import pandas as pd

samples = pd.DataFrame({
    "A1": [0, 0, 0, 0, 0, 0, 0, 0, 0, 1, 1, 1, 1, 1, 1, 1, 1, 1, 1, 1],
    "A2": [1, 1, 1, 1, 1, 1, 1, 1, 1, 0, 0, 0, 0, 0, 0, 0, 0, 0, 0, 0],
    "F": [0, 0, 0, 0, 0, 0, 0, 1, 1, 0, 1, 1, 1, 1, 1, 1, 1, 1, 1, 1],
})

formulaExpressions = {
    "(xor_A1_A2)": ["xor", "A1", "A2", 0],
    "(imp_F_A1)": ["imp", "F", "A1", 0],
}

from tnreason.engine import normalize
from tnreason.application import data_to_cores as dtc
from tnreason.application import create_cores_to_expressionsDict as cte
from demonstrations.comp_act_nets.algorithms import moment_matching as mm

satRates = {
    formulaKey + "_cV":
        normalize({**dtc.create_data_cores(samples),
                   **cte({formulaKey: formulaExpressions[formulaKey]})},
                  outColors=[formulaKey + "_cV"], inColors=[])[{formulaKey + "_cV": 1}]
    for formulaKey in formulaExpressions
}

matcher = mm.MomentMatcher(cores=cte(formulaExpressions),
                           satRates=satRates, hCols=["(xor_A1_A2)_cV", "(imp_F_A1)_cV"])
matcher.alternate(iterations=1)
assert abs(matcher.softParams["(imp_F_A1)_cV"] - 1.09861228866811) < 1e-8
\end{lstlisting}

\end{document}